%% file: SubPruning_neurips_2022.tex
\documentclass{article}

\usepackage[final]{neurips_2022}

\usepackage[utf8]{inputenc} 
\usepackage[T1]{fontenc}    
\usepackage[pagebackref=true]{hyperref}       
\usepackage{url}            
\usepackage{booktabs}       
\usepackage{amsfonts}       
\usepackage{nicefrac}       
\usepackage{microtype}      
\usepackage{xcolor}         

\usepackage{amsmath}
\usepackage{amssymb}
\usepackage{mathtools}
\usepackage{amsthm}
\usepackage{graphicx}
\usepackage{algorithm, algorithmic}
\usepackage{amsfonts,color,url}
\usepackage{thm-restate}
\usepackage{xspace}
\usepackage{enumitem}
\usepackage{subcaption}
\usepackage[capitalize,noabbrev]{cleveref}
\usepackage[disable,textsize=small]{todonotes}
\newcommand{\mtodo}[1]{\todo[color=red!30, inline]{ #1}} 
\renewcommand*\backref[1]{\ifx#1\relax \else (Cited on #1) \fi}
\usepackage{pifont}
\usepackage{newunicodechar}
\newunicodechar{✓}{\ding{51}}
\newunicodechar{✗}{\ding{55}}
\usepackage{adjustbox}

\title{Data-Efficient Structured Pruning via \\Submodular Optimization}

\author{%
Marwa El Halabi\thanks{work done partially at MIT, CSAIL.} \\ 
Samsung - SAIT AI Lab, Montreal\\
  \And
  Suraj Srinivas\thanks{work done partially at Idiap Research Institute, Switzerland.} \\
  Harvard University \\
  \AND
  Simon Lacoste-Julien \\
 Mila, Universit\'e de Montreal\\ 
Samsung - SAIT AI Lab, Montreal\\
 Canada CIFAR AI Chair \\
}

\input{preamble}

\begin{document}

\maketitle

\begin{abstract}
\looseness=-1 Structured pruning is an effective approach for compressing large pre-trained neural networks without significantly affecting their performance. 
However, most current structured pruning methods do not provide any performance guarantees, and often require fine-tuning, which makes them inapplicable in the limited-data regime.
We propose a principled data-efficient structured pruning method based on submodular optimization. In particular, for a given layer, we select neurons/channels to prune and corresponding new weights for the next layer, that minimize the change in the next layer's input induced by pruning. We show that this selection problem is a weakly submodular maximization problem, thus it can be provably approximated using an efficient greedy algorithm. Our method is guaranteed to have an exponentially decreasing error between the original model and the pruned model outputs w.r.t the pruned size, under reasonable assumptions. 
It is also one of the few methods in the literature that uses only a limited-number of training data and no labels. Our experimental results demonstrate that our method outperforms state-of-the-art methods in the limited-data regime. 
\end{abstract}

\input{introduction}
\input{preliminaries}

\input{onelayer-neurons}

\input{onelayer-channels}

\input{multiple-layers}

\input{error-bounds}

\input{experiments}

\input{conclusion}

\begin{ack}
\looseness=-1 We thank Stefanie Jegelka, Debadeepta Dey, Jose Gallego-Posada for their helpful discussions, and Yan Zhang, Boris Knyazev for their help with running experiments.
We also acknowledge the MIT SuperCloud and Lincoln Laboratory Supercomputing Center (\url{supercloud.mit.edu}), Compute Canada (\url{www.computecanada.ca}), Calcul Quebec (\url{www.calculquebec.ca}), WestGrid (\url{www.westgrid.ca}), ACENET (\url{ace-net.ca}), the Mila IDT team, Idiap Research Institute and the Machine Learning Research Group at the University of Guelph, for providing HPC resources that have contributed to the research results reported within this paper. This research was partially supported by the Canada CIFAR AI Chair Program. Simon Lacoste-Julien is a CIFAR Associate Fellow in the Learning Machines \& Brains program.
\end{ack}


\bibliographystyle{abbrvnat} 
\bibliography{biblio}

\section*{Checklist}


\begin{enumerate}

\item For all authors...
\begin{enumerate}
  \item Do the main claims made in the abstract and introduction accurately reflect the paper's contributions and scope?
    \answerYes{}
  \item Did you describe the limitations of your work? 
    \answerYes{We specify that our focus is on the limited-data regime in both the abstract and introduction. See also the discussions on lines 143-151 and 207-213 on when our approximation guarantee is non-zero, and on lines 173-181 on the cost of our method.}
  \item Did you discuss any potential negative societal impacts of your work?
    \answerNA{}
  \item Have you read the ethics review guidelines and ensured that your paper conforms to them?
   \answerYes{}
\end{enumerate}

\item If you are including theoretical results...
\begin{enumerate}
  \item Did you state the full set of assumptions of all theoretical results?
  \answerYes{}
        \item Did you include complete proofs of all theoretical results?
  \answerYes{See \cref{app:proofs}}
\end{enumerate}

\item If you ran experiments...
\begin{enumerate}
  \item Did you include the code, data, and instructions needed to reproduce the main experimental results (either in the supplemental material or as a URL)?
  \answerYes{We include the code with instructions on how to reproduce all our experimental results in the supplemental material.}
  \item Did you specify all the training details (e.g., data splits, hyperparameters, how they were chosen)?
   \answerYes{See \cref{sec:exps} and \cref{app:setup}}
        \item Did you report error bars (e.g., with respect to the random seed after running experiments multiple times)?
\answerYes{}
        \item Did you include the total amount of compute and the type of resources used (e.g., type of GPUs, internal cluster, or cloud provider)?
  \answerYes{See \cref{app:setup}}
\end{enumerate}

\item If you are using existing assets (e.g., code, data, models) or curating/releasing new assets...
\begin{enumerate}
  \item If your work uses existing assets, did you cite the creators?
  \answerYes{See \cref{app:setup}}
  \item Did you mention the license of the assets?
\answerYes{The license of the assets we used are included in the code we provide.}
  \item Did you include any new assets either in the supplemental material or as a URL?
    \answerYes{We include our code in the supplemental material.}
  \item Did you discuss whether and how consent was obtained from people whose data you're using/curating?
     \answerNA{}
  \item Did you discuss whether the data you are using/curating contains personally identifiable information or offensive content?
     \answerNA{}
\end{enumerate}

\item If you used crowdsourcing or conducted research with human subjects...
\begin{enumerate}
  \item Did you include the full text of instructions given to participants and screenshots, if applicable?
     \answerNA{}
  \item Did you describe any potential participant risks, with links to Institutional Review Board (IRB) approvals, if applicable?
    \answerNA{}
  \item Did you include the estimated hourly wage paid to participants and the total amount spent on participant compensation?
     \answerNA{}
\end{enumerate}

\end{enumerate}


\appendix

\newpage
\input{appendix-relwork}
\input{appendix-proofs}

\input{appendix-subbounds}
\input{appendix-exps}

\end{document}

%% file: preamble.tex
\newcommand{\R}{\mathbb{R}}
\newcommand{\N}{\mathbb{N}}
\newcommand{\supp}{\mathrm{supp}}

\newcommand{\proj}{\mathrm{proj}}

\newcommand{\1}{\boldsymbol{1}}

\newcommand{\asymF}{\tilde{F}}
\newcommand{\asymG}{\tilde{G}}

\newcommand{\Greedy}{\textsc{Greedy}\xspace}
\newcommand{\OurSeq}{\textsc{SeqInChange}\xspace}
\newcommand{\OurAsym}{\textsc{AsymInChange}\xspace}
\newcommand{\OurLayer}{\textsc{LayerInChange}\xspace}
\newcommand{\Rand}{\textsc{Random}\xspace}
\newcommand{\LayerRand}{\textsc{LayerRandom}\xspace}

\newcommand{\LayerWeight}{\textsc{LayerWeightNorm}\xspace}
\newcommand{\ActGrad}{\textsc{ActGrad}\xspace}
\newcommand{\LayerActGrad}{\textsc{LayerActGrad}\xspace}
\newcommand{\LayerGreedyFS}{\textsc{LayerGreedyFS}\xspace}
\newcommand{\LayerSampling}{\textsc{LayerSampling}\xspace}

\DeclareMathOperator*{\argmax}{arg\,max}
\DeclareMathOperator*{\argmin}{arg\,min}

\graphicspath{{./figs/}}

\theoremstyle{plain}
\newtheorem{theorem}{Theorem}[section]
\newtheorem{proposition}[theorem]{Proposition}
\newtheorem{lemma}[theorem]{Lemma}
\newtheorem{corollary}[theorem]{Corollary}
\theoremstyle{definition}
\newtheorem{definition}[theorem]{Definition}

\theoremstyle{remark}

%% file: introduction.tex
\section{Introduction} \label{sec:Intro}

As modern neural networks (NN) grow increasingly large, with some models reaching billions of parameters \citep{McGuffie2020}, they require an increasingly large amount of memory, power, hardware, and inference time, which makes it necessary to compress them. This is especially important for models deployed on resource-constrained devices like mobile phones and smart speakers, and for latency-critical applications such as self-driving cars. 

Several approaches exist to compress NNs. 
Some methods approximate model weights using quantization and hashing \citep{gong2014compressing, courbariaux2015binaryconnect}, or
low-rank approximation and tensor factorization \citep{Denil2013Predicting,lebedev2015speeding,su2018tensorial}. In another class of methods called knowledge distillation, a small network is trained to mimic a much larger network \citep{bucila2006model,hinton2015distilling}. 
Other methods employ sparsity and group-sparsity regularisation during training, to induce sparse weights \citep{collins2014memory, voita2019analyzing}. 

In this work, we follow the network pruning approach, where the redundant units (weights, neurons or filters/channels) of a pre-trained NN are removed; see \citep{Kuzmin2019, Blalock2020, Hoefler2021} for recent surveys. We also focus on the limited-data regime, where only few training data is available and data labels are unavailable.
The advantage of pruning approaches is that, unlike weights approximation-based methods, they preserve the network structure, allowing retraining after compression, and unlike training-based approaches, they do not require training from scratch, which is costly and requires large training data.
It is also possible to combine different compression approaches to compound their benefits, see e.g., \citep[Section 4.3.4]{Kuzmin2019}.

Existing pruning methods fall into two main categories: unstructured pruning methods which prune individual weights leading to irregular sparsity patterns, and structured pruning methods which prune regular regions of weights, such as neurons, channels, or attention heads. 
Structured pruning methods are generally preferable as the resulting pruned models can work with off-the-shelf hardware or kernels, as opposed to models pruned with unstructured pruning which require specialized ones.

 \looseness=-1 The majority of existing structured pruning methods are heuristics that do not offer any theoretical guarantees. Moreover, most pruning methods are inapplicable in the limited-data regime, as they  
rely on fine-tuning with large training data for at least a few epochs to recover some of the accuracy lost with pruning.
\cite{Mariet2015} proposed a ``reweighting" procedure applicable to any pruning method, 
which optimize the remaining weights of the next layer to minimize the change in the input to the next layer.
Their empirical results on pruning single linear layers suggest that reweighting can provide a similar boost to performance as fine-tuning, without the need for data labels.  


\vspace{-5pt}
\paragraph{Our contributions} 
 \looseness=-1 We propose a principled data-efficient structured pruning method based on submodular optimization. In each layer, our method simultaneously 
selects neurons to prune and new weights for the next layer, that minimize the change in the next layer's input induced by pruning. The optimization with respect to the weights, for a \emph{fixed} selection of neurons, is the same one used for reweighting in \citep{Mariet2015}.
The resulting subset selection problem is intractable, but we show that it can be formulated as a weakly submodular maximization problem (see \cref{def:weaklySub}). We can thus use the standard greedy algorithm to obtain a $(1 - e^{-\gamma})$-approximation to the optimal solution, where $\gamma$ is non-zero if we use sufficient training data. 
We further adapt our method to prune any regular regions of weights;  we focus in particular on pruning channels in convolution layers. To prune multiple layers in the network, we apply our method to each layer independently or sequentially. 

We show that the error induced by pruning with our method on the model output decays with an $O(e^{-\gamma k})$ rate w.r.t the number $k$ of neurons/channels kept, under reasonable assumptions.
Our method uses only limited training data and no labels. Similar to \citep{Mariet2015}, we observe that reweighting provides a significant boost in performance not only to our method, but also to other baselines we consider. However unlike \citep{Mariet2015}, we only use a small fraction of the training data, around $\sim$ 1\% in our experiments. 
Our experimental results demonstrate that our method outperforms state-of-the-art pruning methods, even when reweighting is applied to them too, in the limited-data regime, and it is among the best performing methods in the standard setting.

\vspace{-5pt}
 \paragraph{Related work}  
\looseness=-1 A large variety of structured pruning approaches has been proposed in the literature, based on different selection schemes and algorithms to solve them. 
Some works prune neurons/channels individually based on some importance score \citep{He2014, Li2017, Liebenwein2020,Mussay2020,Mussay2021, Molchanov2017, Srinivas2015}.
Such methods are efficient and easy to implement, but they fail to capture higher-order interactions between the pruned parameters. Most do not provide any performance guarantee.  One exception are the sampling-based methods of \citep{Liebenwein2020,Mussay2020,Mussay2021}, who show an $O(1/k)$ error rate,
under some assumptions on the model activations.  

Closer to our approach are methods that aim to prune neurons/channels that minimize the change induced by pruning in the output of the layer being pruned, or its input to the next layer \citep{luo2017, He2017, Zhuang2018,Ye2020b}.  
These criteria yield an intractable combinatorial problem. 
Existing methods either use a heuristic greedy algorithm to solve it \citep{luo2017,Zhuang2018}, or they solve instead its $\ell_1$-relaxation 
using alternating minimization \citep{He2017}, or a greedy algorithm  with Frank-Wolfe like updates \citep{Ye2020b}.
Among these works only \citep{Ye2020b} provides theoretical guarantees, showing an $O(e^{-ck})$ error rate. Their method is more expensive than ours, and only optimize the scaling of the next layer weights instead of the weights themselves.  
A global variant of this method is proposed in 
\cite{Ye2020, Ye2020b}, which aim to prune neurons/channels that directly minimize the loss of the pruned network. 
A similar greedy algorithm with Frank-Wolfe like updates
is used to solve the $\ell_1$-relaxation of the selection problem. 
This method has an $O(1/k^2)$ error rate and is very expensive, 
as it requires a full forward pass through the network at each iteration. See \cref{app:relwork} for a more detailed comparison of our method with those of \citep{Ye2020, Ye2020b}.

\looseness=-1 \citet{Mariet2015} depart from the usual strategy of pruning parameters whose removal influences the network the least. They instead select a subset of diverse neurons to keep in each layer by sampling from a Determinantal Point Process, then they apply their reweighting procedure. Their experimental results
show that the advantage of their method is mostly due to reweighting (see Figure 4 therein).

%% file: preliminaries.tex
\section{Preliminaries}
\label{sec:Prelim}

We begin by introducing our notation and some relevant background from submodular optimization.
\paragraph{Notation:} 
Given a ground set $V = \{1,2, \cdots, d\}$ and a set function $F: 2^V \to \R_+$, we denote the \emph{marginal gain} of adding 
a set $I \subseteq V$ to another set $S \subseteq V$ by $F( I \mid S) =  F( S  \cup I) - F(S)$, which quantifies the change in value when adding $I$ to $S$. The cardinality of a set $S$ is written as $|S|$.
Given a vector $x \in \R^d$, we denote its support set by $\supp(x) = \{ i \in V | x_i \not = 0\}$, and its
$\ell_2$-norm by $\| x \|_2$.  Given a matrix $X \in \R^{d' \times d}$, we denote its $i$-th column by $X_i$, and its Frobenius norm by $\| X \|_F$. Given a set $S \subseteq V$, $X_S$ is the matrix with columns $X_i$ for all $i \in S$, and $0$ otherwise, and $\1_S$ is the indicator vector of $S$, with $[\1_S]_i = 1$ for all $i \in S$, and $0$ otherwise. 

\paragraph{Weakly submodular maximization:} 

\begin{algorithm}
   \caption{\Greedy}
\label{alg:greedy}
\begin{algorithmic}[1]
\STATE {\bfseries Input:} Ground set $V$, set function $F: 2^V \to \R_+$, budget $k \in \N_+$
\STATE $S \leftarrow \emptyset$
\WHILE{$|S|<k$}
\STATE $i^* \leftarrow \argmax_{i \in V \setminus S} F(i \mid S)$\label{l:max-gain}
\STATE $S \leftarrow S \cup \{i^*\}$
\ENDWHILE
\STATE {\bfseries Output:} $S$
\end{algorithmic}
\end{algorithm}

A set function $F$ is \emph{submodular} if it has diminishing marginal gains: $F( i \mid S) \geq F(i \mid T)$ for all $S \subseteq T$, $i \in V\setminus T$. We say that $F$ is \emph{normalized} if $F(\emptyset) = 0$, and non-decreasing if $F(S) \leq F(T)$ for all $S \subseteq T$. \\
Given a non-decreasing submodular function $F$, selecting a set $S \subseteq V$ with cardinality $|S| \leq k$ that maximize $F(S)$ can be done efficiently using the \Greedy algorithm (Alg. \ref{alg:greedy}). 
The returned solution is guaranteed to satisfy $F(\hat{S}) \geq (1 - 1/e) \max_{|S| \leq k} F(S)$ \citep{Nemhauser1978}. 
In general though maximizing a non-submodular function over a cardinality constraint is NP-Hard \citep{Natarajan1995}. However,  \citet{Das2011} introduced a notion of \emph{weak submodularity} which is sufficient to obtain a constant factor approximation with the  \Greedy algorithm.

\begin{definition}\label{def:weaklySub}
Given a set function $F: 2^V \to \R$, $U \subseteq V, k \in \N_+$, we say that $F$ is $\gamma_{U, k}$-weakly submodular, with $\gamma_{U, k} > 0$ if
$$\gamma_{U, k} F(S | L) \leq \sum_{i \in S} F(i | L),$$
for every two disjoint sets $L, S \subseteq V$, such that $L \subseteq U, |S| \leq k$.
\end{definition}

\looseness=-1 The parameter $\gamma_{U, k}$ is called the \emph{submodularity ratio} of $F$. It characterizes how close a set
function is to being submodular. If $F$ is  non-decreasing then $\gamma_{U, k} \in [0,1]$, and $F$ is submodular if and only if $\gamma_{U, k} = 1$ for all  $U \subseteq V, k \in \N_+$.
Given a non-decreasing $\gamma_{\hat{S}, k}$-weakly submodular function $F$, the Greedy algorithm is  guaranteed to return a solution $\hat{S}$ satisfying $F(\hat{S}) \geq (1 - e^{-\gamma_{\hat{S}, k}}) \max_{|S| \leq k} F(S)$ \citep{Elenberg2016, Das2011}. Hence, the closer $F$ is to being submodular, the better is the approximation guarantee.


%% file: onelayer-neurons.tex
\section{Reweighted input change pruning}\label{sec:onelayer-neurons}
 \looseness=-1 In this section, we introduce our approach for pruning neurons in a single layer. 
Given a large pre-trained NN, $n$ training data samples, and a layer $\ell$ with $n_\ell$ neurons, our goal is to select a small number $k$ out of the $n_\ell$ neurons to keep, and prune the rest, in a way that influences the network the least. One way to achieve this is by minimizing the change in input to the next layer $\ell+1$, induced by pruning. However, simply throwing away the activations from the dropped neurons is wasteful. Instead, we optimize the weights of the next layer to reconstruct the inputs from the remaining neurons. 

Formally, let $A^\ell \in \R^{n \times n_{\ell}}$ be the activation matrix of layer $\ell$ with columns $a^\ell_1, \cdots, a^\ell_{n_\ell}$, where $a^\ell_i \in \R^{n}$ is the vector of activations of the $i$th neuron in layer $\ell$ for each training input, and let $W^{\ell+1} \in \R^{n_\ell \times n_{\ell+1}}$ be the weight matrix of layer $\ell+1$ with columns $w^{\ell+1}_1, \cdots, w^{\ell+1}_{n_{\ell+1}}$, where $w^{\ell+1}_i \in  \R^{n_{\ell}}$ is the vector of weights connecting the $i$th neuron in layer $\ell+1$ to the neurons in layer $\ell$. 
When a neuron is pruned in layer $\ell$, the corresponding column of weights in $W^\ell$ and row in $W^{\ell+1}$ are removed. Pruning $n_\ell - k$ neurons in layer $\ell$ reduces the number of parameters and computation cost by $(n_\ell - k) / n_\ell$ for both layer $\ell$ and $\ell+1$.

Let $V_\ell = \{1, \cdots, n_\ell\}$.  Given a set $S \subseteq V_\ell$, we denote by $A^\ell_S$ the matrix with columns $a^\ell_i$ for all $i \in S$, and $0$ otherwise. That is, $A^\ell_S$ is the activation matrix of layer $\ell$ after pruning.
We choose a set of neurons $S \subseteq V_\ell$ to keep  and new weights $\tilde{W}^{\ell+1} \in \R^{n_\ell \times n_{\ell+1}}$ that minimize:
\begin{equation}\label{eq:optReweighting}
\min_{|S| \leq k, \tilde{W}^{\ell+1} \in \R^{n_\ell \times n_{\ell+1}}} \| A^\ell W^{\ell+1} - A^\ell_S \tilde{W}^{\ell+1}  \|_F^2 
\end{equation}
Note that $A^\ell W^{\ell+1}$ are the original inputs of layer $l+1$, and $A^\ell_S \tilde{W}^{\ell+1}$ are the inputs after pruning and reweighting, i.e., replacing the weights $W^{\ell+1}$ of layer $\ell+1$ with the new weights $\tilde{W}^{\ell+1}$. 

\subsection{Greedy selection}\label{sec:Greedy}

Solving Problem \eqref{eq:optReweighting} exactly is NP-Hard \citep{Natarajan1995}. However, we show below that it can be formulated as a weakly submodular maximization problem, hence it can be efficiently approximated. Let 
\begin{equation}\label{eq:F}
F(S) =  \| A^\ell W^{\ell+1}\|_F^2 - \min_{\tilde{W}^{\ell+1}} \| A^\ell W^{\ell+1} - A^\ell_S \tilde{W}^{\ell+1}  \|_F^2,
\end{equation}
then Problem \eqref{eq:optReweighting} is equivalent to $\max_{|S| \leq k} F(S)$.

\begin{restatable}{proposition}{weaklySub}\label{prop:WeaklySub}
Given $U \subseteq V, k \in \N_+$, $F$ is a normalized non-decreasing $\gamma_{U, k}$-weakly submodular function, with $$\gamma_{U, k} \geq \frac{\min_{\| z \|_2 = 1, \| z \|_0 \leq |U|+k} \|A^\ell z\|_2^2}{\max_{\| z \|_2 = 1, \| z \|_0 \leq |U|+1} \|A^\ell z\|^2_2}. $$ 
\end{restatable}
The proof of \cref{prop:WeaklySub} follows by writing $F$ as the sum of $n_{\ell+1}$ sparse linear regression problems $F(S) =\sum_{m=1}^{n_{\ell+1}} \| A^\ell w^{\ell+1}_m\|^2_2 - \min_{\supp(\tilde{w}_m) \subseteq S} \| A^\ell w^{\ell+1}_m - A^\ell  \tilde{w}_m\|^2_2$, and from the relation established in \citep{Elenberg2016,Das2011} between weak submodularity and sparse eigenvalues of the covariance matrix (see \cref{app:WeakSub-proofs}).

We use the \Greedy algorithm to select a set $\hat{S} \subseteq V_\ell$ of $k$ neurons to keep in layer $\ell$.
As discussed in \cref{sec:Prelim}, the returned solution is guaranteed to satisfy 
\begin{equation}\label{eq:approx}
F(\hat{S}) \geq (1 - e^{-\gamma_{\hat{S}, k}}) \max_{|S| \leq k} F(S)
\end{equation}
Computing the lower bound on the submodularity ratio $\gamma_{\hat{S}, k}$ in \cref{prop:WeaklySub} is NP-Hard \citep{Das2011}. It is non-zero if any $\min\{2 k, n_\ell\}$ columns of $A^\ell$ are linearly independent.
If the number of training data is larger than the number of neurons, i.e., $n > n_\ell$, this is likely to be satisfied. 
We verify that this is indeed the case in our experiments in \cref{app:gammaValues}. We also discuss the tightness of the lower bound in \cref{app:tightness}.

We show in \cref{app:weakDRsub} that $F$ satisfies an even stronger notion of approximate submodularity than weak submodularity, which implies a better approximation guarantee for \Greedy than the one provided in Eq. \eqref{eq:approx}. Though, this requires a stronger assumption: any $k+1$ columns of $A^\ell$ should be linearly independent and all rows of $W^{\ell+1}$ should be linearly independent. In particular, we would need that $n_\ell \leq n_{\ell+1}$, which is not always satisfied.

In \cref{sec:errorbds}, we show that the approximation guarantee of Greedy implies an exponentially decreasing 
bound on the layerwise error, and on the final output error under a mild assumption. 

\subsection{Reweighting}\label{sec:RW}

For a fixed $S \subseteq V_\ell$, the reweighted input change $ \| A^\ell W^{\ell+1} - A^\ell_S \tilde{W}^{\ell+1}  \|_F^2$ is minimized by setting 
\begin{equation}\label{eq:optW}
\tilde{W}^{\ell+1} = x^S(A^{\ell}) W^{\ell+1},
\end{equation}
where $x^S(A^{\ell}) \in \R^{n_\ell \times n_{\ell}}$ is the matrix with columns $x^S(a^\ell_j)$ such that
\begin{align}\label{eq:proj}
x^S(a^\ell_j) \in \argmin_{\supp(x) \subseteq S} \|a_j^\ell - A x \| _2^2 \text{ for all $j \in V_\ell$}.
\end{align}

\looseness=-1 Note that 
the new weights are given by $\tilde{w}^{\ell+1}_{i m} = w^{\ell+1}_{i m} + \sum_{j \not \in S} [x^S(A^{\ell})]_{ij}  w_{j m}^{\ell+1}$ for all $i \in S$, and $\tilde{w}^{\ell+1}_{i m} = 0$ for all $i \not \in S, m \in V_{\ell+1}$. Namely, the new weights merge the weights from the dropped neurons into the kept ones.
\mtodo{Remove the above observation if we need the space}
This is the same reweighting procedure introduced in \citep{Mariet2015}. But instead of applying it only at the end to the selected neurons $\hat{S}$, it is implicitly done at each iteration of our pruning method, as it is required to evaluate $F$. We discuss next   how this can be done efficiently. 

\subsection{Cost}\label{sec:cost}
Each iteration of \Greedy requires $O(n_\ell)$ function evaluations of $F$. 
Computing $F(S)$ from scratch needs $O(k \cdot ( n_\ell \cdot n_{\ell+1} + n \cdot (n_\ell + n_{\ell+1}))$ time, so a naive implementation of \Greedy is too expensive. 
The following Proposition outlines how we can efficiently evaluate $F(S + i)$ 
given that $F(S)$ was computed in the previous iteration.

\begin{restatable}{proposition}{cost}\label{prop:cost}
Given $S \subseteq V_\ell$ such that $|S| \leq k$, $i \not \in S$,  let $\proj_S(a_j^\ell) = A^\ell_S x^S(a_j^\ell)$ be the projection of $a_j^\ell$ onto the column space of $A^\ell_S$,
$R_S(a_i^\ell) = a_i^\ell - \proj_S(a_i^\ell)$ and $\proj_{R_S(a_i^\ell)}(a_j^\ell) \in \argmin_{z =  R_S(a_i^\ell) \gamma , \gamma \in \R} \| a_j^\ell - z\|_2^2$ the corresponding residual and the projection of $a_j^\ell$ onto it. 
We can write 
$$F(i | S)= \sum_{m=1}^{n_{\ell+1}}  \|  \proj_{R_S(a_i^\ell)}(A^{\ell}_{V \setminus S}) w_m^{\ell+1} \|_2^2, $$ where
$\proj_{R_S(a_i)}(A^{\ell}_{V \setminus S}) $ is the matrix with columns $\proj_{R_S(a_i)}(a_j^\ell)$ for all $j \not \in S$, $0$ otherwise.
Assuming $F(S), \proj_S(a_j^\ell)$ and $x^S(a_j^\ell)$ for all  $j  \not \in S$ 
were computed in the previous iteration, we can compute $F(S+i), \proj_{S+i}(a_j^\ell)$ and $x^{S+i}(a_j^\ell)$ for all $j \not \in (S + i)$ 
in $$O(n_{\ell}\cdot (n_{\ell+1} + n + k)) \text{ time}.$$
 \looseness=-1 The optimal weights in Eq. \eqref{eq:optW} can then be computed in $O(k \cdot  n_{\ell}\cdot n_{\ell+1})$ time, at the end of \Greedy.
\end{restatable}
The proof is given in \cref{app:cost-proofs}, and relies on using optimality conditions to construct the least squares solution $x^{S+i}(a_j^\ell)$ from $x^S(a_j^\ell)$.\\
In total \Greedy 's runtime is then $O(k \cdot(n_{\ell})^2 \cdot (n_{\ell+1} + n + k))$. 
In other words, our pruning method costs as much as $O(k)$ forward passes in layer $\ell+1$ with a batch of size $n$ (assuming $n_{\ell+1} = O(n_\ell)$).  
Using a faster variant of \Greedy, called \textsc{Stochastic-Greedy} \citep{Mirzasoleiman2015}, further reduces the cost to $O(\log(1/\epsilon) \cdot (n_{\ell})^2 \cdot (n_{\ell+1} + n + k))$, or equivalently $O(\log(1/\epsilon))$ forward passes in layer $\ell+1$ with a batch of size $n$,  while maintaining almost the same approximation guarantee $(1 - e^{-\gamma_{\hat{S}, k}} - \epsilon)$  in expectation. \footnote{\citet{Mirzasoleiman2015} only consider submodular functions, but it is straighforward to extend their result to weakly submodular functions \cref{app:StochasticGreedy}.} 

Note also that computing the solutions for different budgets $k' \leq k$ can be done at the cost of one by running \Greedy with budget $k$. Our method is more expensive than methods which prune neurons individually \citep{He2014, Li2017, Liebenwein2020,Mussay2020,Mussay2021, Molchanov2017, Srinivas2015}, 
but much less expensive than a loss-based method like \citep{Ye2020, Ye2020b}, which requires 
$O(k)$ forward passes in the full network, for each layer.


%% file: onelayer-channels.tex
\section{Pruning regular regions of neurons}\label{sec:onelayer-channels}
In this section, we discuss how to adapt our approach to pruning regular regions of neurons. 
This is easily achieved by mapping any set of regular regions to the corresponding set of neurons, then applying the same method in \cref{sec:onelayer-neurons}.
In particular, we focus on pruning channels in CNNs.  

\looseness=-1 Given a layer $\ell$ with $n_\ell$ output channels, let 
$X^\ell \in \R^{n \cdot p_\ell \times n_\ell \times r_h \times r_w}$ be its activations for each output channel and training input, where $p_\ell$ is number of patches obtained by applying a filter of size $r_h \times r_w$, and let $F^{\ell+1} \in \R^{n_{\ell+1} \times n_\ell \times  r_h \times r_w}$ be the weights of layer $\ell+1$, corresponding to $n_\ell$ filters of size $r_h \times r_w$ for each of its output channels.
When an output channel is pruned in layer $\ell$, the corresponding 
weights in $F^\ell$ and $F^{\ell+1}$ are removed. Pruning $n_\ell - k$ output channels in layer $\ell$ reduces the number of parameters and computation cost by $(n_\ell - k) / n_\ell$ for both layer $\ell$ and $\ell+1$. 
If layer $\ell$ is followed by a batch norm layer, the weights therein corresponding to the pruned channels are also removed.
 
We arrange the activations $X^\ell_c \in \R^{n \cdot p_\ell \times  r_h \cdot r_w}$ of each channel $c$ into $r_h  r_w$ columns of $A^\ell \in \R^{n \cdot p_\ell \times n_\ell \cdot r_h \cdot r_w}$, i.e., $A^\ell = [X^\ell_1, \cdots, X^\ell_{n_\ell}]$. Similarly, we arrange the weights $F_c^{\ell+1} \in \R^{n_{\ell+1} \times  r_h \times r_w}$ of each channel $c$ into $r_h \cdot r_w$ rows of $W^{\ell+1} \in \R^{n_\ell \cdot r_h \cdot r_w \times n_{\ell+1} }$, i.e., $(W^{\ell+1})^\top =  [(F^\ell_1)^\top, \cdots, (F^\ell_{n_\ell})^\top]$. 
Recall that $V_{\ell} = \{1, \cdots, n_\ell\}$, and let  $V'_{\ell} = \{1, \cdots, r_h  r_w  n_\ell\}$. 
We define a function $M: 2^{V_{\ell}} \to 2^{V'_{\ell}}$ which maps every channel $c$ to its corresponding $r_h r_w$ columns in $A^\ell$.
Let $G(S) = F(M(S))$, with $F$ defined in Eq. \eqref{eq:F}, then minimizing the reweighted input change $ \| A^\ell W^{\ell+1} - A^\ell_{M(S)} \tilde{W}^{\ell+1}  \|_F^2$ with a budget $k$ is equivalent to $\max_{|S| \leq k} G(S)$. The following proposition shows that this remains a weakly submodular maximization problem.

\begin{restatable}{proposition}{weaklySubChannels}\label{prop:WeaklySubChannels}
Given $U \subseteq V_{\ell}, k \in \N_+$, 
$G$ is a normalized non-decreasing $\gamma_{U, k}$-weakly submodular function, with $$\gamma_{U, k} \geq \frac{\min_{\| z \|_2 = 1, \| z \|_0 \leq r_h r_w(|U|+k)} \|A^\ell z\|_2^2}{\max_{\| z \|_2 = 1, \| z \|_0 \leq r_h r_w (|U|+ 1) } \|A^\ell z\|^2_2}.$$
\end{restatable}
\begin{proof}[Proof sketch]
$G$ is $\gamma_{U, k}$-weakly submodular iff $F$ satisfies 
$\gamma_{U, k} F(M(S) | M(L)) \leq \sum_{i \in S} F(M(i) | M(L)),$
for every two disjoint sets $L, S \subseteq V_{\ell}$, such that $L \subseteq U, |S| \leq k$. The proof follows by extending the relation established in \citep{Elenberg2016,Das2011} between weak submodularity and sparse eigenvalues of the covariance matrix to this case. 
\end{proof}

\looseness=-1 As before, we use the \Greedy algorithm, with function $G$, to select a set $\hat{S} \subseteq V_\ell$ of $k$ channels to keep in layer $\ell$. We get the same approximation guarantee $G(\hat{S}) \geq (1 - e^{-\gamma_{\hat{S}, k}}) \max_{|S| \leq k} G(S).$ The submodularity ratio $\gamma_{\hat{S}, k}$ is non-zero if any $\min\{2 k, n_\ell \} r_h r_w$ columns of $A^\ell$ are linearly independent.
In our experiments, we observe that in certain layers linear independence only holds for $k$ very small, e.g., $k \leq 0.01 n_\ell$. This is due to the correlation between patches which overlap. To remedy this, we experimented with using only $r_h r_w$ random patches from each image, instead of using all patches. This indeed raises the rank of $A^\ell$, but certain layers have a very small feature map size so that even the small number of random patches have significant overlap, resulting in still a very small range where linear independence holds, e.g., $k \leq 0.08 n_\ell$ (see \cref{app:gammaValues} for more details). 
The results obtained with random patches were worst than the ones with all patches, we thus omit them.  Note that our lower bounds on $\gamma_{\hat{S}, k}$ are not necessarily tight (see \cref{app:tightness}). Hence, having linear dependence does not necessarily imply that  $\gamma_{\hat{S}, k}=0$; our method still performs well in these cases. 
 
 

For a fixed $S \subseteq V_{\ell}$, the optimal weights are again given by $\tilde{W}^{\ell+1} = x^{M(S)}(A^{\ell}) W^{\ell+1}$.  The cost of running \Greedy and reweighting is the same as before (see \cref{app:cost-proofs}). 

%% file: multiple-layers.tex
\section{Pruning multiple layers}
In this section, we explain how to apply our pruning method to prune multiple layers of a NN. 

\subsection{Reweighted input change pruning variants}\label{sec:variants}
\looseness=-1 We consider three variants of our method:  \OurLayer, \OurSeq, and \OurAsym. 
In \OurLayer, we prune each layer independently, i.e., we apply exactly the method in Section \ref{sec:onelayer-neurons} or \ref{sec:onelayer-channels}, according to the layer's type. 
This is the fastest variant; it has the same cost as pruning a single layer, as each layer can be pruned in parallel, and it only requires one forward pass to get the activations of all layers. However, it does not take into account the effect of pruning one layer on subsequent layers. 

In \OurSeq, we prune each layer sequentially, starting from the earliest layer to the latest one. For each layer $\ell$, we apply our method with $A^\ell$ replaced by the \emph{updated} activations $B^\ell$ after having pruned previous layers, i.e., we solve $\min_{|S| \leq k, \tilde{W}^{\ell+1} \in \R^{n_\ell \times n_{\ell+1}}} \| B^\ell W^{\ell+1} -  B^\ell_S \tilde{W}^{\ell+1}\|_F^2 $. 
In \OurAsym, we also prune each layer sequentially, but
to avoid the accumulation of error, 
we use an asymmetric formulation of the reweighted input change, where instead of approximating the \emph{updated} input $B^\ell W^{\ell+1}$, 
we approximate the \emph{original} input $A^\ell W^{\ell+1}$,  i.e., we solve $\min_{|S| \leq k, \tilde{W}^{\ell+1} \in \R^{n_\ell \times n_{\ell+1}}} \| A^\ell W^{\ell+1} -  B^\ell_S \tilde{W}^{\ell+1}\|_F^2 $. This problem is still a weakly submodular maximization problem, with the same submodularity ratio given in Propositions \ref{prop:WeaklySub} and \ref{prop:WeaklySubChannels}, with $A^\ell$ replaced by $B^\ell$ therein (see \cref{app:WeakSub-proofs}). Hence, the same approximation guarantee as in the symmetric formulation holds here. Moreover, a better approximation guarantee can again be obtained 
under stronger assumptions (see \cref{app:weakDRsub}).
The cost of running \Greedy with the asymmetric formulation  and reweighting is also the same as before (see \cref{app:cost-proofs}).


In \cref{sec:errorbds}, we show that the sequential variants of our method both have an exponential error rate, which is faster for the asymmetric variant. We evaluate all three variants in our experiments. As expected, \OurAsym usually performs the best, and \OurLayer the worst. 

\subsection{Per-layer budget selection}\label{sec:fractionSel}
Another important design choice is how much to prune in each layer, given a desired global compression ratio (see \cref{app:fractionSel} for the effect of this choice on performance). 
In our experiments, we use the budget selection method introduced in \citep[Section 3.4.1]{Kuzmin2019}, which can be applied to any layerwise pruning method, thus enabling us to have a fair comparison.

Given a network with $L$ layers to prune, let $c = \tfrac{\text{original size}}{\text{pruned size}}$ be the desired compression ratio.
We want to select for each layer $\ell$, the number of neurons/channels $k_\ell = \alpha_{\ell} n_\ell$ to keep, with $\alpha_{\ell}$ chosen from a fixed set of possible values, e.g., $\alpha_\ell \in \{0.05, 0.1, \cdots, 1\}$. We define a layerwise accuracy metric $P_\ell(k_\ell)$ as the accuracy obtained after pruning layer $\ell$, with a budget $k_{\ell}$, while other layers are kept intact, evaluated on a verification set. We set aside a subset of the training set to use as a verification set. Let $P_{\text{orig}}$ be the original model accuracy, $C_{\text{orig}}$ the original model size, and $C(k_1, \cdots, k_L)$ the pruned model size. We select the per-layer budgets that minimize the per-layer accuracy drop while satisfying the required compression ratio:
\begin{align}\label{eq:BudgetSelProb}
\min_{k_1, \cdots, k_L} \{ \tau : \forall \ell \in [L], P_\ell(k_\ell) \geq P_{\text{orig}} - \tau,   C(k_1, \cdots, k_L) \leq C_{\text{orig}} / c \}.
\end{align} 
We can solve the selection problem \eqref{eq:BudgetSelProb} using binary search, if the layerwise accuracy $P_\ell(k_\ell)$ is a non-decreasing function of $k_\ell$. Empirically, this is not always the case, the general trend is non-decreasing, but some fluctuations occur. In such cases, we use interpolation to ensure monotonicity. 

Alternatively, another simple strategy is to prune each layer until the perlayer error (the reweighted input change in our case) reaches some threshold $\epsilon$, and vary $\epsilon$ to obtain the desired compression ratio, as done in \citep{Zhuang2018, Ye2020}. 

%% file: error-bounds.tex
\section{Error convergence rate}\label{sec:errorbds}

In this section, we provide the error rate of our proposed method. The omitted proofs are given in \cref{app:errorbds}. 
 We first show that the change in input to the next layer induced by pruning with our method, with both the symmetric and asymmetric formulation, decays with exponentially fast rate.
\begin{restatable}{proposition}{layerError}\label{prop:layerError}
Let $\hat{S}$ be the output of the \Greedy algorithm and $\hat{W}^{\ell+1}$ the corresponding optimal weights (Eq. \eqref{eq:optW}), then 
$$\| A^\ell W^{\ell+1} - A^\ell_{\hat{S}} \hat{W}^{\ell+1}  \|_F^2 \leq e^{- \gamma_{\hat{S}, n_\ell} {k}/{n_\ell}} \| A^\ell W^{\ell+1}\|_F^2,$$
and
$$\| A^\ell W^{\ell+1} - B^\ell_{\hat{S}} \hat{W}^{\ell+1}  \|_F^2 \leq e^{- \gamma_{\hat{S}, n_\ell} {k}/{n_\ell}} \| A^\ell W^{\ell+1}\|_F^2 + (1 - e^{- \gamma_{\hat{S}, n_\ell} {k}/{n_\ell}}) \min_{\tilde{W}^{\ell+1} \in \R^{n_\ell \times n_{\ell+1}}} \| A^\ell W^{\ell+1} - B^\ell \tilde{W}^{\ell+1}  \|_F^2 $$
\end{restatable}
This follows by extending the approximation guarantee of \Greedy in \citep{Elenberg2016, Das2011} to $F(\hat{S}) \geq (1 - e^{- \gamma_{\hat{S}, n_\ell} {k}/{n_\ell}}) \max_{|S|\leq n_\ell} F(S)$. Note that this bounds uses the submodularity ratio $\gamma_{\hat{S}, n_\ell}$, for which the lower bound in \cref{prop:WeaklySub} is non-zero only if \emph{all} columns of $A^\ell$ are linearly independent, which is more restrictive.  Though as discussed earlier, this bound is not necessarily tight.
We can further extend this exponential layerwise error rate to an exponentially rate on the final output error, if we assume as in \citep{Ye2020b} that the function corresponding to all layers coming after layer $\ell$ is Lipschitz continuous. 

\begin{corollary}\label{corr:globalError-onelayer}
Let $y \in \R^n$ be the original model output, $y^{\hat{S}} \in \R^n$ the output after layer $\ell$ is pruned using our method, and $H$ the function corresponding to all layers coming after layer $\ell$, i.e., $y = H(A^\ell W^{\ell+1}),  y^{\hat{S}} = H(A^\ell_{\hat{S}} \hat{W}^{\ell+1})$. 
If $H$ is Lipschitz continuous with constant $\| H\|_{\text{Lip}}$, then  
 $$\| y - y^{\hat{S}}\|_2^2 \leq e^{- \gamma_{\hat{S}, n_\ell} {k}/{n_\ell}} \| H\|_{\text{Lip}}^2  \| A^\ell W^{\ell+1}\|_F^2.$$
\end{corollary}
\begin{proof}
Since $H$ is Lipschitz continuous, we have 
$\| y - y^{\hat{S}}\|_2^2 \leq \| H\|_{\text{Lip}}^2 \| A^\ell W^{\ell+1} - A^\ell_{\hat{S}} \hat{W}^{\ell+1}  \|_F^2$. The claim then follows from \cref{prop:layerError}.
\end{proof}
This matches the exponential convergence rate achieved by the local imitation method in \citep[Theorem 1]{Ye2020b}, albeit with a different constant. Under the same assumption, we can show that pruning multiple layers with the sequential variants of our method, \OurSeq and  \OurAsym, also admits an exponential convergence rate:

\begin{restatable}{corollary}{globalError}\label{corr:globalError-multiplelayers}
Let $y \in \R^n$ be the original model output, $y^{\hat{S}_\ell}, {y}^{\tilde{S}_\ell} \in \R^n$ the outputs after layers $1$ to $\ell$ are sequentially pruned using \OurSeq and  \OurAsym, respectively, and $H_\ell$ the function corresponding to all (unpruned) layers coming after layer $\ell$. 
If every function $H_\ell$ is Lipschitz continuous with constant $\| H_\ell \|_{\text{Lip}}$, then  
 $$\| y - y^{\hat{S}_L}\|_2^2 \leq \sum_{\ell=1}^L e^{- \gamma_{\hat{S}_\ell, n_\ell} {k_\ell}/{n_\ell}} \| H_\ell\|_{\text{Lip}}^2  \| A^\ell W^{\ell+1}\|_F^2,$$
 and \vspace{-15pt}
$$\| y - y^{\tilde{S}_L}\|_2^2 \leq \sum_{\ell=1}^L \prod_{\ell'=\ell+1}^L (1- e^{- \gamma_{\tilde{S}_{\ell'}, n_{\ell'}} {k_{\ell'}}/{n_{\ell'}}}) e^{- \gamma_{\tilde{S}_\ell, n_\ell} {k_\ell}/{n_\ell}} \| H_\ell\|_{\text{Lip}}^2  \| A^\ell W^{\ell+1}\|_F^2.$$
\end{restatable}
The result is obtained by iteratively applying \cref{prop:layerError} to the error incurred after each layer is pruned. The rate of \OurSeq matches the exponential convergence rate achieved by the local imitation method in \citep[Theorem 6]{Ye2020b}. The bound on \OurAsym is stronger, confirming that the asymmetric formulation indeed reduces the accumulation of errors.

%% file: experiments.tex
\section{Empirical Evaluation}\label{sec:exps}
\mtodo{check shrinkbench paper for recommendations for good experimental results}
In this section, we examine the performance of our proposed pruning method in the limited-data regime. To that end, we focus on 
one-shot pruning, in which a pre-trained model is compressed in a single step, without any fine-tuning. We study the effect of fine-tuning with both limited and sufficient data in \cref{app:finetuning}.
We compare the three variants of our method, \OurLayer, \OurSeq, and \OurAsym, with the following baselines:
\begin{itemize}[leftmargin=1em, itemindent=1em]
\item \looseness=-1 \LayerGreedyFS \citep{Ye2020}:  for each layer, 
first removes all neurons/channels in that layer, then 
gradually adds back the neuron/channel that yields the largest decrease of the
loss, evaluated on one batch of training data. Layers are pruned sequentially from the input to the output layer. 
\item \looseness=-1 \LayerSampling \citep{Liebenwein2020}: samples neurons/channels, in each layer, with probabilities proportional to sensitivities based on (activations $\times$ weights), and prunes the rest. 

\item \ActGrad \citep{Molchanov2017}: prunes neurons/channels with the lowest (activations $\times$ gradients), averaged over the training data, with layerwise $\ell_2$-normalization. 
\item \LayerActGrad: prunes neurons/channels with the lowest (activations $\times$ gradients), averaged over the training data, in each layer. This is the layerwise variant of \ActGrad.

\item \LayerWeight \citep{Li2017}: prunes neurons/channels with the lowest output weights $\ell_1$-norm, in each layer. 

\item \Rand: prunes randomly selected neurons/channels globally across layers in the network.
\item \LayerRand: prunes randomly selected neurons/channels in each layer.
\end{itemize}

\looseness=-1 We also considered the global variant of \LayerWeight proposed in \citep{He2014}, but we exclude it from plots, as it is always the worst performing method. We evaluate the performance of these methods on the LeNet model \citep{LeCun1989} on the MNIST dataset \citep{Lecun1998}, and on the ResNet56 \citep{He2016} and the VGG11 \citep{Simonyan2015} models  on the CIFAR-10 dataset \citep{Krizhevsky2009}. 
To ensure a fair comparison, all experiments are based on our own implementation of 
all the compared methods.  
To compute the gradients and activations used for pruning in \LayerSampling, \ActGrad, \LayerActGrad, and our method's variants, we use four batches of $128$ training images, i.e.,  $n= 512$, which corresponds to $\sim 1 \%$ of the training data in MNIST and CIFAR10. 
We consider two variants of the method proposed in \citep{Ye2020}: a limited-data variant \LayerGreedyFS which only uses the same four batches of data used in our method, and a full-data variant \LayerGreedyFS-fd with access to the full training data.

\looseness=-1 We report top-1 accuracy results evaluated on the validation set, as we vary the compression ratio ($\tfrac{\text{original size}}{\text{pruned size}}$). 
 Unless otherwise specified, we use the per-layer budget selection method described in \cref{sec:fractionSel} for all the layerwise pruning methods, except for \LayerSampling for which we use its own budget selection strategy provided in \citep{Liebenwein2020}. We use a subset of the training set, of the same size as the validation set, as a verification set for the budget selection method. 
To disentangle the benefit of using our pruning method from the benefit of reweighting (\cref{sec:RW}), we report results with reweighting applied to all pruning methods, or none of them. Though, we will focus our analysis on the more interesting results with reweighting, with the plots without reweighting mostly serving as a demonstration of the benefit of reweighting.
Results are averaged over five random runs, with standard deviations plotted as error bars. 
We report the speedup ($\tfrac{\text{original number of FLOPs}}{\text{pruned number of FLOPs}}$) and pruning time values in \cref{app:othermetrics}.  
For additional details on the experimental set-up, see \cref{app:setup}. The code for reproducing all experiments is 
available at \url{https://github.com/marwash25/subpruning}.

\paragraph{LeNet on MNIST}\label{sec:LeNet-MNIST}
\looseness=-1 We pre-train LeNet model on MNIST 
achieving $97.75\%$ top-1 accuracy. We prune all layers except the last classifier layer. Results are presented in Figure \ref{fig:oneshot} (left).
All three variants of our method consistently outperform other baselines, even when reweighting is applied to them, with \OurAsym doing the best and \OurLayer the worst. We observe that reweighting significantly improves the performance of all methods except \LayerGreedyFS variants.

\paragraph{ResNet56 on CIFAR-10}\label{sec:ResNet56-CIFAR10}
\looseness=-1  We use the ResNet56 model pre-trained on CIFAR-10 provided in ShrinkBench \citep{Blalock2020}, which achieves $92.27\%$ top-1 accuracy. We prune all layers except the last layer in each residual branch, the last layer before each residual branch, and the last classifier layer. 
Results are presented in Figure \ref{fig:oneshot} (middle).
The sequential variants of our method perform the best. Their performance is closely matched by  \LayerWeight and \ActGrad (with reweighting) for most compression ratios, except very large ones.  
\OurLayer performs significantly worst here than the sequential variants of our method. This is likely due to the larger number of layers pruned in ResNet56 compared to LeNet (27 vs 4 layers), which increases the effect of pruning earlier layers on subsequent ones. 
Here also reweighting improves the performance of all methods except the \LayerGreedyFS variants.

\paragraph{VGG11 on CIFAR-10}\label{sec:VGG11-CIFAR10}
\looseness=-1  We pre-train VGG11 model on CIFAR-10 
obtaining $90.11\%$ top-1 accuracy. We prune all layers except the last features layer and the last classifier layer. 
Results are presented in Figure \ref{fig:oneshot} (right).
The three variants of our method perform the best. Their performance is matched by \ActGrad and \LayerWeight (with reweighting). 
\OurLayer performs similarly to the sequential variants of our method here, even slightly better at compression ratio $32$, probably because the number of layers being pruned is again relatively small (9 layers). As before, reweighting benefits all methods except the \LayerGreedyFS variants.

\begin{figure}
\vspace{-25pt}
\begin{subfigure}{.31\textwidth}
 \centering
  \hspace*{-5pt}
\includegraphics[trim=65 50 1410 35, clip, scale=0.1]{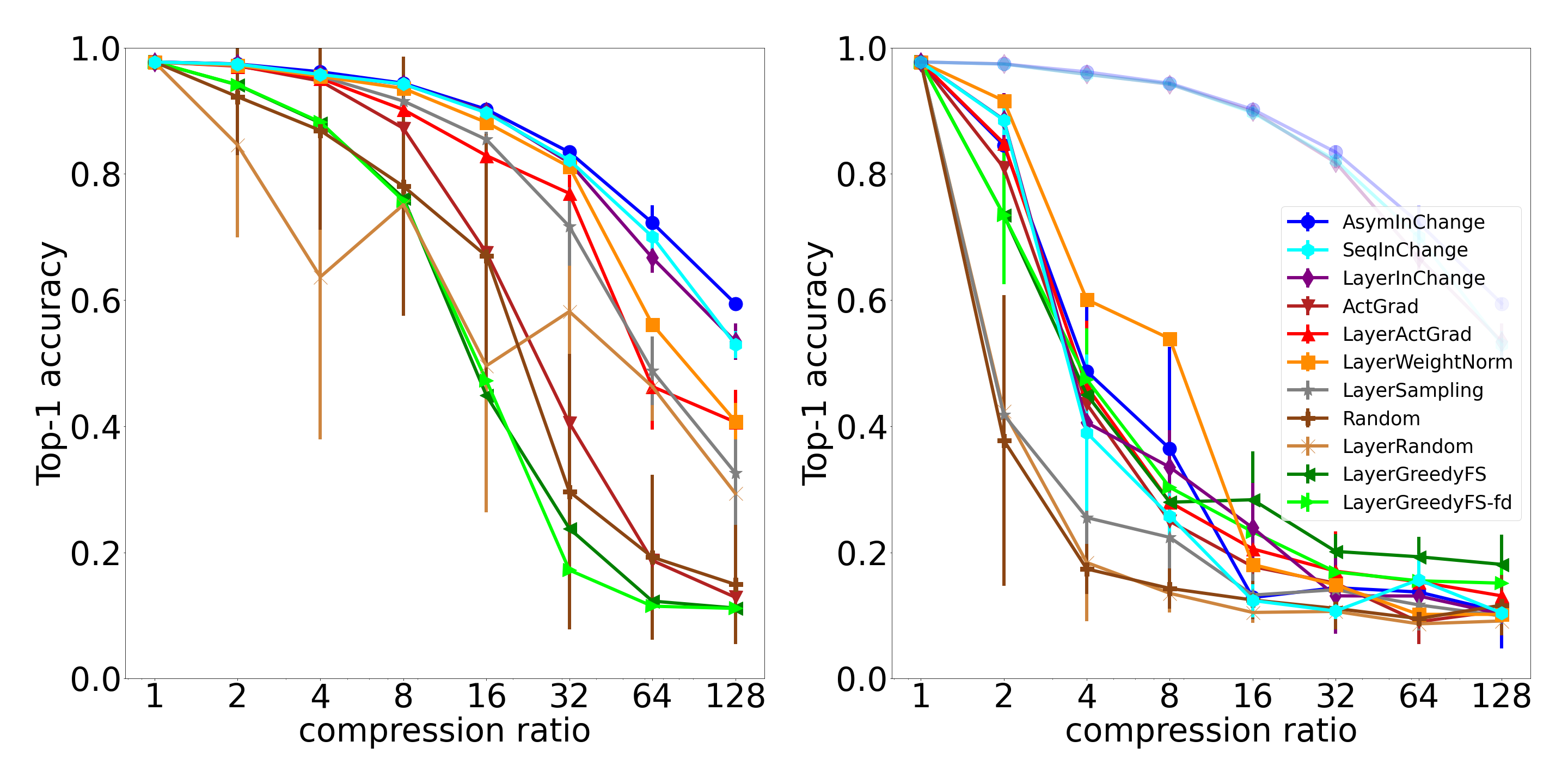}
\end{subfigure}\hfill
\begin{subfigure}{.31\textwidth}
 \centering
\includegraphics[trim=65 50 1410 35, clip, scale=0.1]{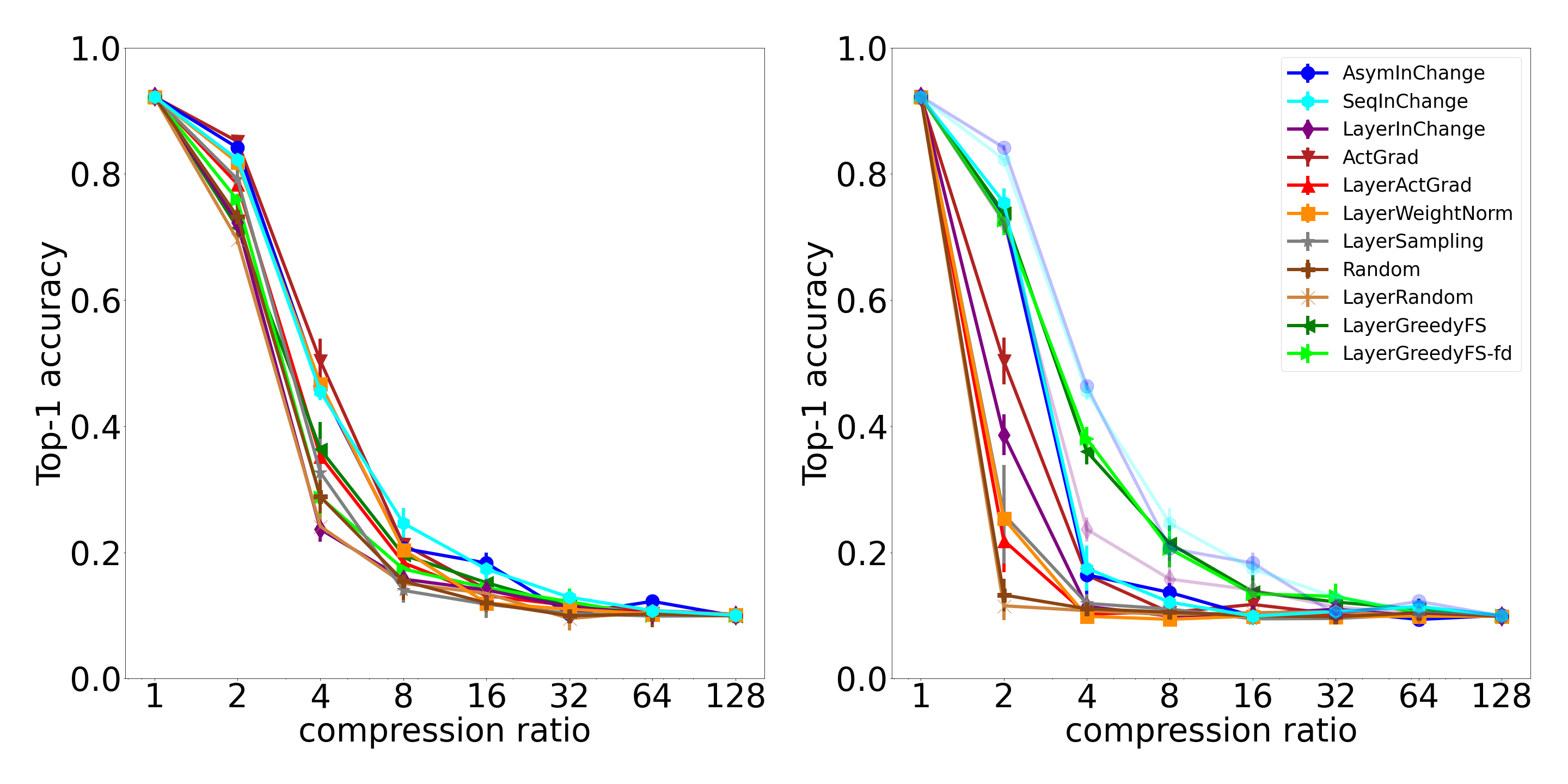}
\end{subfigure}\hfill
\begin{subfigure}{.31\textwidth}
 \centering
\includegraphics[trim=65 50 1410 35, clip, scale=0.1]{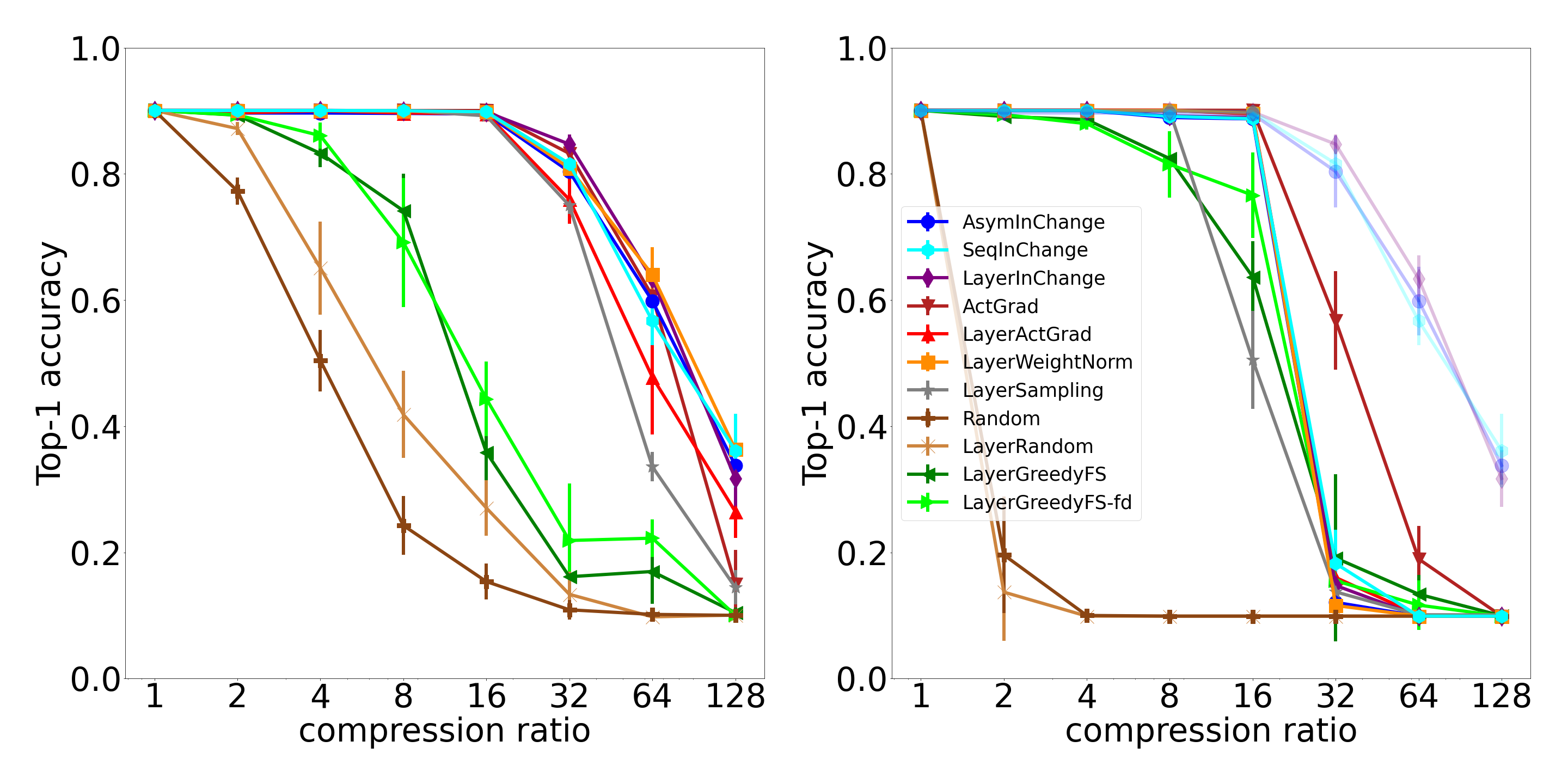}
\end{subfigure}
\\ 
\begin{subfigure}{.31\textwidth}
 \centering
   \hspace*{-5pt}
\includegraphics[trim=1475 50 0 35, clip, scale=0.1]{acc1-compression-1ae094f-2057043-all.png}
\end{subfigure}\hfill
\begin{subfigure}{.31\textwidth}
 \centering
\includegraphics[trim=1475 50 0 35, clip, scale=0.1]{acc1-compression-11b9a94-7862381-all.png}
\end{subfigure}\hfill
\begin{subfigure}{.31\textwidth}
 \centering
\includegraphics[trim=1475 50 0 35, clip, scale=0.1]{acc1-compression-0ebdd32-40042510-all.png}
\end{subfigure}
\caption{\label{fig:oneshot} Top-1 Accuracy of different pruning methods applied to LeNet on MNIST (left), ResNet56 on CIFAR10 (middle), and VGG11 on CIFAR10 (right), for several compression ratios (in log-scale), with (top) and without (bottom) reweighting. We include the three reweighted variants of our method in the bottom plots (faded) for reference.} 
\end{figure}

\mtodo{add smthg similar to "We suspect that these favorable properties are explained by the data-informed
evaluations of filter importance and the corresponding theoretical guarantees of our algorithm – which
enable robustness to variations in network architecture and data distribution."}
\paragraph{Discussion}
We summarize our observations from the empirical results: 

\begin{itemize}[leftmargin=1em, itemindent=0.5em]
\item Our proposed pruning method outperforms state-of-the-art structured pruning methods in various one shot pruning settings. 
As expected, \OurAsym is the best performing variant of our method, and \OurLayer the worst, with its performance deteriorating with deeper models.
Our results also illustrate the robustness of our method, as it reliably yields the best results in various settings, while other baselines perform well in some settings but not in others.
\item 
Reweighting significantly improves performance for all methods, except \LayerGreedyFS and \LayerGreedyFS-fd. We suspect that reweighting does not help in this case because this method already scales the next layer weights, and it takes into account this scaling when selecting neurons/channels to keep, so replacing it with reweighting can hurt performance.  
\item The choice of how much to prune in each layer given a global budget can have a drastic effect on performance, as illustrated in \cref{app:fractionSel}.
\item  \looseness=-1 Fine-tuning with full-training data boosts performance more than reweighting, while fine-tuning with limited data helps less, as illustrated in \cref{app:finetuning}. Reweighting still helps when fine-tuning with limited-data, except for \LayerGreedyFS variants, 
but it can actually deteriorate performance when fine-tuning with full-data. 
Our method still outperforms other baselines after fine-tuning with limited-data, and is among the best performing methods even in the full-data setting.
\end{itemize}
\mtodo{ time will be highly implementation dependent, and we did not try to optimize our method's implementation for efficiency. We give theoretical running time which reflect more accurately the computational cost of our method.}

%% file: conclusion.tex
\section{Conclusion}
\label{sec:Concl}

We proposed a data-efficient structured pruning method, based on submodular optimization. 
By casting the layerwise subset selection problem as a weakly submodular optimization problem, we are able to use the \Greedy algorithm to provably approximate it. 
Empirically, our method consistently outperforms existing structured pruning methods on different network architectures and datasets.

%% file: appendix-relwork.tex
\section{Additional details on related work}\label{app:relwork}


In this section, we give a more detailed comparison of our method with that of \citep{Ye2020, Ye2020b}.

\cite{Ye2020} select neurons/channels to keep in a given layer that minimize the loss of the pruned network. More precisely, they are solving $$\min_{\alpha \geq 0, \sum_i \alpha_i = 1} \sum_{i=1}^n [(f^\ell_\alpha(x_i) - y_i)^2],$$ where $(x_i, y_i)$ are data points, $f^\ell_\alpha(x)$ is the output of the model where neurons in layer $\ell$ corresponding to $\alpha_i=0$ are pruned and the weights of the next layer are scaled by $\alpha_i n_\ell$, i.e., $A^\ell W^{\ell+1}$ is replaced by $n_\ell A^\ell \mathrm{Diag}(\alpha) W^{\ell+1}$ where $\mathrm{Diag}(\alpha)$ is the diagonal matrix with $\alpha$ as its diagonal.
This is similar to the $\ell_1$-relaxation of the selection problem \eqref{eq:optReweighting} we solve, in the special case of a two layer network with a single output, and instead of optimizing the weights of the next layer like we do, they optimize how much to scale them, i.e., in this case their selection problem reduces to $$\min_{\alpha \geq 0, \sum_i \alpha_i = 1} \| A^\ell w^{\ell+1} - n_\ell A^\ell \mathrm{Diag}(\alpha) w^{\ell+1}\|_2^2.$$ 
They use a greedy algorithm with Frank-Wolfe like updates to approximate it (see \citep[Section 12.1]{Ye2020} for the relation between their greedy algorithm and Frank-Wolfe algorithm). 
This method is very expensive as it requires $O(k n_\ell)$ forward passes in the full network, to prune each layer.
The provided theoretical guarantees only holds for two layer networks, and are with respect to an $\ell_1$-relaxation of the selection problem. Empirically, our method significantly outperforms the method of \citep{Ye2020} in all settings we consider.

\cite{Ye2020b} propose two pruning method: Greedy Global imitation and Greedy Local imitation. Greedy Global imitation is the same method from \citep{Ye2020} but with an additional approximation technique which reduces the cost of pruning one layer from $O(k n_\ell)$ to $O(k)$ forward passes through the full network. This is still more expensive than the cost of our method which is equivalent to $O(1/\epsilon)$ forward passes through only the layer being pruned, if using the fast Greedy algorithm from \citep{Li2022} (see \cref{sec:cost}). Greedy Local imitation is closer to our approach, as it selects neurons/channels to keep in a given layer that minimize the change in the input to the next layer, but it also solves an $\ell_1$-relaxation of the selection problem and only optimize the scaling of the next layer weights instead of the weights directly, i.e., it solves
$$\min_{\alpha \geq 0, \sum_i \alpha_i = 1} \| A^\ell w^{\ell+1} - n_\ell A^\ell \mathrm{Diag}(\alpha) w^{\ell+1}\|_2^2.$$ 
\looseness=-1 A similar greedy algorithm with Frank-Wolfe like updates as in \citep{Ye2020} is used. Although the selection problem solved is simpler than ours, the cost of pruning one layer is still more expensive than ours: $O(k n_\ell n_{\ell+1} n)$ vs $O((n_\ell)^2 (n_{\ell+1}+n+k)/\epsilon)$. 
\cite{Ye2020b} also provide bounds on the difference between the output of the original network and the pruned one, with exponential convergence rate for Greedy Local imitation, and  $O(1/k^2)$ rate for the Greedy Global imitation. Similar guarantees with exponential convergence rate  hold for our method (see \cref{sec:errorbds}). 
Empirically, the results in \citep{Ye2020b} show that their global method typically outperforms their local one. So we expect our method to also outperform their local method, since it outperforms their global method.

%% file: appendix-proofs.tex
\section{Missing proofs}\label{app:proofs}
\mtodo{write the whole appendix for the asymmetric formulation and note that the same properties hold for symmetric formulation by setting $A = B$ and expand on the cases where it makes a difference e.g., for cost. Proofs still need some cleaning up.}
Recall that $F(S) =  \| A^\ell W^{\ell+1}\|_F^2 - \min_{\tilde{W}^{\ell+1} \in \R^{n_\ell \times n_{\ell+1}}} \| A^\ell W^{\ell+1} - A^\ell_S \tilde{W}^{\ell+1}  \|_F^2$, and $G(S) = F(M(S))$, where $M$ maps each channel to its corresponding columns in $A^\ell$.
We  denote by $\asymF(S)$ the objective corresponding to the asymmetric formulation introduced in \cref{sec:variants}, i.e., $\asymF(S) =  \| A^\ell W^{\ell+1}\|_F^2 - \min_{\tilde{W}^{\ell+1} \in \R^{n_\ell \times n_{\ell+1}}} \| A^\ell W^{\ell+1} - B^\ell_S \tilde{W}^{\ell+1}  \|_F^2$, and similarly $\asymG(S) = \asymF(M(S))$, where $M$ maps each channel to its corresponding columns in $A^\ell$.
 
We introduce some notation that will be used throughout the Appendix. Given any matrix $D$ and vector $y$, we denote by $x^S(y) \in \argmin_{\supp(x) \subseteq S} \frac{1}{2} \| y - D x\|_2^2$ the vector of optimal regression coefficients, and by $\proj_S(y) = D x^S(y)$, $R^S(y)= y - \proj_S(y)$ the corresponding projection and residual.

\subsection{Submodularity ratio bounds: Proof of \cref{prop:WeaklySub} and \ref{prop:WeaklySubChannels} and their extension to the asymmetric formulation}\label{app:WeakSub-proofs}
In this section, we prove that $F, G$, and their asymmetric variants $\asymF, \asymG$ are all non-decreasing weakly submodular functions.
We start by reviewing the definition of restricted smoothness (RSM) and restricted strong convexity(RSC).

\begin{definition}[RSM/RSC]
Given a differentiable function $\ell: \R^d \to \R$ and $\Omega \subset \R^d \times \R^d$, $\ell$ is $\mu_\Omega$-RSC and $\nu_\Omega$-RSM if 
$\frac{\mu_\Omega}{2} \| x - y \|_2^2 \leq \ell(y) - \ell(x) - \langle \nabla \ell(x), y - x \rangle \leq \frac{\nu_\Omega}{2} \| x - y \|_2^2,\;\; \forall (x,y) \in \Omega.$
\end{definition}
If $\ell$ is RSC/RSM on $\Omega = \{(x,y): \|x\|_0 \leq k, \| y \|_0 \leq k, \| x - y \|_0 \leq k\}$, we denote by $\mu_{k}, \nu_{k}$ the corresponding RSC and RSM parameters. 

\weaklySub*
\begin{proof}
We can write $F(S) =\sum_{m=1}^{n_{\ell+1}} F_m(S) := \ell_m(0) - \min_{\supp(\tilde{w}_m) \subseteq S} \ell_m(\tilde{w}_m)$, where $\ell_m(\tilde{w}_m) =  \| A^\ell w^{\ell+1}_m - A^\ell  \tilde{w}_m\|^2_2$.
The function $F_m(S)$ is then $\gamma_{U, k}$-weakly submodular with $\gamma_{U, k} \geq \tfrac{\mu_{|U|+k}}{\nu_{|U|+1}}$  \citep{Elenberg2016}, where 
$\mu_{|U|+k}$ and $\nu_{|U|+1}$ are the RSC and RSM parameters of $\ell_m$, given by $\mu_{|U|+k} = \min_{\| z \|_2 = 1, \| z \|_0 \leq |U|+k} \|A^\ell z\|^2_2$, and $\nu_{|U|+1} = \max_{\| z \|_2 = 1, \| z \|_0 \leq |U|+1} \|A^\ell z\|^2_2$. It follows then that $F$ is also $\gamma_{U, k}$-weakly submodular. It is easy to check that $F$ is also normalized and non-decreasing. 
\end{proof}

\weaklySubChannels*
\begin{proof}
By definition, $G$ is $\gamma_{U, k}$-weakly submodular iff $F$ satisfies 
$$\gamma_{U, k} F(M(S) | M(L)) \leq \sum_{i \in S} F(M(i) | M(L)),$$
for every two disjoint sets $L, S \subseteq V_{\ell}$, such that $L \subseteq U, |S| \leq k$. We extend the relation established in \citep{Elenberg2016} between weak submodularity and RSC/RSM parameters to this case. 

Let $S' = M(S), L'=M(L)$, $I' = M(i)$, and $k' = r_h r_w k$.
As before, we can write $G(S) = F(S') = \sum_{m=1}^{n_{\ell+1}} F_m(S') := \ell_m(0) - \min_{\supp(\tilde{w}_m) \subseteq S'} \ell_m(\tilde{w}_m)$, where $\ell_m(\tilde{w}_m) =  \| A^\ell w^{\ell+1}_m - A^\ell  \tilde{w}_m\|^2_2$. 
We denote by $\mu_{k}$ and $\nu_{k}$ the RSC and RSM parameters of $\ell_m$, given by
$\mu_{k} = \min_{\| z \|_2 = 1, \| z \|_0 \leq k} \|A^\ell z\|^2_2$, and $\nu_{k} = \max_{\| z \|_2 = 1, \| z \|_0 \leq k} \|A^\ell z\|^2_2$. To simplify notation, we use $x^S := x^S(A^\ell w^{\ell+1}_m)$.

For every two disjoint sets $L, S \subseteq V_{\ell}$, such that $L \subseteq U, |S| \leq k$, we have:
\begin{align*}
0 \leq F_m(S' | L') &=  \ell_m(x^{L'}) - \ell_m(x^{S' \cup L'})\\
&\leq - \langle \nabla \ell_m(x^{L'}), x^{S' \cup L'} - x^{L'}\rangle - \frac{\mu_{|L'| + k'}}{2} \| x^{S' \cup L'} - x^{L'}\|_2^2\\
&\leq \max_{\supp(x) \subseteq S' \cup L'} -\langle \nabla \ell_m(x^{L'}), x- x^{L'} \rangle - \frac{\mu_{|L'| + k'}}{2} \| x - x^{L'}\|_2^2
\end{align*}
By setting $x= x^{L'} - \tfrac{[\nabla \ell_m(x^{L'})]_{S'}}{\mu_{|L'| + k'}}$, we get $G(S' | L') \leq \frac{ \| [\nabla \ell_m(x^{L'})]_{S'} \|_2^2}{2 \mu_{|L'| + k'}}$.\\
Given any $i \in S$, $I' = M(i)$, we have
\begin{align*}
F_m(I' | L') &=  \ell_m( x^{L'} ) - \ell_m(x^{I' \cup L'})\\
&\geq \ell_m( x^{L'} ) - \ell_m( x^{L'}  - \frac{[\nabla \ell_m( x^{L'} )]_{I'}}{\nu_{|L'| + |I'|}}) \\
&\geq  \langle \nabla \ell_m( x^{L'} ), \frac{[\nabla \ell_m( x^{L'} )]_I}{\nu_{|L'| + |I'|}} \rangle - \frac{\nu_{|L'| + k'}}{2} \| \frac{[\nabla \ell_m( x^{L'} )]_I}{\nu_{|L'| + |I'|}} \|_2^2\\
&=\frac{  \|[\nabla \ell_m( x^{L'} )]_{I'} \|_2^2 }{2 \nu_{|L'| + |I'|}}
\end{align*}
Hence, 
\begin{align*}
G(S | L)  \leq \sum_{m=1}^{n_{\ell+1}} \frac{ \| [\nabla \ell_m( x^{L'} )]_{S'} \|_2^2}{2 \mu_{|L'| + k'}} &= \sum_{i \in S, I' = M(i)} \frac{ \| [\nabla \ell_m( x^{L'} )]_{I'} \|_2^2}{2 \mu_{|L'| + k'}} \\
&=  \sum_{i \in S, I' = M(i)} \frac{\nu_{|L'| + |I'|}}{\mu_{|L'| + k'}} \frac{ \| [\nabla \ell_m( x^{L'} )]_{I'} \|_2^2}{2 \nu_{|L'| + |I'|}} =  \frac{\nu_{|L'| + k'}}{\mu_{|L'| + k'}} \sum_{i \in S} G(i | L).
\end{align*}
We thus have $\gamma_{U, k} \geq \frac{\mu_{|U'| + k'}}{\nu_{|U'| + |I'|}}$. 
\end{proof}

Both \cref{prop:WeaklySub} and \cref{prop:WeaklySubChannels} apply also to the asymmetric variants, using exactly the same proofs.
\begin{proposition}\label{prop:WeaklySubAsym}
Given $U \subseteq V, k \in \N_+$, $\asymF$ is a normalized non-decreasing $\gamma_{U, k}$-weakly submodular function, with $$\gamma_{U, k} \geq \frac{\min_{\| z \|_2 = 1, \| z \|_0 \leq |U|+k} \|B^\ell z\|_2^2}{\max_{\| z \|_2 = 1, \| z \|_0 \leq |U|+1} \|B^\ell z\|^2_2}. $$ 
\end{proposition}

\begin{proposition}\label{prop:WeaklySubChannelsAsym}
Given $U \subseteq V_{\ell}, k \in \N_+$, 
$G$ is a normalized non-decreasing $\gamma_{U, k}$-weakly submodular function, with $$\gamma_{U, k} \geq \frac{\min_{\| z \|_2 = 1, \| z \|_0 \leq r_h r_w(|U|+k)} \|A^\ell z\|_2^2}{\max_{\| z \|_2 = 1, \| z \|_0 \leq r_h r_w (|U|+ 1) } \|A^\ell z\|^2_2}.$$
\end{proposition}

\subsection{Cost bound: Proof of \cref{prop:cost} and its extension to other variants}\label{app:cost-proofs}

In this section, we investigate the cost of applying \Greedy with $F, G$ and their asymmetric variants $\asymF, \asymG$. 
To that end, we need the following key lemmas showing how to update the least squares solutions and the function values after adding one or more elements.

 \begin{lemma}\label{lem:LS-sols}
Given a matrix $D$, vector $y$, and a vector of optimal regression coefficients $x^S(y) \in \argmin_{\supp(x) \subseteq S} \frac{1}{2} \| y - D x\|_2^2$, we have
for all $S \subseteq V, i \not \in S$:
\[x^{S \cup i}(y) = (x^S(y) - x^S(d_i) \gamma^{S,i}(y)) + \gamma^{S,i}(y) \1_i \in \argmin_{\supp(x) \subseteq S \cup i} \frac{1}{2} \| y - D x\|_2^2,\]
where $\gamma^{S,i}(y) \in \argmin_{\gamma \in \R} \frac{1}{2} \| y - R^{S}(d_i) \gamma\|_2^2$.
Hence, $\proj_{S \cup i}(y)  = \proj_{S}(y) + \proj_{R^{S}(d_i)}(y)$, where $\proj_{R^{S}(d_i)}(y) = R^{S}(d_i) \gamma^{S,i}(y)$.\\

Similarly, for $I \subseteq V \setminus S$, let $R^S(D_I)$ be the matrix with columns $R^{S}(d_i)$, $x^S(D_I)$ the matrix with columns $x^S(d_i)$, and $\gamma^{S,I}(y) \in \argmin_{\gamma \in \R^{|I|}} \frac{1}{2} \| y - R^{S}(D_I) \gamma\|_2^2$, then
\[x^{S \cup I}(y) = (x^S(y) - x^S(D_I) \gamma^{S,I}(y)) +   e_I  \gamma^{S,I}(y)  \in \argmin_{\supp(x) \subseteq S \cup I} \frac{1}{2} \| y - D x\|_2^2,\]
where $e_I \in \R^{|V| \times |I|}$ is the matrix with $[e_I]_{i,i} = 1$ for all $i \in I$, and $0$ elsewhere. Hence, $\proj_{S \cup I}(y)  = \proj_{S}(y) + \proj_{R^{S}(D_I)}(y)$, where $\proj_{R^{S}(D_I)}(y) = R^{S}(D_I) \gamma^{S,I}(y)$. 
\end{lemma}
\begin{proof}
By optimality conditions, we have:
\begin{align}
&D_S^\top(D_S x^S(y) - y) = 0\\
&D_S^\top(D_S x^S(d_i) - d_i) = 0 \Rightarrow  -D_S^\top  R^{S}(d_i) = 0 \\
&R^{S}(d_i)^\top(R^{S}(d_i) \gamma^{S,i}(y) - y) = 0
\end{align}
We prove that $\hat{x}^{S \cup i}(y) =  (x^S(y) - x^S(d_i) \gamma^{S,i}(y)) + \gamma^{S,i}(y) \1_i$ satisfies the optimality conditions on $x^{S \cup i}(y)$, hence $\hat{x}^{S \cup i}(y) = x^{S \cup i}(y)$.

We have $D_{S \cup i} \hat{x}^{S \cup i}(y) = D_S x^S(y) + R^S(d_i) \gamma^{S,i}(y)$, then
\begin{align*}
D_S^\top( D_{S \cup i} \hat{x}^{S \cup i}(y) - y ) &= D_S^\top(  D_S x^S(y) - y ) +  D_S^\top R^S(d_i) \gamma^{S,i}(y) = 0
\end{align*}
and 
\begin{align*}
d_i^\top( D_{S \cup i} \hat{x}^{S \cup i}(y) - y ) &= (R^{S}(d_i) + D_S x^S(d_i))^\top( D_S x^S(y) + R^S(d_i) \gamma^{S,i}(y) - y ) \\
&= R^{S}(d_i)^\top D_S x^S(y) + R^{S}(d_i)^\top (R^S(d_i) \gamma^{S,i}(y) - y)\\ 
&~~ + (D_S x^S(d_i))^\top (D_S x^S(y) - y ) + (D_S x^S(d_i))^\top R^S(d_i) \gamma^{S,i}(y) \\
&= 0
\end{align*}
The proof for the case where we add multiple indices at once follows similarly.
\end{proof}

\begin{lemma}\label{lem:obj-updates-asym}
 For all $S \subseteq V_\ell, i \not \in S$, let $R_S(b_i^\ell) = b_i^\ell - \proj_S(b_i^\ell)$,  $\proj_{R_S(b_i^\ell)}(y) = R_S(b_i^\ell) \gamma^{S, i}(y) $ with $ \gamma^{S, i}(y) \in \argmin_{\gamma \in \R} \| y - R_S(b_i^\ell) \gamma \|_2^2$. 
We can write the marginal gain of adding $i$ to $S$ w.r.t $\asymF$ as:
$$\asymF(i | S)= \sum_{m=1}^{n_{\ell+1}}  \|  \proj_{R_S(b^\ell_i)}(A^{\ell}) w_m^{\ell+1} \|_2^2,$$
where $\proj_{R_S(b^\ell_i)}(A^{\ell}) $ is the matrix with columns $\proj_{R_S(b^\ell_i)}(a_j^\ell)$ for all $j \in V_\ell$. 
Similarly, for all $S \subseteq V_\ell, I \subseteq V_\ell \setminus S$,  let $R_S(B_I^\ell) = B_I^\ell - \proj_S(B_I^\ell)$,  $\proj_{R_S(B_I^\ell)}(y) = R_S(B_I^\ell) \gamma^{S, I}(y) $ with $ \gamma^{S, I}(y) \in \argmin_{\gamma \in \R^{|I|}} \| y - R_S(B_I^\ell) \gamma \|_2^2$. We can write the marginal gain of adding $I$ to $S$ w.r.t $\asymF$ as:
$$\asymF(I | S)= \sum_{m=1}^{n_{\ell+1}}  \|  \proj_{R_S(B^\ell_I)}(A^{\ell}) w_m^{\ell+1} \|_2^2,$$
where $\proj_{R_S(B^\ell_I)}(A^{\ell}) $ is the matrix with columns $\proj_{R_S(B^\ell_I)}(a_j^\ell)$ for all $j \in V_\ell$. 
\end{lemma}
\begin{proof}
We prove the claim for the case where we add several elements. The case where we add a single element then follows as a special case. 
For a fixed $S \subseteq V_\ell$, the reweighted asymmetric input change $ \| A^\ell W^{\ell+1} - B^\ell_S \tilde{W}^{\ell+1}  \|_F^2$ is minimized by setting $\tilde{W}^{\ell+1} = x^S(A^{\ell}) W^{\ell+1},$ where $x^S(A^{\ell}) \in \R^{n_\ell \times n_{\ell}}$ is the matrix with columns $x^S(a^\ell_j)$ such that
\begin{align}\label{eq:proj-asym}
x^S(a^\ell_j) \in \argmin_{\supp(x) \subseteq S} \|a_j^\ell - B^\ell x \|  _2^2 \text{ for all $j \in V_\ell$}.
\end{align}
Plugging  $\tilde{W}^{\ell+1}$ into the expression of $\asymF(S)$ yields
$ \asymF(S) = \| A^\ell W^{\ell+1}\|_F^2 -  \| (  A^\ell - \proj_S(A^{\ell})) W^{\ell+1}\|_F^2, $ where $\proj_S(A^{\ell}) = B^\ell_S x^S(A^{\ell})$.
For every $m \in \{1, \cdots, n_{\ell+1}\}$, we have:
\begin{align*}
&\quad \|  (\proj_{S \cup I}(A^{\ell}) -   A^\ell) w_m^{\ell+1} \|_2^2 - \|  (\proj_S(A^{\ell}) -   A^\ell) w_m^{\ell+1} \|_2^2 \\
&= \|  (\proj_{S}(A^{\ell}) + \proj_{R_S(B_I)}(A^{\ell}) -   A^\ell) w_m^{\ell+1} \|_2^2 - \|  (\proj_S(A^{\ell}) -   A^\ell) w_m^{\ell+1} \|_2^2  &\text{(by  \cref{lem:LS-sols})}\\
&= \|  \proj_{R_S(B_I)}(A^{\ell}) w_m^{\ell+1} \|_2^2 -  2 \langle (A^\ell - \proj_S(A^{\ell})) w_m^{\ell+1},  \proj_{R_S(B_I)}(A^{\ell}) w_m^{\ell+1} \rangle\\
&= \|  \proj_{R_S(B_I)}(A^{\ell}) w_m^{\ell+1} \|_2^2 -  2 \langle  \proj_{R_S(B_I)}(A^{\ell})w_m^{\ell+1},  \proj_{R_S(B_I)}(A^{\ell}) w_m^{\ell+1} \rangle\\
&= - \|  \proj_{R_S(B_I)}(A^{\ell}) w_m^{\ell+1} \|_2^2 
\end{align*}
where the second to last equality holds because $y - \proj_S(y) -  \proj_{R_S(B_I)}(y)$ and $\proj_{R_S(B_I)}(y')$ are orthogonal by optimality conditions (see proof of  \cref{lem:LS-sols}):
\[\langle y - \proj_S(y) -  \proj_{R_S(B_I)}(y), \proj_{R_S(B_I)}(y')\rangle = \langle y - B_S x^S(y) -  R_S(B_I) \gamma^{S, I}(y), R_S(B_I) \gamma^{S, I}(y')\rangle = 0.\]
Hence, $\asymF(I | S)= \sum_{m=1}^{n_{\ell+1}} \|  \proj_{R_S(B_I)}(A^{\ell}) w_m^{\ell+1} \|_2^2 $. In particular, if $I = \{i\}$, $\asymF(i | S)= \sum_{m=1}^{n_{\ell+1}}  \|  \proj_{R_S(b^\ell_i)}(A^{\ell}) w_m^{\ell+1} \|_2^2$.
\end{proof}

\begin{lemma}\label{lem:obj-updates-sym}
For all $S \subseteq V_\ell, i \not \in S$, let $R_S(a_i^\ell) = a_i^\ell - \proj_S(a_i^\ell)$,  $\proj_{R_S(a_i^\ell)}(y) = R_S(a_i^\ell) \gamma^{S, i}(y) $ with $ \gamma^{S, i}(y) \in \argmin_{\gamma \in \R} \| y - R_S(a_i^\ell) \gamma \|_2^2$. We can write the marginal gain of adding $i$ to $S$ w.r.t $F$ as:
$$F(i | S)= \sum_{m=1}^{n_{\ell+1}}  \|  \proj_{R_S(a^\ell_i)}(A^{\ell}_{V \setminus S}) w_m^{\ell+1} \|_2^2,$$ 
where $\proj_{R_S(a^\ell_i)}(A^{\ell}_{V \setminus S}) $ is the matrix with columns $\proj_{R_S(a^\ell_i)}(a_j^\ell)$ for all $j \in V \setminus S$, $0$ otherwise.
Similarly, for all $S \subseteq V_\ell, I \subseteq V_\ell \setminus S$, let $R_S(A_I^\ell) = A_I^\ell - \proj_S(A_I^\ell)$,  $\proj_{R_S(A_I^\ell)}(y) = R_S(A_I^\ell) \gamma^{S, I}(y) $ with $ \gamma^{S, I}(y) \in \argmin_{\gamma \in \R^{|I|}} \| y - R_S(A_I^\ell) \gamma \|_2^2$. We can write the marginal gain of adding $I$ to $S$ w.r.t $F$ as:
$$F(I | S)= \sum_{m=1}^{n_{\ell+1}}  \|  \proj_{R_S(A^\ell_I)}(A^{\ell}_{V \setminus S}) w_m^{\ell+1} \|_2^2,$$
where $\proj_{R_S(A^\ell_I)}(A^{\ell}_{V \setminus S}) $ is the matrix with columns $\proj_{R_S(A^\ell_I)}(a_j^\ell)$ for all $j \in V \setminus S$, $0$ otherwise.
\end{lemma}
\begin{proof}
Setting $B^\ell = A^\ell$ in \cref{lem:obj-updates-asym}, we get $F(I | S)= \sum_{m=1}^{n_{\ell+1}}  \|  \proj_{R_S(A^\ell_I)}(A^{\ell}) w_m^{\ell+1} \|_2^2$.
Note that for all $i \in I, j \in S$, $a_j^\ell$ and $R_S(a_i^\ell)$ are orthogonal, and hence $\proj_{R_S(A_I^\ell)}(a_j^\ell) = 0$, by optimality conditions (see proof of  \cref{lem:LS-sols}). It follows then that $F(I | S)= \sum_{m=1}^{n_{\ell+1}}  \|  \proj_{R_S(A^\ell_I)}(A^{\ell}_{V \setminus S}) w_m^{\ell+1} \|_2^2$.
In particular, if $I = \{i\}$, $F(i | S)= \sum_{m=1}^{n_{\ell+1}}  \|  \proj_{R_S(a^\ell_i)}(A^{\ell}_{V \setminus S}) w_m^{\ell+1} \|_2^2$.
\end{proof}

\begin{proposition}\label{prop:cost-Asym}
Given $S \subseteq V_\ell$ such that $|S| \leq k$, $i \not \in S$,  let $\proj_S(a_j^\ell) = A^\ell_S x^S(a_j^\ell)$. 
Assuming $\asymF(S), \proj_S(b_j^\ell), x^S(b_j^\ell)$ for all $j \not \in  S$, and $x^S(a^\ell_j)$  for all $j \in V_\ell$, were computed in the previous iteration, we can compute $\asymF(S+i), \proj_{S+i}(b_j^\ell), x^{S+i}(b_j^\ell)$ for all $j \not \in  (S + i)$, and $x^{S+i}(a^\ell_j)$  for all $j \in V_\ell$, in $$O(n_{\ell}\cdot (n_{\ell+1} + n + k)) \text{ time}.$$
Computing the optimal weights in Eq. \eqref{eq:optW} at the end of \Greedy can then be done in $O(k \cdot  n_{\ell}\cdot n_{\ell+1})$ time.
\end{proposition}
\begin{proof}
By \cref{lem:obj-updates-asym}, we can update the function value using  $$\asymF(i | S)= \sum_{m=1}^{n_{\ell+1}}  \|  \proj_{R_S(b_i^\ell)}(A^{\ell}) w_m^{\ell+1} \|_2^2  =  \sum_{m=1}^{n_{\ell+1}} ( \sum_{j \in V_\ell} \gamma^{S, i}(a^\ell_j)  w_{jm}^{\ell+1} \big)^2 \| R_S(b_i^\ell) \|_2^2. $$
This requires $O(n)$ to compute $R_S(b_i^\ell)$ and its norm, $O(n_\ell \cdot n)$ to compute  $\gamma^{S, i}(a^\ell_j) = \tfrac{R_S(b_i^\ell)^\top a^\ell_j}{\| R_S(b_i^\ell)\|_2^2}$ for all $j \in V_\ell$, and an additional $O(n_\ell \cdot n_{\ell+1})$ to finally evaluate $\asymF(S + i)$.
We also need $O(|S^c| \cdot n)$ to update $\proj_{S \cup i}(b_j) = \proj_{S}(b_j) + \proj_{R_S(b_i^\ell)}(b_j)$ (by  \cref{lem:LS-sols}), using $\proj_{R_S(b_i^\ell)}(b_j) = R_S(b_i^\ell) \gamma^{S,i}(b^\ell_j)$ for all $j \in V_\ell \setminus S \cup i$, $O(|S^c| \cdot |S|)$ to update $x^{S \cup i}(b^\ell_j) = x^S(b^\ell_j) + (\1_i- x^S(b_i)) \gamma^{S,i}(b^\ell_j))$ (by  \cref{lem:LS-sols}) for all $j \in V_\ell \setminus S \cup i$, and $O(n_\ell \cdot |S|)$ to update $x^{S \cup i}(a^\ell_j) = x^S(a^\ell_j) + (\1_i- x^S(b_i)) \gamma^{S,i}(a^\ell_j)) $ for all $j \in V_\ell$.
So in total, we need $O(n_{\ell}\cdot (n_{\ell+1} + n + k)).$
Computing the new weights $\tilde{W}^{\ell+1} = x^S(A^{\ell}) W^{\ell+1}$ at the end can be done in $O(n_\ell \cdot n_{\ell+1} \cdot |S|) = O(n_\ell \cdot n_{\ell+1} \cdot k)$.
\end{proof}

\cost*
\begin{proof}
The proof follows from \cref{lem:obj-updates-sym}  and \ref{lem:LS-sols} in the same way as in \cref{prop:cost-Asym}.
\end{proof}

\cref{prop:cost} and \cref{prop:cost-Asym} apply also to $G$ and $\asymG$ respectively, since $|M(i)|= O(1)$. 

\subsection{Extension of \textsc{Stochastic-Greedy} to weakly submodular functions}\label{app:StochasticGreedy}

In this section, we show that the guarantee of \textsc{Stochastic-Greedy} (\cref{alg:stochasticGreedy}) easily extends to weakly submodular functions.
\begin{algorithm}
   \caption{\textsc{Stochastic-Greedy}}
\label{alg:stochasticGreedy}
\begin{algorithmic}[1]
\STATE {\bfseries Input:} Ground set $V$, set function $F: 2^V \to \R_+$, budget $k \in \N_+$
\STATE $S \leftarrow \emptyset$
\WHILE{$|S|<k$}
\STATE $R \gets$ a random subset obtained by sampling $s$ random elements from $V \setminus S$.
\STATE $i^* \leftarrow \argmax_{i \in R} F(i \mid S)$\label{l:max-gain}
\STATE $S \leftarrow S \cup \{i^*\}$
\ENDWHILE
\STATE {\bfseries Output:} $S$
\end{algorithmic}
\end{algorithm}

\begin{proposition}
Let $\hat{S}$ be the solution returned by \textsc{Stochastic-Greedy} with $s = \tfrac{n}{k} \log(\tfrac{1}{\epsilon})$, and let $F$ be a non-negative monotone $\gamma_{\hat{S}, k}$-weakly submodular function. Then 
\[\mathbb{E}[F(\hat{S})] \geq (1 - e^{-\gamma_{\hat{S}, k}}) \max_{|S| \leq k} F(S) \]
\end{proposition}
\begin{proof}
Denote by $S_t$ the solution at iteration $t$ of \textsc{Stochastic-Greedy}, and $S^*$ an optimal solution. 
The proof follows in the same way as in \cite[Theorem 1]{Mirzasoleiman2015}. In particular, Lemma 2 therein, does not use submodularity, so it holds here too.  It states that the expected gain of \textsc{Stochastic-Greedy} in one step is at least $\tfrac{1 - \epsilon}{k} \sum_{i \in S^* \setminus S_t} F(i | S_t)$ for any $t$. Therefore,
\begin{align*}
\mathbb{E}[F(S_{t+1}) - F(S_t) \mid S_t] &\geq \frac{1 - \epsilon}{k} \sum_{i \in S^* \setminus S_t} F(i | S_t)\\
&\geq \gamma_{\hat{S}, k} \frac{1 - \epsilon}{k} F(S^* \setminus S_t | S_t) \\
&\geq  \gamma_{\hat{S}, k}\frac{1 - \epsilon}{k} (F(S^*) - F(S_t))
\end{align*}
By taking expectation over $S_t$ and induction, we get
\begin{align*}
\mathbb{E}[F(S_{t})] &\geq \Big(1 - \big(1 - \gamma_{\hat{S}, k}\frac{1 - \epsilon}{k} \big)^k\Big) F(S^*)  \\
&\geq  \big(1 - e^{- \gamma_{\hat{S}, k} (1 - \epsilon)} \big)  F(S^*) \\
&\geq  \big(1 - e^{- \gamma_{\hat{S}, k}} - \epsilon \big)  F(S^*) 
\end{align*}

\end{proof}

\section{Error rates: Proof of \cref{prop:layerError} and \cref{corr:globalError-multiplelayers} }\label{app:errorbds}

In this section, we provide the proofs of our method's error rates.
\layerError*
\begin{proof}
This follows by extending the approximation guarantee in Eq. (3) to: $$F(\hat{S}) \geq (1 - e^{- \gamma_{\hat{S}, k^*} {k}/{k^*}}) \max_{|S|\leq k^*} F(S),$$ by a slight adaption of the proof in \citep{Elenberg2016, Das2011}.
In particular, taking $k^* = n_\ell$ yields:  
$$F(\hat{S}) = \| A^\ell W^{\ell+1}\|_F^2 - \| A^\ell W^{\ell+1} - A^\ell_{\hat{S}} \hat{W}^{\ell+1}  \|_F^2 \geq (1 - e^{- \gamma_{\hat{S}, n_\ell} {k}/{n_\ell}}) \| A^\ell W^{\ell+1}\|_F^2.$$
The first part of the claim follows by rearraging terms.
Similarly, for the asymmetric formulation we have:
$$\asymF(\hat{S}) \geq (1 - e^{- \gamma_{\hat{S}, k^*} {k}/{k^*}}) \max_{|S|\leq k^*} \asymF(S).$$
Taking $k^* = n_\ell$ yields:  
\begin{align*}
\asymF(\hat{S}) &= \| A^\ell W^{\ell+1}\|_F^2 - \| A^\ell W^{\ell+1} - B^\ell_{\hat{S}} \hat{W}^{\ell+1}  \|_F^2 \\
&\geq (1 - e^{- \gamma_{\hat{S}, n_\ell} {k}/{n_\ell}}) (\| A^\ell W^{\ell+1}\|_F^2 - \min_{\tilde{W}^{\ell+1} \in \R^{n_\ell \times n_{\ell+1}}} \| A^\ell W^{\ell+1} - B^\ell \tilde{W}^{\ell+1}  \|_F^2 )
\end{align*}
The second part of the claim follows by rearraging terms.
\end{proof}

\globalError*
\begin{proof}
We start by proving the bound for \OurSeq. Recall that $B^\ell$ are the updated activations of layer $\ell$ after layers $1$ to $\ell-1$ are pruned, and let $\hat{W}^{\ell+1}$ be the optimal weights corresponding to $\hat{S}_\ell$ (Eq. \eqref{eq:optW}).
We can write $y = H_1(A^1 W^{2}) = H_1(B^1 W^{2})$ since $A^1 = B^1$, and  $y^{\hat{S}_\ell} = H_\ell(B^\ell_{\hat{S}_\ell} \hat{W}^{\ell+1}), y^{\hat{S}_{\ell-1}} = H_\ell(B^\ell {W}^{\ell+1})$ for all $\ell \in [L]$. 
Since $H_\ell$ is Lipschitz continuous for all $\ell \in [L]$, we have by the triangle inequality:
\begin{align*}
\| y - y^{\hat{S}_L}\|_2^2 &\leq  \sum_{\ell=1}^L \| y^{\hat{S}_\ell} - y^{\hat{S}_{\ell-1}}\|_F^2 \\
&\leq  \sum_{\ell=1}^L \| H_\ell(B^\ell_{\hat{S}_\ell} \hat{W}^{\ell+1}) -  H_\ell(B^\ell {W}^{\ell+1})\|_F^2 \\
&\leq  \sum_{\ell=1}^L \| H_\ell\|_{\text{Lip}}^2 \| B^\ell_{\hat{S}_\ell} \hat{W}^{\ell+1} -  B^\ell {W}^{\ell+1}\|_F^2 \\
&\leq \sum_{\ell=1}^L e^{- \gamma_{\hat{S}_\ell, n_\ell} {k_\ell}/{n_\ell}} \| H_\ell\|_{\text{Lip}}^2  \| B^\ell W^{\ell+1}\|_F^2,
\end{align*}
where the last inequality follows from \cref{prop:layerError}.

Next, we prove the bound for \OurAsym. Let $\tilde{W}^{\ell+1}$ be the optimal weights corresponding to $\tilde{S}_\ell$ (Eq. \eqref{eq:optW}). We can write $y = H_\ell(A^\ell W^{\ell+1})$ and $y^{\tilde{S}_\ell} = H_\ell(B^\ell_{\tilde{S}_\ell} \tilde{W}^{\ell+1})$ for all $\ell \in [L]$.
By \cref{prop:layerError} and the Lipschitz continuity of $H_L$, we have 
$$\| y - y^{\tilde{S}_L}\|_2^2 \leq e^{- \gamma_{\tilde{S}_L, n_L} {k_L}/{n_L}} \| H_L\|_{\text{Lip}}^2  \| A^L W^{L+1}\|_F^2 + (1 - e^{- \gamma_{\tilde{S}_L, n_L} {k_L}/{n_L}}) \| H_L\|_{\text{Lip}}^2 \| A^L W^{L+1} - B^L {W}^{L+1}  \|_F^2$$ 
Let $H_\ell^{\ell+1}$ denote the function corresponding to all layers between the $\ell$th layer and the $(\ell+1)$th layer (not necessarily consecutive in the network), then we can write $A^L W^{L+1} = H_{L-1}^{L}(A^{L-1} {W}^{L})$ and $B^L {W}^{L+1} = H_{L-1}^{L}(B^{L-1}_{\tilde{S}_{L-1}} \tilde{W}^{L})$, and $H_{L-1}(Z) = H_L(H_{L-1}^{L}(Z))$. It follows then that $\| H_{L-1} \|_{\text{Lip}}  = \| H_{L} \|_{\text{Lip}} \| H_{L-1}^L \|_{\text{Lip}}$, and
\begin{align*}
\| y - y^{\tilde{S}_L}\|_2^2 &\leq e^{- \gamma_{\tilde{S}_L, n_L} {k_L}/{n_L}} \| H_L\|_{\text{Lip}}^2  \| A^L W^{L+1}\|_F^2 + (1 - e^{- \gamma_{\tilde{S}_L, n_L} {k_L}/{n_L}})  \| H_L\|_{\text{Lip}}^2  \| H_{L-1}^L \|_{\text{Lip}}^2 \| A^{L-1} {W}^{L} - B^{L-1}_{\tilde{S}_{L-1}} \tilde{W}^{L}\|_F^2 \\
&\leq e^{- \gamma_{\tilde{S}_L, n_L} {k_L}/{n_L}} \| H_L\|_{\text{Lip}}^2  \| A^L W^{L+1}\|_F^2 + (1 - e^{- \gamma_{\tilde{S}_L, n_L} {k_L}/{n_L}})  \| H_{L-1} \|_{\text{Lip}}^2 \| A^{L-1} {W}^{L} - B^{L-1}_{\tilde{S}_{L-1}} \tilde{W}^{L}\|_F^2 
\end{align*}
Repeatedly applying the same arguments yields the claim.
\end{proof}

\section{Stronger notion of approximate submodularity}\label{app:weakDRsub}

In this section, we show that $F$ and $\asymF$ satisfy stronger properties 
than the weak submodularity discussed in \cref{sec:Greedy}, which lead to a stronger approximation guarantee for \Greedy. These properties do not necessarily hold for $G$ and $\asymG$.

\subsection{Additional preliminaries}\label{sec:addPrelim}
We start by reviewing some preliminaries. Recall that a set function $F$ is \emph{submodular} if it has diminishing marginal gains: $F( i \mid S) \geq F(i \mid T)$ for all $S \subseteq T$, $i \in V\setminus T$. If $-F$ is submodular, then $F$ is said to be \emph{supermodular}, i.e., $F$ satisfies $F( i \mid S) \leq F(i \mid T)$, for all $S \subseteq T$, $i \in V\setminus T$. 
When $F$ is both submodular and supermodular, it is said to be \emph{modular}.


Relaxed notions of submodularity/supermodularity, called \emph{weak DR-submodularity/supermodularity}, were introduced in \citep{Lehmann2006} and \citep{Bian2017a}, respectively.
\begin{definition}[Weak DR-sub/supermodularity] \label{def:WDR}
A set function $F$ is $\alpha_k$-weakly DR-submodular, with $k \in \N_+, \alpha_k > 0$, if 
\[F( i | S) \geq \alpha_k F( i | T), \text{ for all $ S \subseteq T, i \in V \setminus T, |T| \leq k$}. \]
Similarly, $F$ is $\beta_k$-weakly DR-supermodular, with $k \in \N_+, \beta > 0$, if 
\[ F( i | T)  \geq \beta_k F( i | S), \text{ for all $S \subseteq T, i \in V \setminus T, |T| \leq k$}. \]
We say that $F$ is \emph{$(\alpha_k,\beta_k)$-weakly DR-modular} if it satisfies both properties.
\end{definition}
The parameters $\alpha_k,\beta_k$ characterize how close a set
function is to being submodular and supermodular, respectively.
If $F$ is non-decreasing, then $\alpha_k, \beta_k \in [0,1]$, 
  $F$ is submodular (supermodular) if and only if $\alpha_k = 1$ ($\beta_k = 1$) for all $k \in \N_+$, and modular if and only if both $\alpha_k = \beta_k = 1$  for all $k \in \N_+$.
 The notion of weak DR-submodularity is a stronger notion of approximate submodularity than weak submodularity, as $\gamma_{S, k} \geq \alpha_{|S|+k-1}$  for all $S \subseteq V, k \in N_+$ \citep[Prop. 8]{ElHalabi2018}. This implies that \Greedy achieves a $(1 - e^{-\alpha_{2k-1}})$-approximation when $F$ is $\alpha_{2k-1}$-weakly DR-submodular.
 
 A stronger approximation guarantee can be obtained with the notion of \emph{total curvature} introduced in \citep{Sviridenko2017}, which is a stronger notion of approximate submodularity than weak DR-modularity.
 \begin{definition}[Total curvature]
Given a set function $F$, we define its total curvature $c_k$ where $k \in \N_+$, as
$$ c_k = 1 - \min_{|S| \leq k, |T| \leq k, i \in V \setminus T} \frac{F(i|S)}{F(i|T)}.$$ 
\end{definition}
Note that if $F$ has total curvature $c_k$, then $F$ is $(1-c_k, 1-c_k)$-weakly DR-modular. Given a non-decreasing function $F$ with total curvature $c_k$, the \Greedy algorithm is  guaranteed to return a solution $F(\hat{S}) \geq (1 - c_k) \max_{|S| \leq k} F(S)$ \citep[Theorem 6]{Sviridenko2017}. 

\subsection{Approximate modularity of reweighted input change}
The reweighted input change objective $F$ is closely related to the column subset selection objective (the latter is a special case of $F$ where $W^{\ell+1}$ is the identity matrix), whose total curvature was shown to be related to the condition number of $A^\ell$ in \citep{Sviridenko2017}.

We show in Propositions \ref{prop:PruningWeaklyDRmod-Asym} and  \ref{prop:PruningWeaklyDRmod} that the total curvatures of $\asymF$ and $F$  are related to the condition number of $A^\ell W^{\ell+1}$.

\begin{proposition}\label{prop:PruningWeaklyDRmod-Asym}
Given $k \in \N_+$, 
$\asymF$ is a normalized non-decreasing $\alpha_k$-weakly DR-submodular function, with $$\alpha_k \geq \frac{\min_{\| z \|_2 = 1} \|(A^\ell W^{\ell+1})^\top z\|^2_2}{\max_{\| z \|_2 = 1} \| (A^\ell W^{\ell+1})^\top z\|^2_2}.$$
Moreover, if any collection of $k+1$ columns of $B^\ell$ are linearly independent, $\asymF$ is also $\alpha_k$-weakly DR-supermodular and has total curvature $1 - \alpha_k$.
\end{proposition}
\begin{proof}
We adapt the proof from \citep[Lemma 6]{Sviridenko2017}. 
For all $ S \subseteq V, i \in V \setminus S$, we have  $F(i | S)= \sum_{m=1}^{n_{\ell+1}}  \|  \proj_{R_S(b^\ell_i)}(A^{\ell}) w_m^{\ell+1} \|_2^2$ by \cref{lem:obj-updates-asym}. For all $j \not \in S$, we have $\proj_{R_S(b_i^\ell)}(a_j^\ell) = R_S(b_i^\ell) \frac{R_S(b_i^\ell)^\top a_j^\ell}{\| R_S(b_i^\ell)\|^2}$ if $\| R_S(b_i^\ell)\| > 0$, and $0$ otherwise, by optimality conditions. Hence, we can write for all $i$ such that $\| R_S(b_i^\ell)\| > 0$,
\begin{align*}
\asymF(i|S) = \sum_{m=1}^{n_{\ell+1}}   \|  \proj_{R_S(b_i)}(A^{\ell}) w_m^{\ell+1} \|_2^2 
= \sum_{m=1}^{n_{\ell+1}}\| \sum_{j  \in V} w_{j m}^{\ell+1} R_S(b_i^\ell) \frac{R_S(b_i^\ell)^\top a_j^\ell}{\| R_S(b_i^\ell)\|^2} \|_2^2
= \| (A^\ell W^{\ell+1})^\top \frac{R_S(b_i^\ell)}{\| R_S(b_i^\ell)\|} \|_2^2
\end{align*}
Hence,  
\[ \min_{\| z \|_2 = 1} \|(A^\ell W^{\ell+1})^\top z\|^2_2  \leq \asymF(i | S) \leq  \max_{\| z \|_2 = 1} \| (A^\ell W^{\ell+1})^\top z\|^2_2\]
Let $v_j = x^S(B^\ell)_{ij}$ for $j \in S$, $v_i = -1$, and $z = v / \| v\|_2$, then $\| v \|_2 \geq 1$, and 
\begin{align*}
\| R_S(b_i^\ell)\|^2 = \| B^\ell v\|^2_2 \geq \| B^\ell z\|^2_2 \geq \min_{\| z \|_2 \leq 1, \| z \|_0 \leq |S|+1} \| B^\ell z\|^2_2
\end{align*}
The bound on $\alpha_k$ then follows by noting that $\| R_S(b_i^\ell)\| \geq \| R_T(b_i^\ell)\|$ for all $S \subseteq T$.
The rest of the proposition follows by noting that if any collection of $k+1$ columns of $B^\ell$ are linearly independent, then $\| R_S(b_i^\ell)\| \geq \min_{\| z \|_2 \leq 1, \| z \|_0 \leq k+1} \| B^\ell z\|^2_2 >0$ for any $S$ such that $|S| \leq k$.
\end{proof}

As discussed in Section \ref{sec:addPrelim}, Proposition \ref{prop:PruningWeaklyDRmod-Asym} implies that \Greedy achieves a $(1 - e^{-\alpha_{2k-1}})$-approximation with $\asymF(S)$, where $\alpha_k$ is non-zero if all rows of $A^\ell W^{\ell+1}$ are linearly independent. If in addition any $k + 1$ columns of $B^\ell$ are linearly independent, then \Greedy achieves an $\alpha_k$-approximation.

\begin{proposition}\label{prop:PruningWeaklyDRmod}
Given $k \in \N_+$, 
$F$ is a normalized non-decreasing $\alpha_k$-weakly DR-submodular function, with $$\alpha_k \geq \frac{\max\{ \min_{\| z \|_2 = 1} \|(A^\ell W^{\ell+1})^\top z\|^2_2, \min_{\| z \|_2 = 1} \|  (W^{\ell+1})^\top z\|^2_2 \min_{\| z \|_2 \leq 1, \| z \|_0 \leq k+1} \| A^\ell z\|^2_2\}}{\max_{\| z \|_2 = 1} \| (A^\ell W^{\ell+1})^\top z\|^2_2}.$$
Moreover, if any collection of $k+1$ columns of $A^\ell$ are linearly independent, $F$ is also $\alpha_k$-weakly DR-supermodular and has total curvature $1 - \alpha_k$.
\end{proposition}
\begin{proof}
Setting $B^\ell = A^\ell$ in \cref{lem:obj-updates-asym}, we get 
\[\alpha_k \geq \frac{\min_{\| z \|_2 = 1} \|(A^\ell W^{\ell+1})^\top z\|^2_2}{\max_{\| z \|_2 = 1} \| (A^\ell W^{\ell+1})^\top z\|^2_2}\]
To obtain the second lower bound, we note that by \cref{lem:obj-updates-sym} we can write for all $ S \subseteq V, i \in V \setminus S$ such that $\| R_S(a_i^\ell)\| > 0$, $$F(i | S) = \sum_{m=1}^{n_{\ell+1}} (\sum_{j  \not \in S} w_{j m}^{\ell+1} \frac{R_S(a_i^\ell)^\top a_j^\ell}{\| R_S(a_i^\ell)\|^2})^2 \|R_S(a_i^\ell) \|_2^2 = \| (W^{\ell+1})^\top (A^\ell_{V \setminus S})^\top \frac{ R_S(a_i^\ell)}{\| R_S(a_i^\ell)\|} \|_2^2 \|R_S(a_i^\ell) \|_2^2.$$ 
Note that $\gamma^{S, i}(a_i) = \frac{R_S(a_i^\ell)^\top a_i^\ell}{\| R_S(a_i^\ell)\|^2} = 1$ by optimality conditions ($R_S(a_i^\ell)^\top a_i^\ell = R_S(a_i^\ell)^\top (R_S(a_i^\ell) + A_S x^S(ai)) = \| R_S(a_i^\ell) \|_2^2$; see proof of  \cref{lem:LS-sols}), hence $\| (A^\ell_{V \setminus S})^\top \frac{ R_S(a_i^\ell)}{\| R_S(a_i^\ell)\|} \|_2^2 \geq 1$ and $\| (W^{\ell+1})^\top (A^\ell_{V \setminus S})^\top \frac{ R_S(a_i^\ell)}{\| R_S(a_i^\ell)\|} \|_2^2 \geq \min_{\| z \|_2 = 1, \| z \|_0 \leq |V \setminus S|} \|  (W^{\ell+1})^\top z\|^2_2 $.
Let $v_j = x^S(A^\ell)_{ij}$ for $j \in S$, $v_i = -1$, and $z = v / \| v\|_2$, then $\| v \|_2 \geq 1$, and 
\begin{align*}
\| R_S(a_i^\ell)\|^2 = \| A^\ell v\|^2_2 \geq \| A^\ell z\|^2_2 \geq \min_{\| z \|_2 \leq 1, \| z \|_0 \leq |S|+1} \| A^\ell z\|^2_2
\end{align*}
We thus have $F(i | S)  \geq \min_{\| z \|_2 = 1, \| z \|_0 \leq |V \setminus S|} \|  (W^{\ell+1})^\top z\|^2_2 \min_{\| z \|_2 \leq 1, \| z \|_0 \leq |S|+1} \| A^\ell z\|^2_2$. 

The bound on $\alpha_k$ then follows by noting that $\| R_S(a_i^\ell)\| \geq \| R_T(a_i^\ell)\|$ for all $S \subseteq T$.
The rest of the proposition follows by noting that if any collection of $k+1$ columns of $A^\ell$ are linearly independent, then $\| R_S(a_i^\ell)\| \geq \min_{\| z \|_2 \leq 1, \| z \|_0 \leq k+1} \| A^\ell z\|^2_2 >0$ for any $S$ such that $|S| \leq k$.
\end{proof}
As discussed in Section \ref{sec:addPrelim}, Proposition \ref{prop:PruningWeaklyDRmod-Asym} implies that \Greedy achieves a $(1 - e^{-\alpha_{2k-1}})$-approximation with $F(S)$, where $\alpha_k$ is non-zero if all rows of $A^\ell W^{\ell+1}$ are linearly independent. Moreover, 
if any $k + 1$ columns of $A^\ell$ are linearly independent and all rows of $W^{\ell+1}$ are linearly independent, then \Greedy achieves an $\alpha_k$-approximation.


It is worth noting that if $W^{\ell+1}$ is identity matrix, i.e., $F$ is the column subset selection function, then  \cref{prop:PruningWeaklyDRmod} implies a stronger result than \citep[Lemma 6]{Sviridenko2017} for some cases. In particular, we have $$\alpha_k \geq \frac{\max\{\min_{\| z \|_2 \leq 1, \| z \|_0 \leq k+1} \| A^\ell z\|^2_2, \min_{\| z \|_2 = 1} \| (A^\ell)^\top z\|^2_2\}}{\max_{\| z \|_2 \leq 1} \|( A^\ell)^\top z\|^2_2\}},$$ which implies that $F$ is weakly DR-submodular if any $k$ columns of $A^\ell$ are linearly independent, or if all rows of $A^\ell$ are linearly independent.

%% file: appendix-subbounds.tex
\section{Empirical values of the submodularity ratio} \label{app:gammaValues}
\mtodo{I replaced all 0.5 values by 1, since in this case $A^\ell$ is full rank, hence any $k$ works. discuss also the empirical values of alpha}

As discussed in \cref{sec:Greedy}, computing the lower bounds on the submodularity ratio $\gamma_{U, k}$ in Proposition \ref{prop:WeaklySub} and \ref{prop:WeaklySubChannels} is NP-Hard \citep{Das2011} ($\min_{\| z \|_2 = 1, \| z \|_0 \leq |U|+k} \|A^\ell z\|_2^2$ corresponds to $\lambda_{\min}(C, |U| +k)$ in their notation). One simple lower bound that can be obtained from the eigenvalue interlacing theorem is $\gamma_{U, k} \geq \frac{\lambda_{\min}((A^\ell)^\top A^\ell)}{\lambda_{\max}((A^\ell)^\top A^\ell)}$. However, such bound is too loose as it is often equal to zero. For this reason, we focus on when our lower bounds on $\gamma_{\hat{S}, k}$ are non-zero: the lower bounds in Proposition \ref{prop:WeaklySub} and \ref{prop:WeaklySubChannels} are non-zero if any $\min\{2 k, n_\ell \}$ and $\min\{2 k, n_\ell \} r_h r_w$ columns of $A^\ell$ are linearly independent, respectively.  We report in Table \ref{table:gammaValues-all} an upper bound on $k$ for which these conditions hold, based on the rank of the activation matrix $A^\ell$, in each pruned layer in the three models we used in our experiments. 

\begin{table}[!ht]
\caption{Largest possible $k$ for which our lower bounds on the submodularity ratio $\gamma_{\hat{S}, k}$ are non-zero, when using all patches per image.}\label{table:gammaValues-all}
    \centering
    \begin{adjustbox}{width=1\textwidth}
    \begin{tabular}{l|l|l}
\toprule
  Dataset &  Model & upper bound on $k / n_\ell$  (all patches, $n=512$) \\ 
        \midrule
  MNIST &  LeNet & conv1: 0.37, conv2: 1, fc1: 0.46, fc2: 0.49 \\ \hline
  CIFAR10 &  VGG11 & features.0: 0.22, features.4: 0.39, features.8: 0.25,  features.11: 0.28, features.15: 0.06, \\ 
      &  & features.18: 0.03,  features.22: 0.01, classifier.0: 1, classifier.3: 1 \\ \hline
 CIFAR10 & ResNet56 & layer1.0.conv1: 0.12, layer1.1.conv1: 0.15, layer1.2.conv1: 0.22,  layer1.3.conv1: 0.22, \\ 
      &  & layer1.4.conv1: 0.30, layer1.5.conv1: 0.28,  layer1.6.conv1: 0.44, layer1.7.conv1: 0.35, \\ 
      &  & layer1.8.conv1: 0.15,  layer2.0.conv1: 0.48, layer2.1.conv1: 0.47, layer2.2.conv1: 0.48, \\ 
      &  & layer2.3.conv1: 0.47, layer2.4.conv1: 0.48, layer2.5.conv1: 0.47,  layer2.6.conv1: 0.48, \\ 
      &  & layer2.7.conv1: 0.47, layer2.8.conv1: 0.48,  layer3.0.conv1: 0.47, layer3.1.conv1: 0.49, \\ 
      &  & layer3.2.conv1: 0.46,  layer3.3.conv1: 0.48, layer3.4.conv1: 0.48, layer3.5.conv1: 0.49, \\ 
      &  & layer3.6.conv1: 0.48, layer3.7.conv1: 0.48, layer3.8.conv1: 1 \\ \bottomrule
    \end{tabular}
    \end{adjustbox}

\end{table}

We observe that the upper bound is close to $0.5$ for most linear layers, but is very small in some convolution layers (e.g., features.15, 18, 22 in VGG11). 
As explained in Section \ref{sec:onelayer-channels}, the linear independence condition required for convolution layers (Proposition 4.1) only holds for very small $k$, due to the correlation between patches which overlap. This can be avoided in most layers by sampling $r_h r_w$ random patches from each image (to ensure $A^\ell$ is a tall matrix), instead of using all patches. We report in \cref{table:gammaValues-rnd} the corresponding upper bounds on $k$ in this setting for the VGG11 model.

\begin{table}[!ht]
\caption{Largest possible $k$ for which our lower bounds on the submodularity ratio $\gamma_{\hat{S}, k}$ are non-zero, in VGG11 on CIFAR10, when using $r_h r_w$ random patches per image.}\label{table:gammaValues-rnd}
    \centering
    \begin{adjustbox}{width=1\textwidth}
    \begin{tabular}{l|l|l}
\toprule
 Dataset & Model & upper bound on $k / n_\ell$  (random patches, $n=512$) \\ 
 \midrule
 CIFAR10 & VGG11 & features.0: 0.38, features.4: 0.43, features.8: 0.29, features.11: 0.31, features.15: 0.15, \\
        & & features.18: 0.08, features.22: 0.1, classifier.0: 1, classifier.3: 1 \\ \bottomrule
    \end{tabular}
        \end{adjustbox}
\end{table}


The upper bounds are indeed larger than the ones obtained with all patches. However, some layers (e.g., features.15, 18, 22) have a very small feature map size (respectively $4\times 4, 4\times 4, 2\times 2$) so that even the small number of random patches have significant overlap, resulting in still a very small upper bound. Our experiments with random patches on VGG11 yielded worst results, so we chose to use all patches.
Note that our lower bounds on $\gamma_{\hat{S}, k}$ are not necessarily tight (see \cref{app:tightness}). Hence, if $k$ is outside these ranges, our lower bound on $\gamma_{\hat{S}, k}$ is zero, but not necessarily $\gamma_{\hat{S}, k}$ itself; indeed in our experiments our methods still perform well in these cases. 

\section{Tightness of lower bounds on the submodularity ratio}\label{app:tightness}

In this section, we investigate how tight are the lower bounds on the submodularity ratio $\gamma_{U, k}$ in Propositions \ref{prop:WeaklySub} and \ref{prop:WeaklySubChannels}. One trivial example where these bounds are tight is when $A^\ell$ is the identity matrix. In this case, both $F$ and $G$ are submodular, hence their corresponding $\gamma_{U, k}=1$ for all $U$ and $k$, and the lower bounds in both Proposition 3.1 and 4.1 are also equal to one. We present below another more interesting example where the bounds are tight.

\begin{proposition}
Given any matrix $A^\ell$ whose columns have equal norm, there exists a matrix $W^{\ell+1}$ such that the corresponding function $F$ has $\gamma_{\emptyset, k} = \frac{\min_{\| z \|_2 = 1, \| z \|_0 \leq k} \|A^\ell z\|_2^2}{\max_{\| z \|_2 = 1, \| z \|_0 \leq 1} \|A^\ell z\|^2_2}$.
\end{proposition}
\begin{proof}
Given a set $S$, let $v_{\min}^S, u_{\min}^S$ be the right and left singular vectors of $A^\ell_S$ corresponding to the smallest singular value $\sigma_{\min}^S$ of $A^\ell_S$, i.e., $A^\ell_S  v_{\min}^S = \sigma_{\min}^S u_{\min}^S$ and $(A^\ell_S)^\top  u_{\min}^S = \sigma_{\min}^S v_{\min}^S$. We consider the case where all columns of $A^\ell$ have equal norm, which we denote by $\sigma_{\max}^1$, and $W^{\ell+1}$ have all columns equal to $c v_{\min}^S$ for some scalar $c>0$. Then for all $m$, $A^\ell_S w^{\ell+1}_m = c \sigma_{\min}^S u_{\min}^S$ and the minimum of $\min_{\supp(x) \subseteq S} \|c \sigma_{\min}^S u_{\min}^S - A^\ell x\|_2^2$ is obtained at $x_S = c v_{\min}^S$. We can thus write
\begin{align*}
F(S) &= \sum_{m=1}^{n_{\ell+1}} \|c \sigma_{\min}^S u_{\min}^S\|_2^2  - \min_{\supp(x) \subseteq S} \|c \sigma_{\min}^S u_{\min}^S - A^\ell x\|_2^2\\
&= \sum_{m=1}^{n_{\ell+1}} (c \sigma_{\min}^S)^2  - \|c \sigma_{\min}^S u_{\min}^S - c\sigma_{\min} u_{\min}^S\|_2^2\\
&= n_{\ell+1}  (c \sigma_{\min}^S)^2
\end{align*}
On the other hand, for any $i \in S$, let $y =  c \sigma_{\min}^S u_{\min}^S$, the minimum of $\min_{\supp(x) \subseteq \{i\}} \|y - A^\ell x\|_2^2$ is obtained at $x_i = \tfrac{(a^\ell_i)^\top y}{\| a^\ell_i\|_2^2}$. We have
\begin{align*}
F(i) &= \sum_{m=1}^{n_{\ell+1}} \|y\|_2^2  - \min_{\supp(x) \subseteq \{i\}} \|y- A^\ell x\|_2^2\\
&=\sum_{m=1}^{n_{\ell+1}} \|y\|_2^2  - \|y- a^\ell_i \tfrac{(a^\ell_i)^\top y}{\| a^\ell_i\|_2^2} \|_2^2\\
&= n_{\ell+1} \frac{((a^\ell_i)^\top y)^2}{\| a^\ell_i\|_2^2}
\end{align*}
Hence, 
\begin{align*}
\sum_{i \in S} F(i) &= n_{\ell+1} \sum_{i \in S} \frac{((a^\ell_i)^\top y)^2}{\| a^\ell_i\|_2^2} \\
&= n_{\ell+1} \frac{\| (A^\ell_S)^\top y \|_2^2}{(\sigma_{\max}^1)^2}\\
&= n_{\ell+1} \frac{c^2 (\sigma_{\min}^S)^4}{(\sigma_{\max}^1)^2}
\end{align*}
Then $$\gamma_{\emptyset, k} \leq \frac{\sum_{i \in S} F(i)}{F(S)} = \frac{(\sigma_{\min}^S)^2}{(\sigma_{\max}^1)^2} = \frac{\min_{\| z \|_2 = 1, \supp(z) = S} \|A^\ell z\|_2^2}{\max_{\| z \|_2 = 1, \| z \|_0 = 1} \|A^\ell z\|^2_2}.$$
If we choose $S$ such that $S \in \argmin_{|S| \leq k} (\sigma_{\min}^S)^2$,  we get $\gamma_{\emptyset, k} = \frac{\min_{\| z \|_2 = 1, \| z \|_0 \leq k} \|A^\ell z\|_2^2}{\max_{\| z \|_2 = 1, \| z \|_0 \leq 1} \|A^\ell z\|^2_2}$ by \cref{prop:WeaklySub}.
\end{proof}
\looseness=-1 Since $F$ is a special case of $G$ where $M$ is the identity map, the above example applies to $G$ too.
On the other hand, there are also cases where these bounds are not tight. In particular, there are cases where the lower bound on $\alpha_{|U|+k-1}$ in 
\cref{prop:PruningWeaklyDRmod} is larger than the lower bound on $\gamma_{U, k}$ in \cref{prop:WeaklySub}, which implies that the latter is not tight since $\gamma_{U, k} \geq \alpha_{|U|+k-1}$ (see Section \ref{sec:addPrelim}). For example, if all rows of $A^\ell W^{\ell+1}$ are linearly independent, but there exists $2k$ columns of $A^\ell$ which are linearly dependent, then the bound in \cref{prop:WeaklySub} is zero while the one in  \cref{prop:PruningWeaklyDRmod} is not.
These borderline cases are unlikely to occur in practice. Whether we can tighten these lower bounds based on realistic assumptions on the weights and activations is an interesting future research question.
\mtodo{discuss also the tightness of our lower bounds on alpha}

%% file: appendix-exps.tex
\section{Experimental setup}\label{app:setup}

Our code uses Pytorch \citep{Paszke2017} and builds on the open source ShrinkBench library introduced in \citep{Blalock2020}. We use the code from \citep{Buschjaeger2020} for \Greedy. Our implementation of \LayerSampling is adapted from the code provided in \citep{Liebenwein2020}. We implemented the version of \LayerGreedyFS implemented in the code of \citep{Ye2020}, which differs from the version described in the paper: added neurons/channels are not allowed to be repeated. We use the implementation of LeNet and ResNet56 included in ShrinkBench \citep{Blalock2020}, and a modified version of the implementation of VGG11 provided in \citep{Phan2021}, where we changed the number of neurons in the first two layers to $128$.  

\mtodo{add hardware used, just list type of CPUs and GPUs on the clusters I used, and explain that exps were run on various clusters with various resources.}

We conducted experiments on 3 different clusters with the following resources (per experiment):
\begin{itemize}
\item Cluster 1:  1 $\times$ NVidia A100 with 40G memory, 20 $\times$ AMD Milan 7413 @ 2.65 GHz  / AMD Rome 7532 @ 2.40 GHz 
\item Cluster 2: 1 $\times$ NVIDIA P100 Pascal with 12G/16G memory / NVIDIA V100 Volta with 32G memory, 20 $\times$ Intel CPU of various types 
\item Cluster 3: 1 $\times$ NVIDIA Quadro RTX 8000 with 48G memory / NVIDIA Tesla M40 with 24G memory / NVIDIA TITAN RTX with 24G memory, 6 $\times$ CPU of various types. 
\end{itemize}
Pruning and fine-tuning with limited data was done on CPUs for all methods, a GPU was used only when fine-tuning with full data.

All our experiments used the following setup:

Random seeds: $42,43,44,45,46$

Pruning setup:
\begin{itemize}
\item Number of Batches: 4 (sampled at random from the training set)
\item Batch size: 128
\item Values used for compression ratio: $$c \in \{1, 2, 4, 8, 16, 32, 64, 128\}$$
\item Values used for per-layer fraction selection: 
$$\alpha_\ell \in \{0.01, 0.05,
 0.075, 0.1, 0.15, 0.2, \cdots , 0.95, 1.0\}$$
\item Verification set used for the budget selection method in \cref{sec:fractionSel}: random subset of training set of same size as validation set.
\end{itemize}

Training and fine-tuning setup for LeNet on MNIST:
\begin{itemize}
\item Batch size: 128
\item Epochs for pre-training: 200
\item Epochs for fine-tuning: 10
\item Optimizer for pre-training: SGD with Nestrov momentum $0.9$
\item Optimizer for fine-tuning: Adam with $\beta_1=0.9$ and $\beta_2= 0.99$
\item Initial learning rate: $1 \times 10^{-3}$
\item Learning rate schedule: Fixed
\end{itemize}

Training and fine-tuning setup for VGG11 on CIFAR10:
\begin{itemize}
\item Batch size: 128
\item Epochs for pre-training: 200
\item Epochs for fine-tuning: 20
\item Optimizer for pre-training: Adam with $\beta_1=0.9$ and $\beta_2= 0.99$
\item Optimizer for fine-tuning: Adam with $\beta_1=0.9$ and $\beta_2= 0.99$
\item Initial learning rate: $1 \times 10^{-3}$
\item Weight decay: $5 \times 10^{-4}$
\item Learning rate schedule for pre-training: learning rate dropped by $0.1$ at epochs $100$ and $150$
\item Learning rate schedule for fine-tuning: learning rate dropped by $0.1$ at epochs $10$ and $15$
\end{itemize}

The setup for fine-tuning ResNet56 on CIFAR10 is the same one used for VGG, as outlined above.

\section{Effect of fine-tuning}\label{app:finetuning} 

\begin{figure}
\vspace{-25pt}
\begin{subfigure}{.31\textwidth}
 \centering
 \hspace*{-10pt}
\includegraphics[trim=65 50 1410 35, clip, scale=0.1]{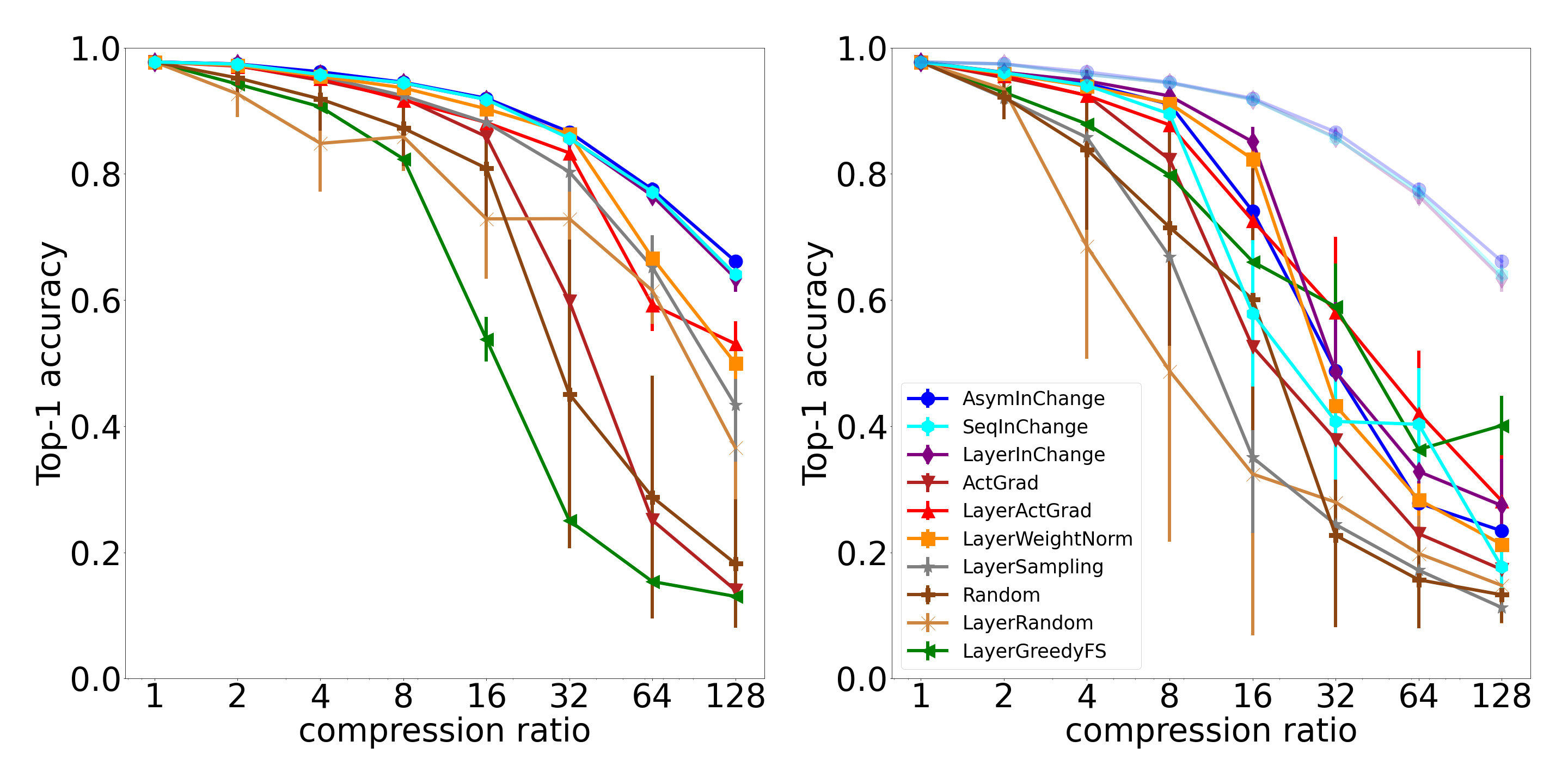}
\end{subfigure}\hfill
\begin{subfigure}{.31\textwidth}
 \centering
   \hspace*{-5pt}
\includegraphics[trim=65 50 1430 35, clip, scale=0.1]{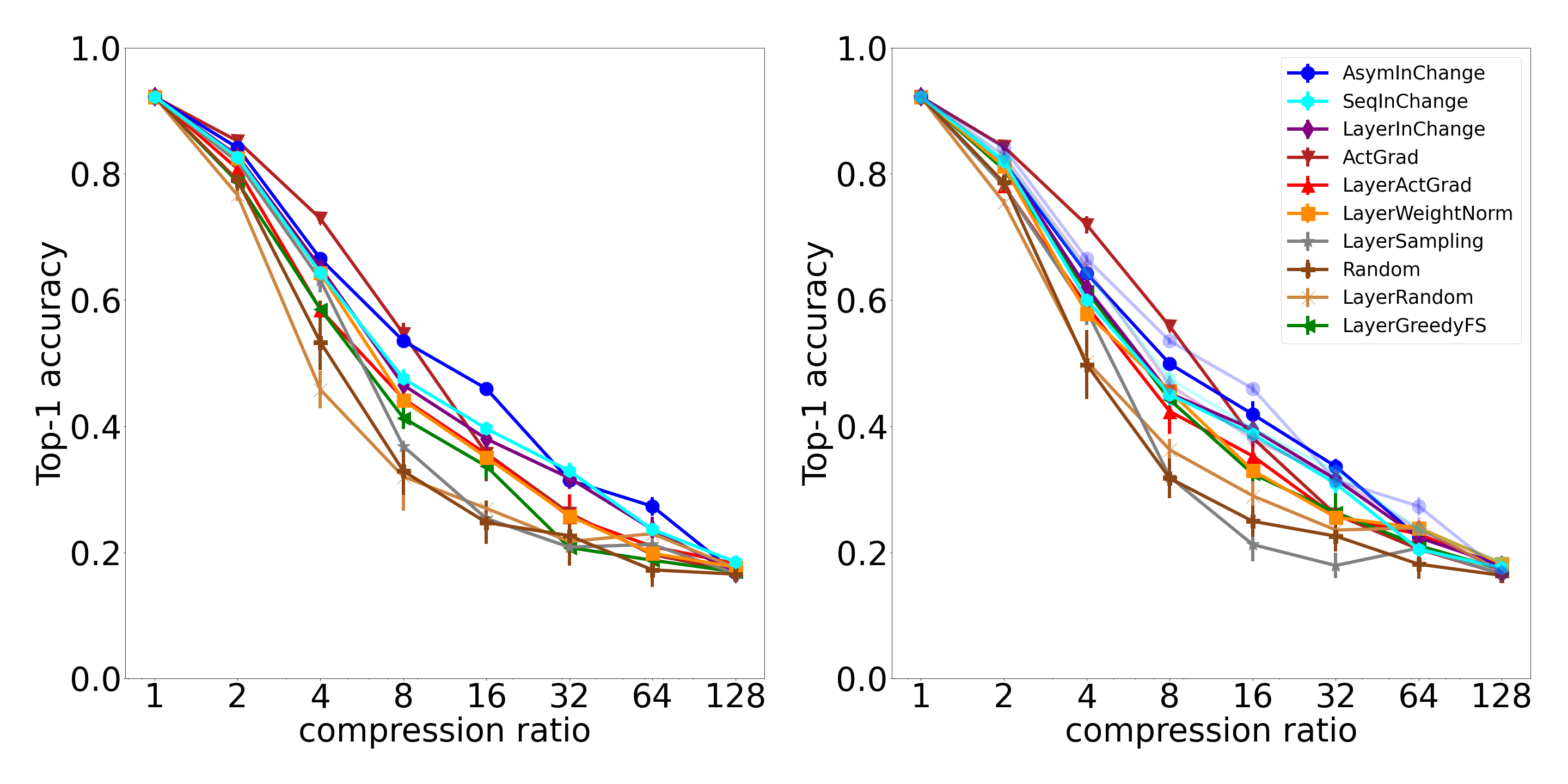}
\end{subfigure}\hfill
\begin{subfigure}{.31\textwidth}
 \centering
\includegraphics[trim=65 50 1410 35, clip, scale=0.1]{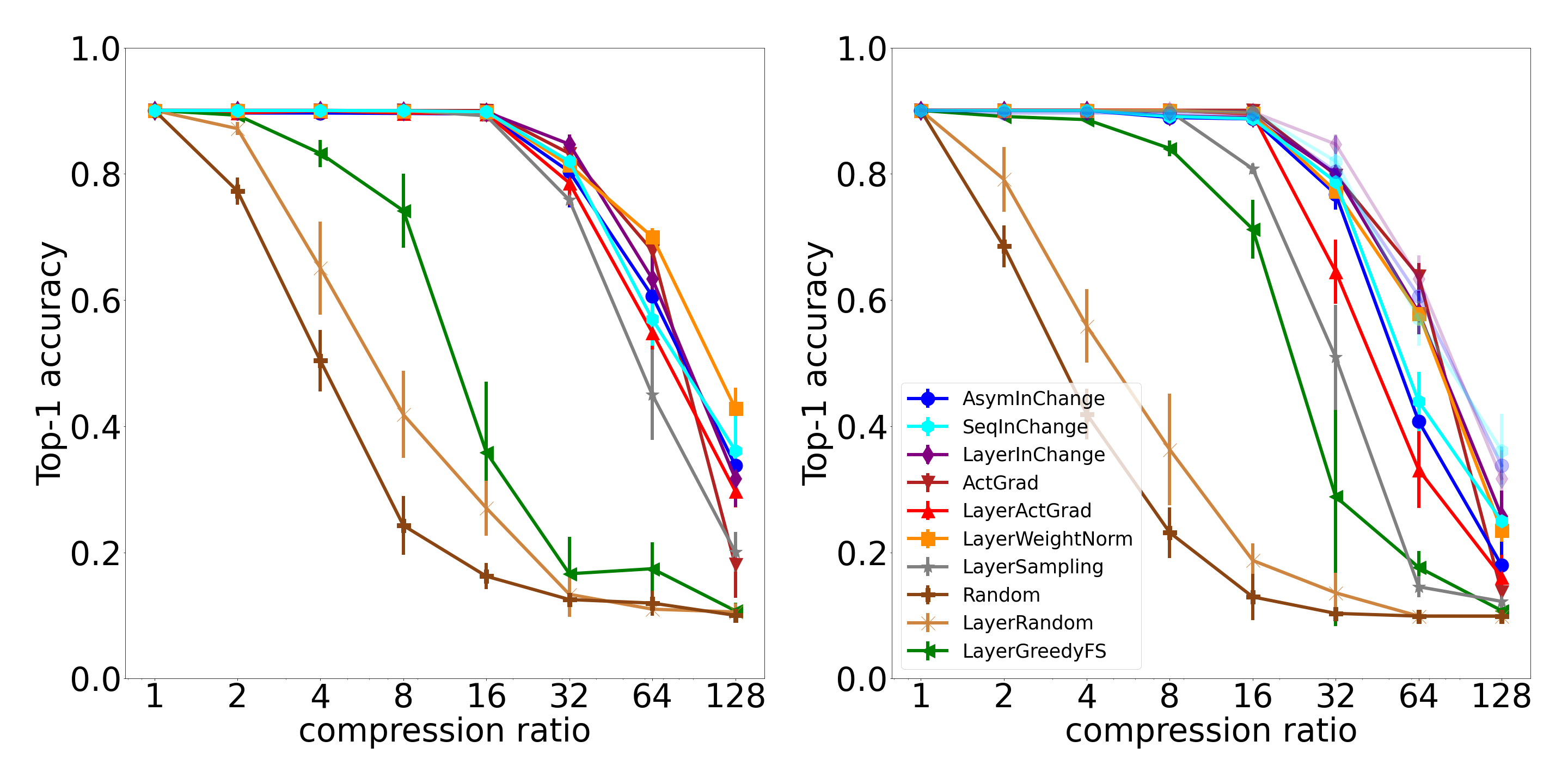}
\end{subfigure}
\\
\begin{subfigure}{.31\textwidth}
 \centering
  \hspace*{-10pt}
\includegraphics[trim=1475 50 0 35, clip, scale=0.1]{acc1-compression-finetuned-7490ce2-2325370-all.png}
\end{subfigure}\hfill
\begin{subfigure}{.31\textwidth}
 \centering
   \hspace*{-5pt}
\includegraphics[trim=1450 50 0 35, clip, scale=0.1]{acc1-compression-finetuned-7490ce2-9872006-all.png}
\end{subfigure}\hfill
\begin{subfigure}{.31\textwidth}
 \centering
\includegraphics[trim=1475 50 0 35, clip, scale=0.1]{acc1-compression-finetuned-6237752-46028840-all.png}
\end{subfigure}
\caption{\label{fig:finetuned-ld} \looseness=-1 Top-1 Accuracy of different pruning methods, after fine-tuning with four batches of training data, applied to LeNet on MNIST (left), ResNet56 on CIFAR10 (middle), and VGG11 on CIFAR10 (right), for several compression ratios (in log-scale), with (top) and without (bottom) reweighting. We include the three reweighted variants of our method in the bottom plots (faded) for reference.} 
\end{figure}

\begin{figure}
\vspace{-5pt}
\begin{subfigure}{.31\textwidth}
 \centering
 \hspace*{-10pt}
\includegraphics[trim=65 50 1410 35, clip, scale=0.1]{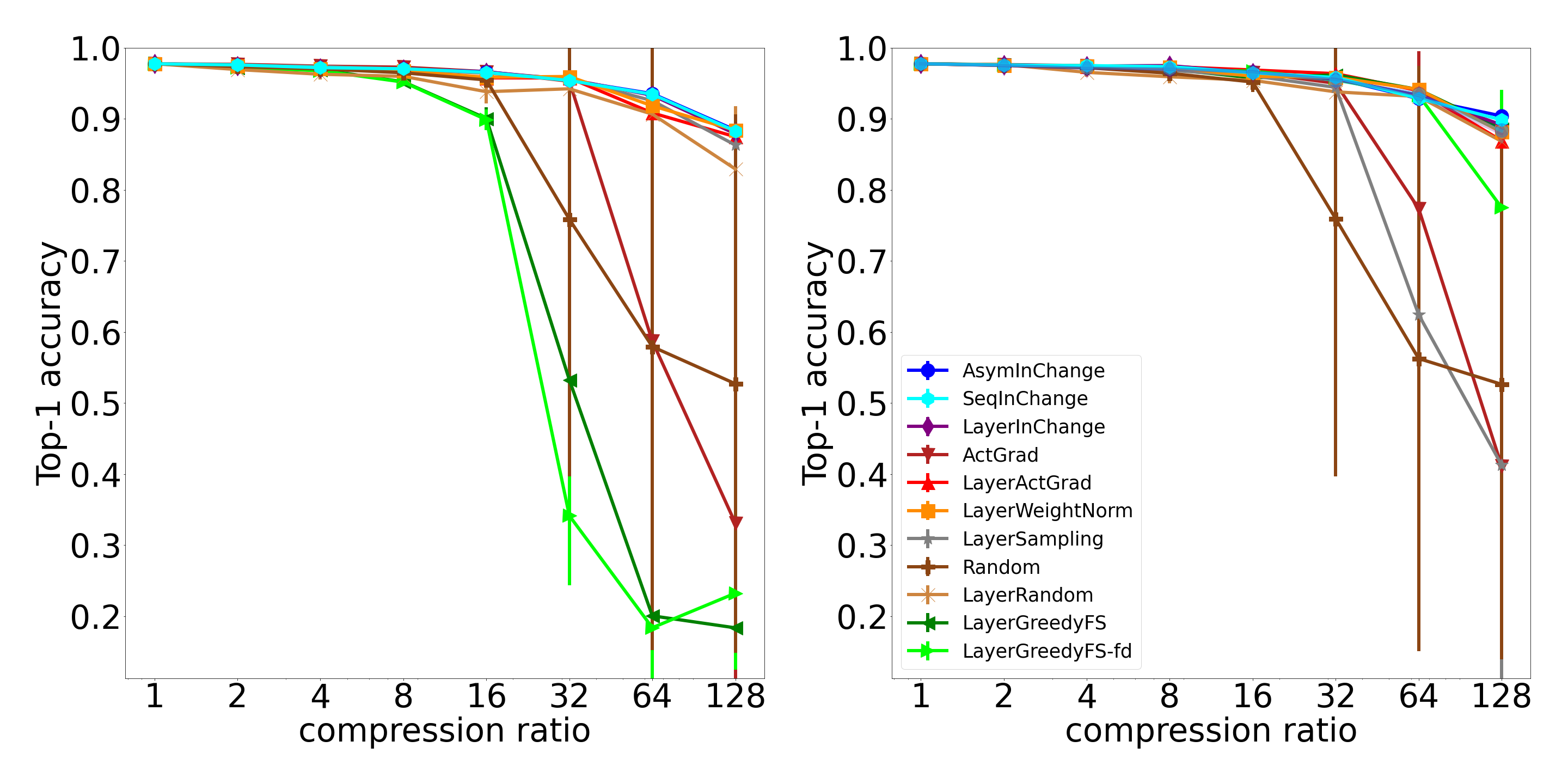}
\end{subfigure}\hfill
\begin{subfigure}{.31\textwidth}
 \centering
  \hspace*{-5pt}
\includegraphics[trim=65 50 1430 35, clip, scale=0.1]{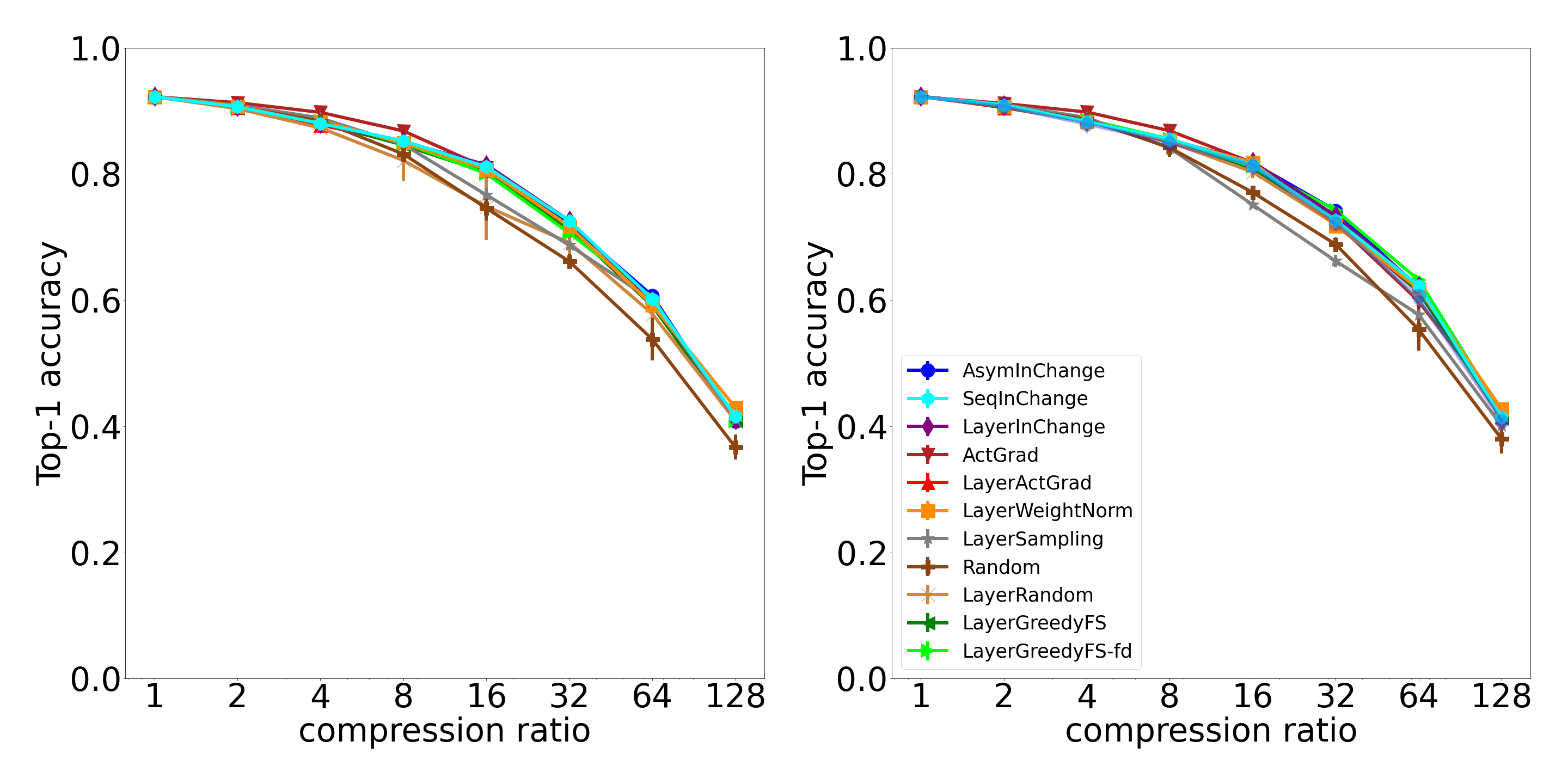}
\end{subfigure}\hfill
\begin{subfigure}{.31\textwidth}
 \centering
\includegraphics[trim=65 50 1410 35, clip, scale=0.1]{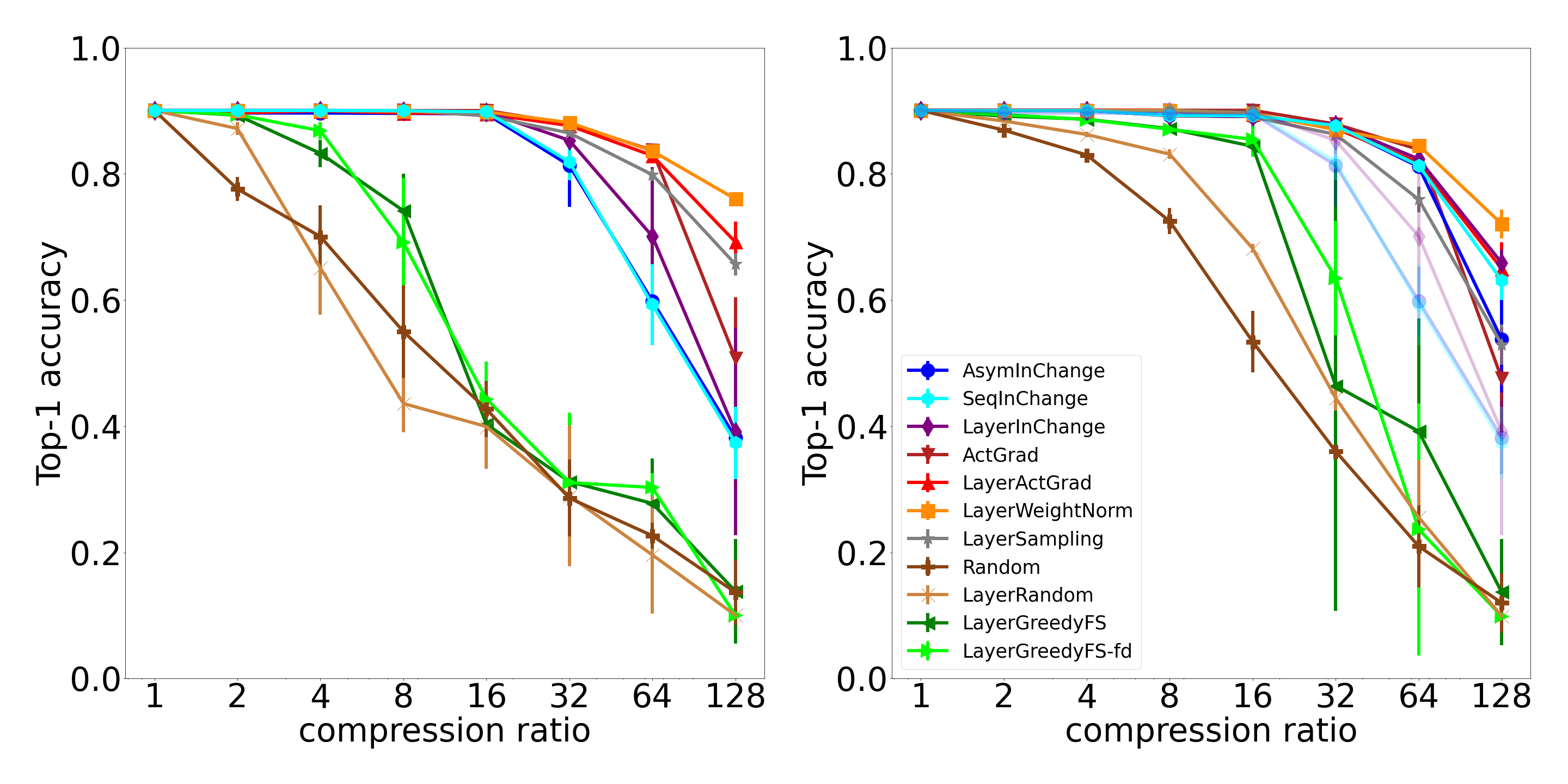}
\end{subfigure}
\\
\begin{subfigure}{.31\textwidth}
 \centering
  \hspace*{-10pt}
\includegraphics[trim=1475 50 0 35, clip, scale=0.1]{acc1-compression-finetuned-1ae094f-2057043-all.png}
\end{subfigure}\hfill
\begin{subfigure}{.31\textwidth}
 \centering
  \hspace*{-5pt}
\includegraphics[trim=1450 50 0 35, clip, scale=0.1]{acc1-compression-finetuned-11b9a94-7862381-all.png}
\end{subfigure}\hfill
\begin{subfigure}{.31\textwidth}
 \centering
\includegraphics[trim=1475 50 0 35, clip, scale=0.1]{acc1-compression-finetuned-0ebdd32-40042510-all.png}
\end{subfigure}
\caption{\label{fig:finetuned-fd} \looseness=-1 Top-1 Accuracy of different pruning methods, after fine-tuning with the full training data, applied to LeNet on MNIST (left), ResNet56 on CIFAR10 (middle), and VGG11 on CIFAR10 (right), for several compression ratios (in log-scale), with (top) and without (bottom) reweighting. We include the three reweighted variants of our method in the bottom plots (faded) for reference.} 
\end{figure}

\looseness=-1 In this section, we study the effect of fine-tuning with both limited and sufficient data. To that end, we report the top-1 accuracy results of all the pruning tasks considered in \cref{sec:exps}, after fine-tuning with only four batches of training data, 
and after fine-tuning with the full training data in Figure \ref{fig:finetuned-fd}. 
We fine-tune for 10 epochs in the MNIST experiment, and for 20 epochs in both CIFAR-10 experiments. We do not fine-tune at compression ratio 1 (i.e., when nothing is pruned).

Our method still outperforms other baselines after fine-tuning with limited-data, and is among the best performing methods even in the full-data setting (if we consider the non-reweighted variants for VGG11 model).
As expected, fine-tuning with the full training data provides a significant boost in performance to all methods, even more than reweighting. Fine-tuning with limited data also helps but significantly less. 
Reweighting still improves the performance of all methods, except LAYERGREEDYFS and LAYERGREEDYFS-fd, even when fine-tuning with limited-data is used, but it can actually deteriorate performance when fine-tuning with full-data is used (see Figure \ref{fig:finetuned-fd} left-bottom plot, the reweighted variants of our methods have lower accuracy than the non-reweighted variants). We suspect that this 
could be due to overfitting to the very small training data used for pruning. This only happens with VGG11 model, because it is larger than LeNet and ResNet56 models. We expect the performance of the reweighted variants of our method to improve if we use more training data for pruning.



\section{Importance of per-layer budget selection}\label{app:fractionSel}

 In this section, we study the effect of per-layer budget selection on accuracy. To that end, we use the same pretrained VGG11 model from \cref{sec:exps}, and prune the first and second to last convolution layers in it.  Since the second to last layer in VGG11 (features.22) has little effect on accuracy when pruned, we expect the choice of how the global budget is distributed on the two layers to have a significant impact on performance. Figure \ref{fig:CIFAR10-Channel-FractSel} shows the top-1 accuracy for different fractions of prunable channels kept, when the per-layer budget selection from \cref{sec:fractionSel} is used, while Figure \ref{fig:CIFAR10-Channel-EqFract} shows the results when equal fractions of channels kept are used in each layer. As expected, all layerwise methods perform much more poorly with equal per-layer fractions, both with and without fine-tuning. Though, the difference is less drastic when fine-tuning is used.
 
\begin{figure}
\vspace{-25pt}
 \centering
\includegraphics[trim=40 50 50 35, clip, scale=0.1]{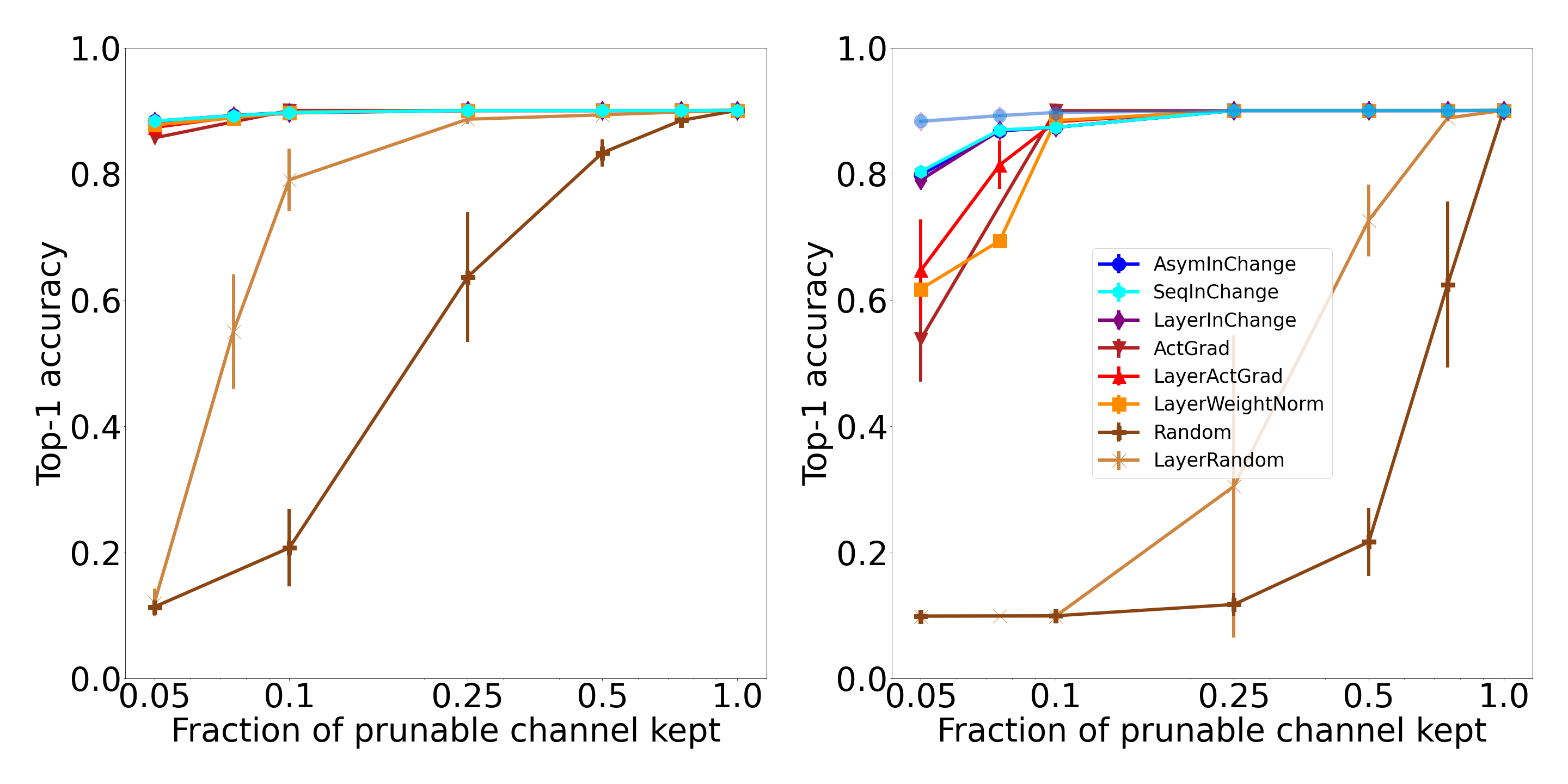}\\
 \includegraphics[trim=40 50 50 35, clip, scale=0.1]{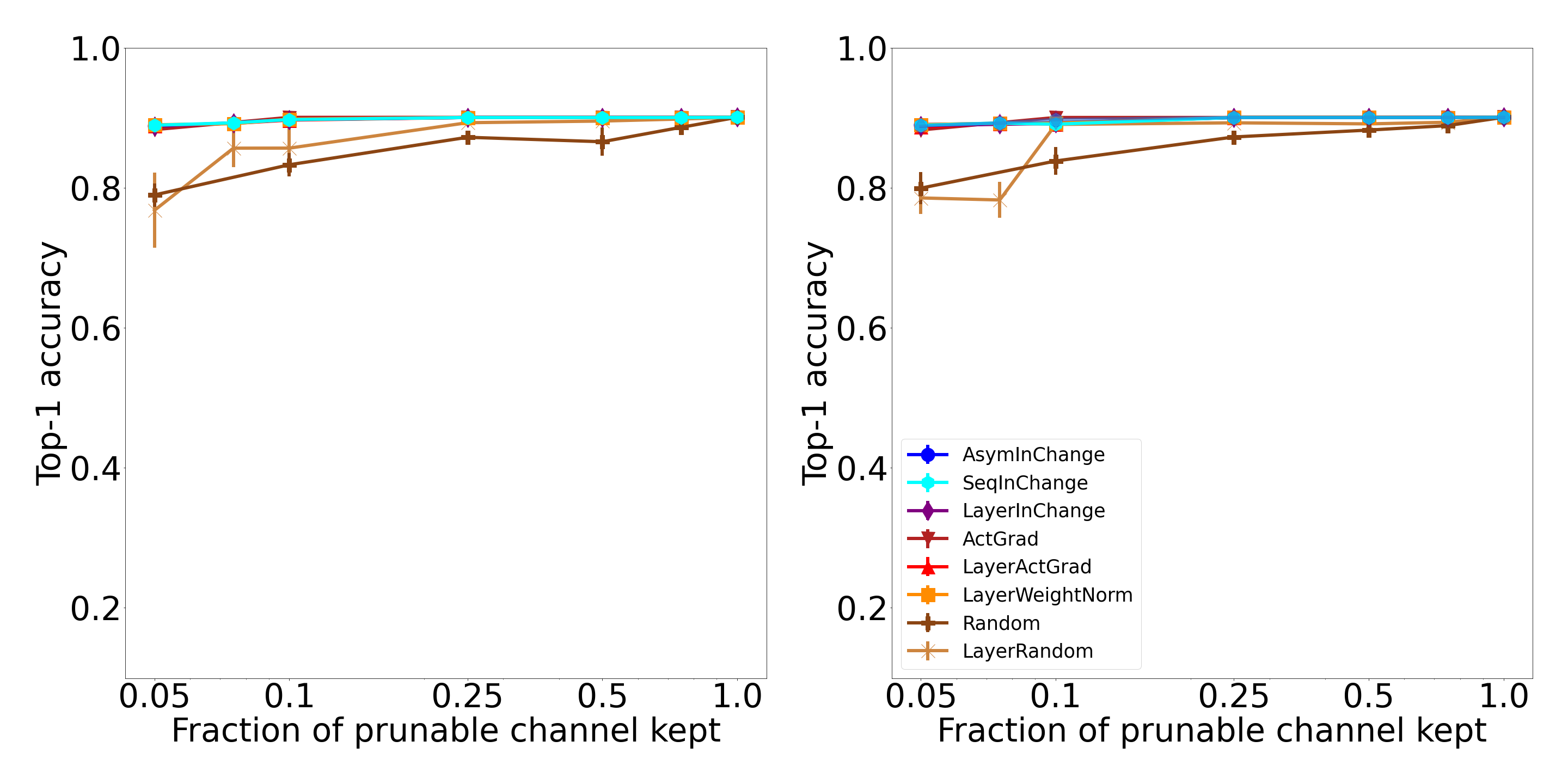}
\caption{\label{fig:CIFAR10-Channel-FractSel}  Top-1 Accuracy of different pruning methods on CIFAR10, after pruning the first and second to last convolution layers in VGG11 model, with different fractions of remaining channels (in log-scale), with (left) and without (right) reweighting, with (bottom) and without (top) fine-tuning, with per-layer fractions selected using the selection method discussed in \cref{sec:fractionSel}.  We include the three reweighted variants of our method in the plots without reweighting (faded) for reference.} 
\end{figure}

\begin{figure}
\vspace{-10pt}
 \centering
\includegraphics[trim=40 50 50 35, clip, scale=0.1]{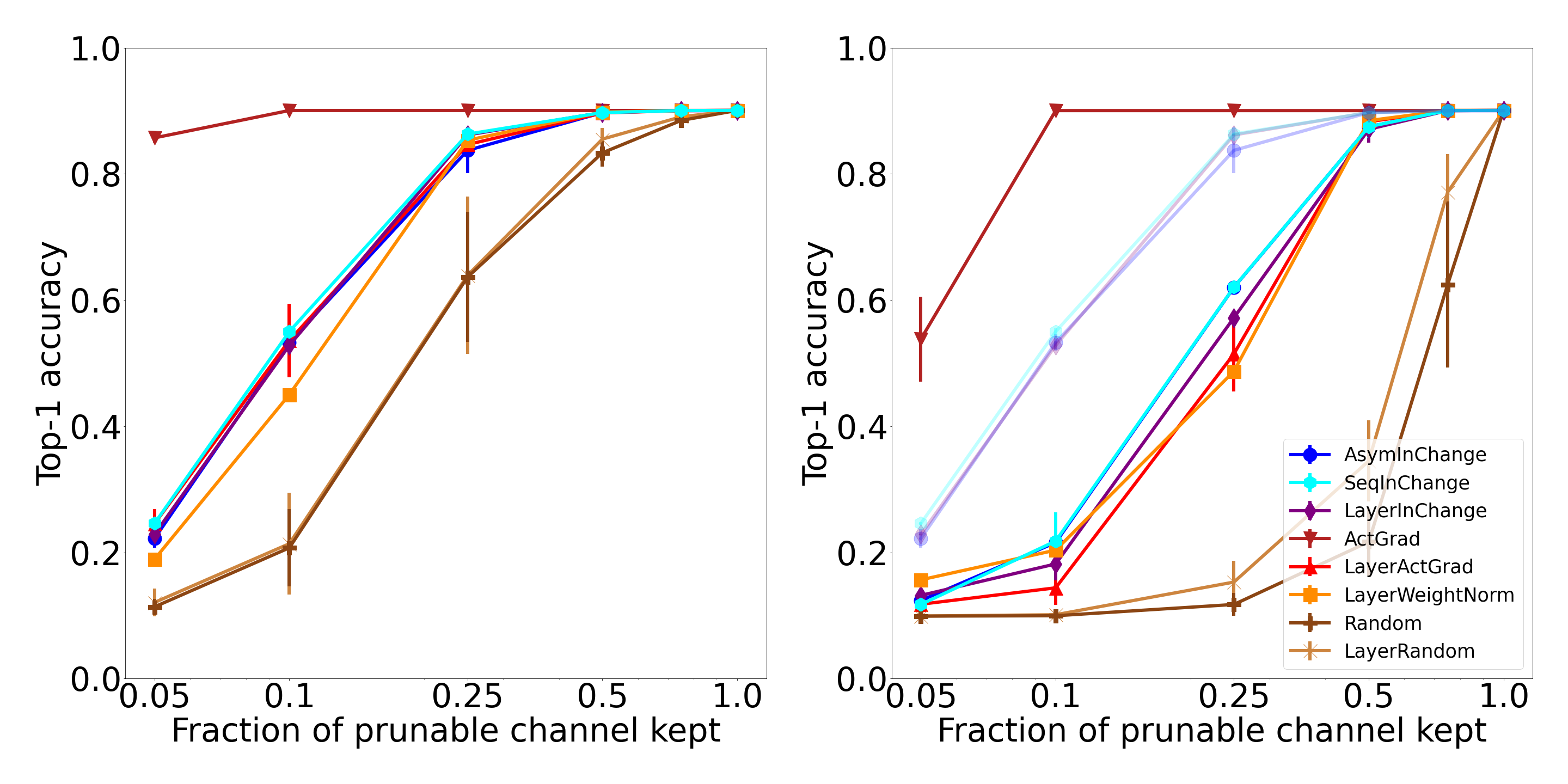}\\
 \includegraphics[trim=40 50 50 35, clip, scale=0.1]{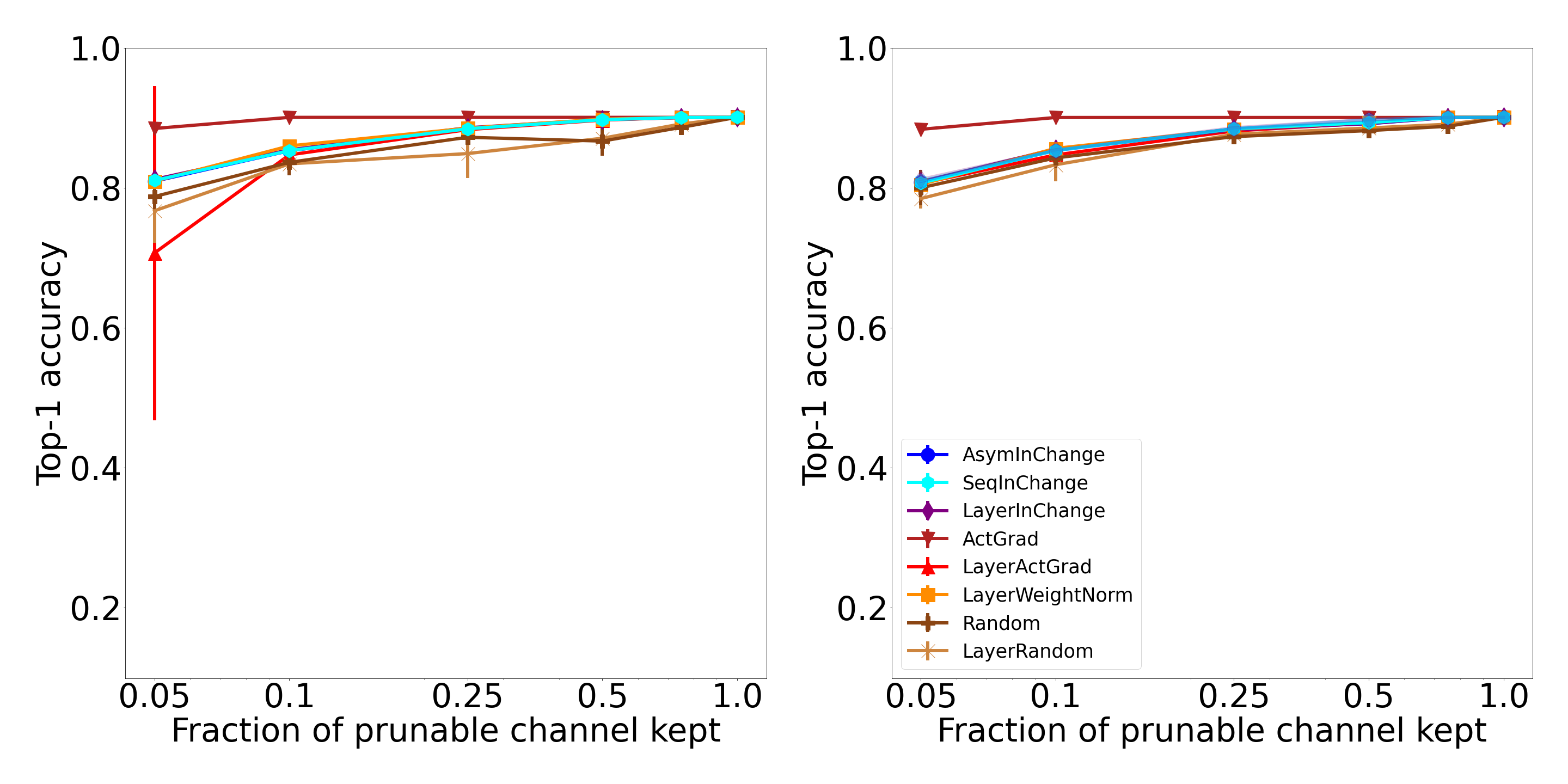}
\caption{\label{fig:CIFAR10-Channel-EqFract} Top-1 Accuracy of different  pruning methods on CIFAR10, after pruning the first and second to last convolution layers in VGG11 model, with different fractions of remaining channels (in log-scale), with (left) and without (right) reweighting, with (bottom) and without (top) fine-tuning, with equal per-layer fractions.  We include the three reweighted variants of our method in the plots without reweighting (faded) for reference.} 
\end{figure}

\section{Results with respect to other metrics} \label{app:othermetrics} 

We report in Tables \ref{table:MNIST-LeNet}, \ref{table:CIFAR10-ResNet56} and \ref{table:CIFAR10-VGG11}, the top-1 accuracy, speedup ratio ($\tfrac{\text{original number of FLOPs}}{\text{pruned number of FLOPs}}$), and pruning time values of the experiments presented in \cref{sec:exps}. We exclude the worst performing methods 
\Rand and \LayerRand. 

Note that for a given compression ratio speedup values vary significantly between different pruning methods, 
because the number of weights and flops vary between layers, and pruning methods differ in their per-layer budget allocations. The best performing methods in terms of compression are not necessarily the best ones in terms of speedup. For example, \ActGrad is among the best performing methods in terms of compression on ResNet56-CIFAR10, but it is the worst one in terms of speedup.
In cases where we care more about speedup than compression, we can replace the constraint in the per-layer budget selection problem \eqref{eq:BudgetSelProb} to be speedup instead of compression. 

 
Since our goal in these experiments was to study performance in terms of accuracy vs compression rate, we did not focus on optimizing our method's implementation for computation time efficiency. For example, our current implementation uses the classical Greedy algorithm. This can be significantly sped-up by switching to the faster Greedy algorithm from \cite{Li2022}.

\begin{table}
\caption{Top-1 Accuracy \% (Acc1), speedup ratio (SR), and pruning time (in hrs:mins:secs) of different pruning methods applied to LeNet on MNIST, with different compression ratios $c$, with and without reweighting (rw) and fine-tuning (ft).} \label{table:MNIST-LeNet}
\begin{center}
\begin{adjustbox}{width=1\textwidth}
\begin{tabular}{c|cc|ccc|ccc|ccc|ccc|ccc} 
\toprule
 & & & \multicolumn{3}{c}{$c=2$} &  \multicolumn{3}{c}{$c=4$} & \multicolumn{3}{c}{$c=8$} & \multicolumn{3}{c}{$c=16$} & \multicolumn{3}{c}{$c=32$} \\
           Method & rw & ft &  Acc1 &   SR &     time &  Acc1 &   SR &     time &  Acc1 &    SR &     time &  Acc1 &    SR &     time &  Acc1 &    SR &     time \\
\midrule
                  &  ✓ &  ✓ &  97.5 &  3.1 &  0:00:03 &  97.1 &  7.3 &  0:00:02 &  97.0 &   6.2 &  0:00:02 &  96.6 &   8.3 &  0:00:02 &  95.4 &   4.0 &  0:00:02 \\
     AsymInChange &  ✓ &  ✗ &  97.4 &  3.1 &  0:00:03 &  96.2 &  7.3 &  0:00:02 &  94.4 &   6.2 &  0:00:02 &  90.3 &   8.3 &  0:00:02 &  83.5 &   4.0 &  0:00:02 \\
                  &  ✗ &  ✓ &  97.6 &  3.1 &  0:00:03 &  97.4 &  4.0 &  0:00:02 &  97.4 &   2.4 &  0:00:02 &  96.2 &   8.3 &  0:00:01 &  95.7 &   3.5 &  0:00:02 \\
                  &  ✗ &  ✗ &  84.6 &  3.1 &  0:00:03 &  48.7 &  4.0 &  0:00:02 &  36.4 &   2.4 &  0:00:02 &  12.8 &   8.3 &  0:00:01 &  14.4 &   3.5 &  0:00:02 \\ \hline
                  &  ✓ &  ✓ &  97.6 &  3.1 &  0:00:03 &  97.2 &  7.3 &  0:00:02 &  97.1 &   6.2 &  0:00:02 &  96.5 &   8.3 &  0:00:02 &  95.4 &   4.0 &  0:00:02 \\
      SeqInChange &  ✓ &  ✗ &  97.4 &  3.1 &  0:00:03 &  95.8 &  7.3 &  0:00:02 &  94.3 &   6.2 &  0:00:02 &  89.7 &   8.3 &  0:00:02 &  82.2 &   4.0 &  0:00:02 \\
                  &  ✗ &  ✓ &  97.6 &  3.1 &  0:00:04 &  97.5 &  4.0 &  0:00:02 &  97.3 &   2.4 &  0:00:03 &  96.5 &   8.3 &  0:00:02 &  95.8 &   3.5 &  0:00:01 \\
                  &  ✗ &  ✗ &  88.5 &  3.1 &  0:00:04 &  39.0 &  4.0 &  0:00:02 &  25.8 &   2.4 &  0:00:03 &  12.4 &   8.3 &  0:00:02 &  10.7 &   3.5 &  0:00:01 \\  \hline
                  &  ✓ &  ✓ &  97.5 &  3.1 &  0:00:03 &  97.1 &  7.3 &  0:00:02 &  97.0 &   6.2 &  0:00:02 &  96.6 &   8.3 &  0:00:02 &  95.6 &   4.0 &  0:00:01 \\
    LayerInChange &  ✓ &  ✗ &  97.4 &  3.1 &  0:00:03 &  95.8 &  7.3 &  0:00:02 &  94.2 &   6.2 &  0:00:02 &  89.9 &   8.3 &  0:00:02 &  81.8 &   4.0 &  0:00:01 \\
                  &  ✗ &  ✓ &  97.6 &  3.1 &  0:00:03 &  97.4 &  4.0 &  0:00:02 &  97.5 &   2.4 &  0:00:02 &  96.5 &   8.3 &  0:00:02 &  95.6 &   3.5 &  0:00:01 \\
                  &  ✗ &  ✗ &  88.6 &  3.1 &  0:00:03 &  40.6 &  4.0 &  0:00:02 &  33.5 &   2.4 &  0:00:02 &  23.9 &   8.3 &  0:00:02 &  13.1 &   3.5 &  0:00:01 \\  \hline
                  &  ✓ &  ✓ &  97.7 &  2.0 &  0:00:00 &  97.4 &  2.6 &  0:00:00 &  97.3 &   3.4 &  0:00:00 &  96.6 &   4.1 &  0:00:00 &  95.3 &   4.8 &  0:00:00 \\
          ActGrad &  ✓ &  ✗ &  97.2 &  2.0 &  0:00:00 &  94.7 &  2.6 &  0:00:00 &  87.2 &   3.4 &  0:00:00 &  67.5 &   4.1 &  0:00:00 &  40.5 &   4.8 &  0:00:00 \\
                  &  ✗ &  ✓ &  97.6 &  2.0 &  0:00:00 &  97.5 &  2.6 &  0:00:00 &  97.2 &   3.4 &  0:00:00 &  96.6 &   4.1 &  0:00:00 &  95.0 &   4.8 &  0:00:00 \\
                  &  ✗ &  ✗ &  81.0 &  2.0 &  0:00:00 &  43.5 &  2.6 &  0:00:00 &  24.9 &   3.4 &  0:00:00 &  17.7 &   4.1 &  0:00:00 &  15.1 &   4.8 &  0:00:00 \\  \hline
                  &  ✓ &  ✓ &  97.6 &  3.1 &  0:00:00 &  97.1 &  4.7 &  0:00:00 &  96.6 &  11.2 &  0:00:00 &  95.8 &  16.2 &  0:00:00 &  95.8 &   4.0 &  0:00:00 \\
     LayerActGrad &  ✓ &  ✗ &  97.1 &  3.1 &  0:00:00 &  95.1 &  4.7 &  0:00:00 &  90.2 &  11.2 &  0:00:00 &  82.9 &  16.2 &  0:00:00 &  76.9 &   4.0 &  0:00:00 \\
                  &  ✗ &  ✓ &  97.5 &  3.2 &  0:00:00 &  97.4 &  4.4 &  0:00:00 &  97.4 &   2.4 &  0:00:00 &  96.9 &   3.0 &  0:00:00 &  96.4 &   3.1 &  0:00:00 \\
                  &  ✗ &  ✗ &  84.9 &  3.2 &  0:00:00 &  46.3 &  4.4 &  0:00:00 &  27.9 &   2.4 &  0:00:00 &  20.5 &   3.0 &  0:00:00 &  17.0 &   3.1 &  0:00:00 \\  \hline
                  &  ✓ &  ✓ &  97.5 &  3.1 &  0:00:00 &  97.1 &  6.7 &  0:00:00 &  96.7 &  11.2 &  0:00:00 &  95.9 &  16.2 &  0:00:00 &  96.0 &   4.0 &  0:00:00 \\
  LayerWeightNorm &  ✓ &  ✗ &  97.3 &  3.1 &  0:00:00 &  95.5 &  6.7 &  0:00:00 &  93.6 &  11.2 &  0:00:00 &  88.2 &  16.2 &  0:00:00 &  81.1 &   4.0 &  0:00:00 \\
                  &  ✗ &  ✓ &  97.5 &  3.0 &  0:00:00 &  97.5 &  3.0 &  0:00:00 &  97.3 &   2.6 &  0:00:00 &  96.0 &  14.7 &  0:00:00 &  95.8 &   3.5 &  0:00:00 \\
                  &  ✗ &  ✗ &  91.6 &  3.0 &  0:00:00 &  60.1 &  3.0 &  0:00:00 &  53.9 &   2.6 &  0:00:00 &  18.1 &  14.7 &  0:00:00 &  14.8 &   3.5 &  0:00:00 \\  \hline
                  &  ✓ &  ✓ &  97.5 &  1.9 &  0:00:20 &  97.4 &  2.7 &  0:00:21 &  97.0 &   3.5 &  0:00:19 &  96.4 &   4.4 &  0:00:20 &  95.7 &   5.2 &  0:00:19 \\
    LayerSampling &  ✓ &  ✗ &  97.2 &  1.9 &  0:00:20 &  95.5 &  2.7 &  0:00:21 &  91.6 &   3.5 &  0:00:19 &  85.5 &   4.4 &  0:00:20 &  71.7 &   5.2 &  0:00:19 \\
                  &  ✗ &  ✓ &  97.6 &  1.9 &  0:00:19 &  97.3 &  2.7 &  0:00:20 &  96.8 &   3.5 &  0:00:18 &  96.1 &   4.4 &  0:00:19 &  94.4 &   5.2 &  0:00:18 \\
                  &  ✗ &  ✗ &  41.8 &  1.9 &  0:00:19 &  25.5 &  2.7 &  0:00:20 &  22.4 &   3.5 &  0:00:18 &  13.2 &   4.4 &  0:00:19 &  14.1 &   5.2 &  0:00:18 \\  \hline
                  &  ✓ &  ✓ &  97.3 &  3.3 &  0:00:33 &  96.9 &  4.9 &  0:00:27 &  95.2 &   7.4 &  0:00:40 &  90.0 &  10.4 &  0:00:19 &  53.3 &   9.0 &  0:00:28 \\
    LayerGreedyFS &  ✓ &  ✗ &  94.1 &  3.3 &  0:00:33 &  88.1 &  4.9 &  0:00:27 &  76.2 &   7.4 &  0:00:40 &  44.9 &  10.4 &  0:00:19 &  23.8 &   9.0 &  0:00:28 \\
                  &  ✗ &  ✓ &  97.6 &  1.2 &  0:00:27 &  97.4 &  4.2 &  0:00:23 &  97.0 &   2.4 &  0:00:18 &  95.8 &  14.7 &  0:00:14 &  96.2 &   5.1 &  0:00:09 \\
                  &  ✗ &  ✗ &  73.5 &  1.2 &  0:00:27 &  45.0 &  4.2 &  0:00:23 &  28.0 &   2.4 &  0:00:18 &  28.3 &  14.7 &  0:00:14 &  20.1 &   5.1 &  0:00:09 \\  \hline
                  &  ✓ &  ✓ &  97.2 &  3.2 &  0:01:44 &  96.8 &  4.9 &  0:01:21 &  95.2 &   7.4 &  0:01:23 &  89.9 &  10.4 &  0:00:56 &  34.1 &  23.9 &  0:00:48 \\
 LayerGreedyFS-fd &  ✓ &  ✗ &  94.2 &  3.2 &  0:01:44 &  88.2 &  4.9 &  0:01:21 &  75.8 &   7.4 &  0:01:23 &  47.2 &  10.4 &  0:00:56 &  17.2 &  23.9 &  0:00:48 \\
                  &  ✗ &  ✓ &  97.6 &  1.2 &  0:01:30 &  97.4 &  4.4 &  0:01:19 &  97.0 &   2.4 &  0:01:01 &  96.7 &   7.5 &  0:00:36 &  95.5 &   3.5 &  0:00:21 \\
                  &  ✗ &  ✗ &  73.5 &  1.2 &  0:01:30 &  47.4 &  4.4 &  0:01:19 &  30.3 &   2.4 &  0:01:01 &  23.3 &   7.5 &  0:00:36 &  16.9 &   3.5 &  0:00:21 \\
\bottomrule
\end{tabular}
\end{adjustbox}
\end{center}
\end{table}

\begin{table}
\caption{Top-1 Accuracy \% (Acc1), speedup ratio (SR), and pruning time (in hrs:mins:secs) of different pruning methods applied to ResNet56 on CIFAR10,with different compression ratios $c$, with and without reweighting (rw) and fine-tuning (ft).} \label{table:CIFAR10-ResNet56}
\begin{center}
\begin{adjustbox}{width=1\textwidth}
\begin{tabular}{c|cc|ccc|ccc|ccc|ccc|ccc} 
\toprule
 & & & \multicolumn{3}{c}{$c=2$} &  \multicolumn{3}{c}{$c=4$} & \multicolumn{3}{c}{$c=8$} & \multicolumn{3}{c}{$c=16$} & \multicolumn{3}{c}{$c=32$} \\
           Method & rw & ft &  Acc1 &   SR &     time &  Acc1 &   SR &     time &  Acc1 &    SR &     time &  Acc1 &    SR &     time &  Acc1 &    SR &     time \\
\midrule
                  &  ✓ &  ✓ &  90.7 &  2.6 &  2:13:10 &  87.9 &  6.0 &  1:14:00 &  84.9 &  11.0 &  0:42:08 &  81.1 &  16.4 &  0:24:43 &  72.1 &  21.3 &  0:14:57 \\
     AsymInChange &  ✓ &  ✗ &  84.2 &  2.6 &  2:13:10 &  46.4 &  6.0 &  1:14:00 &  20.6 &  11.0 &  0:42:08 &  18.3 &  16.4 &  0:24:43 &  10.2 &  21.3 &  0:14:57 \\
                  &  ✗ &  ✓ &  90.9 &  2.5 &  2:15:09 &  88.3 &  6.0 &  1:12:51 &  85.4 &  10.4 &  0:40:31 &  81.7 &  15.1 &  0:26:00 &  74.2 &  20.2 &  0:15:23 \\
                  &  ✗ &  ✗ &  73.3 &  2.5 &  2:15:09 &  16.4 &  6.0 &  1:12:51 &  13.6 &  10.4 &  0:40:31 &   9.9 &  15.1 &  0:26:00 &  10.7 &  20.2 &  0:15:23 \\  \hline
                  &  ✓ &  ✓ &  90.7 &  2.6 &  4:13:21 &  88.0 &  6.0 &  2:50:54 &  85.2 &  11.0 &  1:30:23 &  81.1 &  16.4 &  0:48:09 &  72.6 &  21.3 &  0:24:36 \\
      SeqInChange &  ✓ &  ✗ &  82.3 &  2.6 &  4:13:21 &  45.5 &  6.0 &  2:50:54 &  24.6 &  11.0 &  1:30:23 &  17.3 &  16.4 &  0:48:09 &  12.9 &  21.3 &  0:24:36 \\
                  &  ✗ &  ✓ &  90.9 &  2.5 &  4:01:33 &  88.3 &  6.0 &  2:41:17 &  85.5 &  10.4 &  1:29:22 &  81.4 &  15.1 &  0:46:34 &  72.7 &  20.2 &  0:22:47 \\
                  &  ✗ &  ✗ &  75.5 &  2.5 &  4:01:33 &  17.5 &  6.0 &  2:41:17 &  12.1 &  10.4 &  1:29:22 &   9.9 &  15.1 &  0:46:34 &  10.6 &  20.2 &  0:22:47 \\  \hline
                  &  ✓ &  ✓ &  90.7 &  2.6 &  4:20:16 &  88.0 &  6.0 &  2:53:00 &  85.1 &  11.0 &  1:38:14 &  81.5 &  16.4 &  0:51:46 &  72.6 &  21.3 &  0:24:11 \\
    LayerInChange &  ✓ &  ✗ &  72.3 &  2.6 &  4:20:16 &  23.7 &  6.0 &  2:53:00 &  15.8 &  11.0 &  1:38:14 &  14.0 &  16.4 &  0:51:46 &  11.3 &  21.3 &  0:24:11 \\
                  &  ✗ &  ✓ &  90.9 &  2.5 &  4:21:26 &  88.4 &  6.0 &  2:49:09 &  85.4 &  10.4 &  1:29:47 &  81.9 &  15.1 &  0:49:17 &  73.4 &  20.2 &  0:22:33 \\
                  &  ✗ &  ✗ &  38.6 &  2.5 &  4:21:26 &  11.4 &  6.0 &  2:49:09 &   9.9 &  10.4 &  1:29:47 &  10.2 &  15.1 &  0:49:17 &  11.0 &  20.2 &  0:22:33 \\  \hline
                  &  ✓ &  ✓ &  91.3 &  1.7 &  0:02:07 &  89.8 &  2.6 &  0:02:22 &  86.8 &   4.2 &  0:01:47 &  81.0 &   6.7 &  0:01:55 &  71.3 &  10.8 &  0:01:31 \\
          ActGrad &  ✓ &  ✗ &  85.2 &  1.7 &  0:02:07 &  50.4 &  2.6 &  0:02:22 &  21.3 &   4.2 &  0:01:47 &  14.1 &   6.7 &  0:01:55 &  11.4 &  10.8 &  0:01:31 \\
                  &  ✗ &  ✓ &  91.2 &  1.7 &  0:00:02 &  89.8 &  2.6 &  0:00:02 &  86.8 &   4.2 &  0:00:05 &  81.8 &   6.7 &  0:00:02 &  72.3 &  10.8 &  0:00:07 \\
                  &  ✗ &  ✗ &  50.3 &  1.7 &  0:00:02 &  16.4 &  2.6 &  0:00:02 &  10.5 &   4.2 &  0:00:05 &  11.7 &   6.7 &  0:00:02 &  10.2 &  10.8 &  0:00:07 \\  \hline
                  &  ✓ &  ✓ &  90.6 &  2.8 &  0:02:01 &  87.8 &  6.1 &  0:02:04 &  85.0 &  10.6 &  0:01:37 &  80.6 &  15.5 &  0:01:43 &  71.2 &  21.3 &  0:02:02 \\
     LayerActGrad &  ✓ &  ✗ &  78.4 &  2.8 &  0:02:01 &  35.2 &  6.1 &  0:02:04 &  18.4 &  10.6 &  0:01:37 &  13.1 &  15.5 &  0:01:43 &  11.6 &  21.3 &  0:02:02 \\
                  &  ✗ &  ✓ &  90.5 &  2.9 &  0:00:06 &  88.4 &  5.7 &  0:00:05 &  85.3 &  10.1 &  0:00:10 &  81.6 &  15.1 &  0:00:10 &  72.2 &  20.2 &  0:00:05 \\
                  &  ✗ &  ✗ &  21.9 &  2.9 &  0:00:06 &  10.2 &  5.7 &  0:00:05 &  10.6 &  10.1 &  0:00:10 &   9.9 &  15.1 &  0:00:10 &  10.3 &  20.2 &  0:00:05 \\  \hline
                  &  ✓ &  ✓ &  90.8 &  2.7 &  0:02:14 &  88.4 &  5.8 &  0:01:48 &  84.9 &  10.3 &  0:01:47 &  80.8 &  16.4 &  0:01:30 &  71.7 &  21.3 &  0:01:39 \\
  LayerWeightNorm &  ✓ &  ✗ &  81.8 &  2.7 &  0:02:14 &  46.7 &  5.8 &  0:01:48 &  20.4 &  10.3 &  0:01:47 &  12.0 &  16.4 &  0:01:30 &  11.0 &  21.3 &  0:01:39 \\
                  &  ✗ &  ✓ &  90.7 &  2.7 &  0:00:00 &  88.4 &  5.9 &  0:00:00 &  85.4 &  10.2 &  0:00:00 &  81.8 &  15.1 &  0:00:00 &  71.8 &  20.2 &  0:00:00 \\
                  &  ✗ &  ✗ &  25.4 &  2.7 &  0:00:00 &   9.9 &  5.9 &  0:00:00 &   9.4 &  10.2 &  0:00:00 &   9.9 &  15.1 &  0:00:00 &   9.8 &  20.2 &  0:00:00 \\  \hline
                  &  ✓ &  ✓ &  90.9 &  1.9 &  0:03:37 &  88.9 &  3.6 &  0:05:02 &  84.7 &   7.2 &  0:04:05 &  76.7 &  14.2 &  0:03:47 &  68.7 &  20.7 &  0:05:12 \\ 
    LayerSampling &  ✓ &  ✗ &  79.2 &  1.9 &  0:03:37 &  32.7 &  3.6 &  0:05:02 &  14.0 &   7.2 &  0:04:05 &  11.8 &  14.2 &  0:03:47 &  10.3 &  20.7 &  0:05:12 \\
                  &  ✗ &  ✓ &  91.0 &  1.9 &  0:02:59 &  88.9 &  3.6 &  0:06:36 &  84.1 &   7.2 &  0:02:30 &  75.1 &  14.2 &  0:01:54 &  66.2 &  20.7 &  0:05:55 \\
                  &  ✗ &  ✗ &  26.0 &  1.9 &  0:02:59 &  11.9 &  3.6 &  0:06:36 &  11.0 &   7.2 &  0:02:30 &   9.5 &  14.2 &  0:01:54 &   9.5 &  20.7 &  0:05:55 \\  \hline
                  &  ✓ &  ✓ &  90.6 &  2.6 &  0:38:44 &  88.2 &  5.3 &  8:43:34 &  84.6 &   9.9 &  0:10:17 &  80.2 &  15.1 &  0:06:57 &  71.4 &  17.4 &  0:03:16 \\  
    LayerGreedyFS &  ✓ &  ✗ &  71.5 &  2.6 &  0:38:44 &  36.3 &  5.3 &  8:43:34 &  19.6 &   9.9 &  0:10:17 &  15.2 &  15.1 &  0:06:57 &  11.4 &  17.4 &  0:03:16 \\
                  &  ✗ &  ✓ &  90.8 &  2.6 &  0:30:11 &  88.6 &  5.7 &  0:36:54 &  85.5 &   9.7 &  0:19:09 &  80.7 &  16.4 &  0:04:44 &  73.0 &  21.3 &  0:02:05 \\
                  &  ✗ &  ✗ &  73.9 &  2.6 &  0:30:11 &  36.0 &  5.7 &  0:36:54 &  21.4 &   9.7 &  0:19:09 &  13.8 &  16.4 &  0:04:44 &  12.1 &  21.3 &  0:02:05 \\  \hline
                  &  ✓ &  ✓ &  90.6 &  2.6 &  6:54:48 &  88.4 &  5.6 &  0:21:51 &  85.0 &   9.9 &  4:19:50 &  80.0 &  15.2 &  0:07:27 &  70.6 &  17.4 &  0:06:58 \\  
 LayerGreedyFS-fd &  ✓ &  ✗ &  75.9 &  2.6 &  6:54:48 &  28.7 &  5.6 &  0:21:51 &  17.4 &   9.9 &  4:19:50 &  14.4 &  15.2 &  0:07:27 &  12.2 &  17.4 &  0:06:58 \\
                  &  ✗ &  ✓ &  90.7 &  2.6 &  1:09:55 &  88.6 &  5.7 &  0:25:45 &  85.6 &   9.3 &  0:18:38 &  81.1 &  15.3 &  0:11:40 &  74.3 &  22.0 &  0:04:07 \\
                  &  ✗ &  ✗ &  72.6 &  2.6 &  1:09:55 &  38.0 &  5.7 &  0:25:45 &  20.6 &   9.3 &  0:18:38 &  13.4 &  15.3 &  0:11:40 &  13.0 &  22.0 &  0:04:07 \\
\bottomrule
\end{tabular}
\end{adjustbox}
\end{center}
\end{table}

\begin{table}
\caption{Top-1 Accuracy \% (Acc1), speedup ratio (SR), and pruning time (in hrs:mins:secs) of different pruning methods applied to VGG11 on CIFAR10, with different compression ratios $c$, with and without reweighting (rw) and fine-tuning (ft).} \label{table:CIFAR10-VGG11}
\begin{center}
\begin{adjustbox}{width=1\textwidth}
\begin{tabular}{c|cc|ccc|ccc|ccc|ccc|ccc} 
\toprule
 & & & \multicolumn{3}{c}{$c=2$} &  \multicolumn{3}{c}{$c=4$} & \multicolumn{3}{c}{$c=8$} & \multicolumn{3}{c}{$c=16$} & \multicolumn{3}{c}{$c=32$} \\
           Method & rw & ft &  Acc1 &   SR &      time &  Acc1 &   SR &      time &  Acc1 &   SR &      time &  Acc1 &    SR &      time &  Acc1 &    SR &      time \\
\midrule
                  &  ✓ &  ✓ &  89.7 &  1.9 &  31:38:06 &  89.7 &  2.1 &  30:35:13 &  89.6 &  2.5 &  23:54:22 &  89.5 &   5.0 &  22:38:37 &  81.4 &   8.3 &  19:01:51 \\
     AsymInChange &  ✓ &  ✗ &  89.7 &  1.9 &  31:38:06 &  89.7 &  2.1 &  30:35:13 &  89.6 &  2.5 &  23:54:22 &  89.5 &   5.0 &  22:38:37 &  80.4 &   8.3 &  19:01:51 \\
                  &  ✗ &  ✓ &  90.1 &  1.9 &  34:15:16 &  90.1 &  2.1 &  33:33:05 &  89.3 &  2.5 &  24:58:50 &  89.1 &   4.7 &  18:59:56 &  87.6 &   6.8 &  16:52:36 \\
                  &  ✗ &  ✗ &  90.1 &  1.9 &  34:15:16 &  90.1 &  2.1 &  33:33:05 &  88.9 &  2.5 &  24:58:50 &  88.7 &   4.7 &  18:59:56 &  12.1 &   6.8 &  16:52:36 \\  \hline
                  &  ✓ &  ✓ &  90.1 &  1.9 &  30:29:41 &  90.1 &  2.1 &  30:31:14 &  90.0 &  2.5 &  23:25:17 &  89.8 &   5.0 &  25:04:38 &  81.9 &   8.3 &  23:54:55 \\
      SeqInChange &  ✓ &  ✗ &  90.1 &  1.9 &  30:29:41 &  90.1 &  2.1 &  30:31:14 &  90.0 &  2.5 &  23:25:17 &  89.8 &   5.0 &  25:04:38 &  81.6 &   8.3 &  23:54:55 \\
                  &  ✗ &  ✓ &  90.1 &  1.9 &  32:17:12 &  90.1 &  2.1 &  31:13:02 &  89.3 &  2.5 &  23:51:44 &  89.3 &   4.7 &  22:31:17 &  87.7 &   6.8 &  20:23:53 \\
                  &  ✗ &  ✗ &  90.1 &  1.9 &  32:17:12 &  90.1 &  2.1 &  31:13:02 &  89.1 &  2.5 &  23:51:44 &  88.8 &   4.7 &  22:31:17 &  18.2 &   6.8 &  20:23:53 \\  \hline
                  &  ✓ &  ✓ &  90.1 &  1.9 &  27:27:33 &  90.1 &  2.1 &  29:43:24 &  90.0 &  2.5 &  20:45:59 &  89.8 &   5.0 &  29:00:23 &  85.3 &   8.3 &  18:56:08 \\
    LayerInChange &  ✓ &  ✗ &  90.1 &  1.9 &  27:27:33 &  90.1 &  2.1 &  29:43:24 &  90.0 &  2.5 &  20:45:59 &  89.8 &   5.0 &  29:00:23 &  84.7 &   8.3 &  18:56:08 \\
                  &  ✗ &  ✓ &  90.1 &  1.9 &  36:41:49 &  90.1 &  2.1 &  34:49:36 &  89.2 &  2.5 &  29:05:17 &  89.1 &   4.7 &  21:39:32 &  87.8 &   6.8 &  20:45:04 \\
                  &  ✗ &  ✗ &  90.1 &  1.9 &  36:41:49 &  90.1 &  2.1 &  34:49:36 &  89.2 &  2.5 &  29:05:17 &  88.9 &   4.7 &  21:39:32 &  14.7 &   6.8 &  20:45:04 \\  \hline
                  &  ✓ &  ✓ &  90.1 &  3.2 &   0:00:58 &  90.1 &  3.4 &   0:01:17 &  90.1 &  3.6 &   0:01:28 &  90.1 &   4.9 &   0:00:59 &  88.0 &  10.2 &   0:00:54 \\
          ActGrad &  ✓ &  ✗ &  90.1 &  3.2 &   0:00:58 &  90.1 &  3.4 &   0:01:17 &  90.1 &  3.6 &   0:01:28 &  90.1 &   4.9 &   0:00:59 &  83.2 &  10.2 &   0:00:54 \\
                  &  ✗ &  ✓ &  90.1 &  3.2 &   0:00:03 &  90.1 &  3.4 &   0:00:02 &  90.1 &  3.6 &   0:00:03 &  90.1 &   4.9 &   0:00:02 &  87.9 &  10.2 &   0:00:04 \\
                  &  ✗ &  ✗ &  90.1 &  3.2 &   0:00:03 &  90.1 &  3.4 &   0:00:02 &  90.1 &  3.6 &   0:00:03 &  90.1 &   4.9 &   0:00:02 &  56.8 &  10.2 &   0:00:04 \\  \hline
                  &  ✓ &  ✓ &  89.7 &  2.0 &   0:03:53 &  90.0 &  2.1 &   0:02:17 &  89.7 &  2.5 &   0:01:03 &  89.5 &   4.6 &   0:00:48 &  87.7 &   8.4 &   0:01:03 \\
     LayerActGrad &  ✓ &  ✗ &  89.7 &  2.0 &   0:03:53 &  90.0 &  2.1 &   0:02:17 &  89.7 &  2.5 &   0:01:03 &  89.5 &   4.6 &   0:00:48 &  75.9 &   8.4 &   0:01:03 \\
                  &  ✗ &  ✓ &  90.1 &  1.9 &   0:00:05 &  90.1 &  2.0 &   0:00:06 &  90.1 &  2.6 &   0:00:05 &  89.5 &   5.0 &   0:00:06 &  87.2 &   5.5 &   0:00:08 \\
                  &  ✗ &  ✗ &  90.1 &  1.9 &   0:00:05 &  90.1 &  2.0 &   0:00:06 &  90.1 &  2.6 &   0:00:05 &  89.5 &   5.0 &   0:00:06 &  16.0 &   5.5 &   0:00:08 \\  \hline
                  &  ✓ &  ✓ &  90.1 &  1.8 &   0:01:52 &  90.1 &  2.0 &   0:00:53 &  90.0 &  2.5 &   0:00:46 &  89.8 &   5.0 &   0:01:08 &  88.1 &   8.3 &   0:00:52 \\
  LayerWeightNorm &  ✓ &  ✗ &  90.1 &  1.8 &   0:01:52 &  90.1 &  2.0 &   0:00:53 &  90.0 &  2.5 &   0:00:46 &  89.8 &   5.0 &   0:01:08 &  80.9 &   8.3 &   0:00:52 \\
                  &  ✗ &  ✓ &  90.1 &  1.8 &   0:00:00 &  90.1 &  2.0 &   0:00:01 &  90.1 &  2.6 &   0:00:01 &  89.6 &   5.2 &   0:00:01 &  87.0 &  12.2 &   0:00:01 \\
                  &  ✗ &  ✗ &  90.1 &  1.8 &   0:00:00 &  90.1 &  2.0 &   0:00:01 &  90.1 &  2.6 &   0:00:01 &  89.6 &   5.2 &   0:00:01 &  11.6 &  12.2 &   0:00:01 \\  \hline
                  &  ✓ &  ✓ &  90.1 &  3.6 &   0:16:44 &  90.1 &  3.6 &   0:18:20 &  90.1 &  4.2 &   0:14:59 &  89.2 &   5.6 &   0:15:04 &  86.5 &  14.6 &   0:21:26 \\
    LayerSampling &  ✓ &  ✗ &  90.1 &  3.6 &   0:16:44 &  90.1 &  3.6 &   0:18:20 &  90.1 &  4.2 &   0:14:59 &  89.2 &   5.6 &   0:15:04 &  75.0 &  14.6 &   0:21:26 \\
                  &  ✗ &  ✓ &  90.0 &  3.7 &   0:16:18 &  90.0 &  3.7 &   0:16:48 &  89.9 &  4.2 &   0:19:04 &  89.0 &   5.6 &   0:15:04 &  86.3 &  14.6 &   0:15:47 \\
                  &  ✗ &  ✗ &  90.0 &  3.7 &   0:16:18 &  90.0 &  3.7 &   0:16:48 &  89.9 &  4.2 &   0:19:04 &  50.5 &   5.6 &   0:15:04 &  13.7 &  14.6 &   0:15:47 \\  \hline
                  &  ✓ &  ✓ &  89.3 &  2.0 &  23:30:16 &  83.3 &  2.9 &   7:50:23 &  74.2 &  4.8 &   5:12:07 &  40.3 &  10.9 &   6:01:53 &  31.2 &  18.3 &   1:51:43 \\
    LayerGreedyFS &  ✓ &  ✗ &  89.3 &  2.0 &  23:30:16 &  83.3 &  2.9 &   7:50:23 &  74.2 &  4.8 &   5:12:07 &  35.9 &  10.9 &   6:01:53 &  16.1 &  18.3 &   1:51:43 \\
                  &  ✗ &  ✓ &  89.2 &  2.0 &   9:49:18 &  88.7 &  2.2 &   6:56:45 &  87.2 &  2.8 &   4:10:47 &  84.4 &   4.8 &   2:40:57 &  46.4 &   7.4 &   2:36:44 \\
                  &  ✗ &  ✗ &  89.1 &  2.0 &   9:49:18 &  88.6 &  2.2 &   6:56:45 &  82.4 &  2.8 &   4:10:47 &  63.7 &   4.8 &   2:40:57 &  19.1 &   7.4 &   2:36:44 \\  \hline
                  &  ✓ &  ✓ &  89.4 &  2.0 &  12:06:03 &  86.8 &  2.7 &  10:22:33 &  69.1 &  4.0 &   6:00:33 &  44.3 &  13.0 &   3:27:58 &  31.1 &  13.0 &   3:35:15 \\
 LayerGreedyFS-fd &  ✓ &  ✗ &  89.3 &  2.0 &  12:06:03 &  86.1 &  2.7 &  10:22:33 &  69.1 &  4.0 &   6:00:33 &  44.3 &  13.0 &   3:27:58 &  21.9 &  13.0 &   3:35:15 \\
                  &  ✗ &  ✓ &  89.4 &  2.0 &  12:42:58 &  88.6 &  2.6 &   7:40:45 &  87.1 &  3.3 &   5:04:15 &  85.5 &   4.9 &   3:45:22 &  63.5 &  11.5 &   1:59:09 \\
                  &  ✗ &  ✗ &  89.4 &  2.0 &  12:42:58 &  88.0 &  2.6 &   7:40:45 &  81.5 &  3.3 &   5:04:15 &  76.6 &   4.9 &   3:45:22 &  15.5 &  11.5 &   1:59:09 \\  
\bottomrule
\end{tabular}
\end{adjustbox}
\end{center}
\end{table}